\documentclass{article}

\usepackage[english]{babel}
\usepackage{authblk}
\usepackage[letterpaper,top=2cm,bottom=2cm,left=3cm,right=3cm,marginparwidth=1.75cm]{geometry}
\usepackage{todonotes}
\usepackage{amsmath}
\usepackage{graphicx}
\usepackage[colorlinks=true, allcolors=blue]{hyperref}
\usepackage{mathtools}
\usepackage{amssymb} 
\usepackage{amsthm}
\usepackage{mathrsfs}
\newtheorem{theorem}{Theorem}
\newtheorem{proposition}{Proposition}

\newtheorem{lemma}{Lemma}
\newtheorem{corollary}{Corollary}
\usepackage{bm} 
\usepackage{booktabs}
\usepackage{float}
\usepackage{xcolor}
\usepackage{array}
\usepackage{multirow}
\usepackage{enumitem}
\usepackage[nohints]{minitoc}

\definecolor{dominant}{RGB}{200, 50, 50}  
\definecolor{good}{RGB}{50, 150, 50}      
\definecolor{expweight}{RGB}{180, 100, 40}  

\newcommand{\dom}[1]{\textcolor{dominant}{\mathbf{#1}}}  
\newcommand{\good}[1]{\textcolor{good}{#1}}              
\newcommand{\expw}[1]{\textcolor{expweight}{#1}}         

\newcommand{\bphi}{\bm{\phi}}

\newcommand{\bgamma}{\bm{\gamma}}
\newcommand{\bzeta}{\zeta}
\newcommand{\balpha}{\bm{\alpha}}
\newcommand{\blambda}{\bm{\lambda}}
\newcommand{\bpi}{\bm{\pi}}

\newcommand{\bPhi}{\bm{\Phi}}

\newcommand{\bx}{\bm{x}}

\newcommand{\bJ}{\bm{J}}
\newcommand{\R}{\mathbb{R}}

\newcommand{\bp}{\bm{p}}
\newcommand{\bs}{\bm{s}}
\newcommand{\fwda}{{\scriptscriptstyle\rightarrow}}  
\newcommand{\bwda}{{\scriptscriptstyle\leftarrow}}  
\newcommand{\bda}{{\scriptscriptstyle\leftrightarrow}}  

\newcommand{\by}{\bm{y}}

\newcommand{\btheta}{\bm{{\theta}}}
\newcommand{\bSigma}{\bm{\Sigma}}

\newcommand{\bmeta}{\bm{\eta}}
\newtheorem{remark}{Remark}

\newcommand{\bu}{\bm{u}}
\newcommand{\bmu}{\bm{\mu}}
\newcommand{\EL}{\mathrm{EL}}
\newcommand{\dt}{\mathrm{dt}}
\newcommand{\dd}{\mathrm{d}}
\newcommand{\dtheta}{\mathrm{d}_{\btheta}}
\newcommand{\defn}{:=}
\usepackage{natbib}
\usepackage[normalem]{ulem}
\usepackage{relsize}
\usepackage{tikz}
\usetikzlibrary{arrows.meta, positioning}
\title{Generalizing Equilibrium Propagation to Lagrangian systems with arbitrary boundary conditions \\
\& equivalence with Hamiltonian Echo Learning}
\author[1,2]{Guillaume Pourcel}
\author[2]{Debabrota Basu\footnote{Equal authorship, listed in alphabetical order.}}
\author[3]{Maxence Ernoult$^{\ast, \blacklozenge}$}
\author[4,5]{Aditya Gilra$^{\ast}$}
\affil[1]{University of Groningen, Netherlands, \texttt{g.a.pourcel@rug.nl}} 
\affil[2]{Univ. Lille, Inria,   CNRS, Centrale Lille, UMR 9189 - CRIStAL, Lille, France, \texttt{debabrota.basu@inria.fr}} 
\affil[3]{\texttt{ernoult.m@gmail.com}}
\affil[4]{Centrum Wiskunde \& Informatica, Netherlands, \texttt{aditya.gilra@cwi.nl}} 
\affil[5]{The University of Sheffield, UK} 
\date{}

\begin{document}

\maketitle
\doparttoc
\faketableofcontents

\begin{abstract}

Equilibrium Propagation (EP) is a learning algorithm that applies to
Energy-Based Models (EBMs) on static inputs. 
It estimates loss gradients by contrasting two steady states of the same EBM, rather than resorting to explicit adjoint dynamics. EP originally appealed as a plausible learning theory for biological substrates and has more recently attracted interest for its amenability to analog hardware. Extending EP to time-varying inputs and outputs is a challenging problem,
 as the variational description must apply to the entire system trajectory rather than just its steady state.
 While the use of the action of a Lagrangian system as an energy function appears as a natural choice -- which we herein refer to as \emph{Lagrangian Equilibrium Propagation} (LEP) -- careful consideration of boundary conditions was largely overlooked in prior studies although it becomes essential. It is also unclear how applying LEP to Lagrangian systems theoretically relates to applying \emph{Hamiltonian Echo Learning} (HEL) algorithms -- \emph{i.e.} Hamiltonian Echo Backpropagation (HEB) 
 and Recurrent Hamiltonian Echo Learning (RHEL) -- to Hamiltonian systems.

In this work, we thoroughly revisit LEP and demonstrate that different learning algorithms can be obtained depending on the boundary conditions of the system, many of which are impractical to simulate -- \emph{e.g.} with a prohibitive memory or computational cost, or requiring explicit Jacobian computation. 
We also show that HEL algorithms, which are much easier to simulate, can be explicitly cast as a special case of LEP where the initial conditions can be picked arbitrarily. Building upon this connection enables the extension of LEP to a broader class of systems with dissipation terms. By filtering out intractable instantiations of LEP and building an explicit mapping between HEL and LEP algorithms, this work facilitates the simulation of self-learning Lagrangian systems as well as extensions of LEP to broader classes of physical systems. 

\end{abstract}

\section{Introduction}

\paragraph{The search for an alternative to backpropagation.}
Historically, feedforward networks alongside backpropagation have accidentally dominated the deep learning landscape over the last decade as the result of a ``hardware lottery''~\citep{hookerHardwareLottery2020}: algorithms fitting the best available hardware win.
Thanks to fine-grained CMOS-based compute primitives along with the development of hardware-agnostic compilation flows, digital hardware (\emph{e.g.} CPUs, GPUs, TPUs~\cite{jouppi2017datacenter}) provides the flexibility to execute any feedforward computational graph, including the exact implementation of backpropagation with the least amount of engineering. However this comes at the cost of digital overhead, complex memory hierarchies, and resulting data movement. In the short run, this motivates the search for ``IO--aware'' algorithms~\cite{dao2022flashattention} to mitigate High-Bandwidth Memory (HBM) accesses, quantization algorithms to further reduce the memory cost and computational cost of \emph{verbatim} backpropagation for on-device applications~\cite{lin2022device}, and many other approaches going beyond the scope of this paper. Yet, none of these approaches truly leverage the underlying low-power transistor physics. Instead, transistor circuits remain abstracted away into implementing huge boolean functions in a stateless, unidirectional and deterministic fashion, which entails a significant energy consumption \citep{aifer2025solving}.
In the longer run, a radically different approach is the search for higher-level \emph{analog} compute primitives, in particular, primitives for alternative learning and inference algorithms
grounded in the analog physics 
of the underlying hardware~\citep{jaegerFormalTheoryComputing2023, laydevantHardwareSoftware2024}.

An important direction of research to achieve this goal is 
the development of learning algorithms that unify inference and learning within a \emph{single} physical circuit~\citep{smolensky1986information, spall1992multivariate,fiete2007model,scellierEquilibriumPropagationBridging2017,gilraPredicting2017,ren2022scaling,hinton2022forward,lopez-pastorSelfLearningMachinesBased2023,dillavou2024machine}.
This challenge, which we herein motivate for alternative hardware design, historically originated from neurosciences:
biological neural networks face similar constraints, as ``non-local'' algorithms such as backpropagation 
are widely considered biologically implausible for training neural networks
~\citep{backprop, lillicrap2020backpropagation}. For instance, the strict implementation of backpropagation in biological systems would require a dedicated side network sharing parameters from the inference circuit to propagate error derivatives backward through 
the system, a problem coined weight transport~\citep{lillicrap2016random, akrout2019deep}. The search for 
backpropagation alternatives therefore holds promise 
for both providing insights into how the brain might learn~\cite{richards2019deep,pogodin2023synaptic} and designing 
energy-efficient analog hardware~\cite{momeni2024training}.


\paragraph{Equilibrium Propagation and its limitations.} \emph{Equilibrium Propagation} (EP)~\cite{scellierEquilibriumPropagationBridging2017} is a learning algorithm using a single circuit for inference and gradient computation and yielding an unbiased, variance-free gradient estimation -- which is in stark contrast with alternative approaches relying on the use of noise injection~\cite{smolensky1986information, spall1992multivariate,fiete2007model,ren2022scaling}. A fundamental requirement of EP is that the models that are used should be energy-based. Energy-based Models (EBMs) are models whose prediction is implicitly given as the minimum of some energy function. Therefore, EP falls under the umbrella of implicit learning algorithms such as implicit differentiation (ID)~\cite{bell2008algorithmic} which train implicit models~\cite{bai2019deep} to have steady states mapping static input--output pairs. EP is endowed with strong theoretical guarantees~\cite{scellier2019equivalence,ernoult2019updates} as it can be shown to be equivalent to a variant of ID called Recurrent Backpropagation~\cite{almeida1989backpropagation,pineda1989recurrent}.
While EP has been predominantly explored on Deep Hopfield Networks~\cite{rosenblatt1960perceptual,hopfield1982neural, scellierEquilibriumPropagationBridging2017,ernoult2019updates,laborieux2021scaling,laborieux2022holomorphic,scellier2023energy,nest2024towards}, the application of EP to resistive networks~\citep{kendall2020training,scellier2024fast} has ushered in an exciting direction of research for learning algorithms amenable to analog hardware, with projected energy savings of four orders of magnitude~\cite{yi2023activity}. 
Beyond the single-circuit property, EP naturally yields \emph{local} learning rules whenever the energy is sum-separable~\citep{scellier2023energy}, and can be made agnostic to the underlying physics~\citep{scellierAgnosticPhysicsDrivenDeep2022}. Hopfield-inspired models further give rise to local Hebbian learning rules~\citep{scellierEquilibriumPropagationBridging2017}. These properties resonate with neuroscience, where the same neural circuitry appears to be involved in both inference and learning~\citep{songInferringNeuralActivity2024, aceitunoTargetLearningRather2024}.

Yet, a major conundrum is to extend EP to \emph{time-varying inputs and outputs}. One straightforward approach would be to consider well-crafted EBMs which adiabatically evolve with incoming inputs -- \emph{i.e.} at each time step, the system settles to equilibrium under the influence of the current input and of the steady state reached under the previous input. Such EBMs would formally fall under the umbrella of Feedforward-tied EBMs~\citep{nest2024towards}, which read as feedforward composition of EBM blocks and are reminiscent of fast associative memory models~\cite{ba2016using}. However, this approach is tied to a very specific class of models, would be costly to simulate (\emph{i.e.} computing a steady state for each incoming input) and would be memory intensive (\emph{i.e.} it would require storing the whole sequence of steady states and traversing them backwards for EP-based gradient estimation). 
A more general approach to extend EP to the temporal realm is to instead consider dissipation-free systems and pick their \emph{action} as an energy function~\citep{scellierDeepLearningTheory2021, kendall2021gradient}, which we herein refer to as \emph{Lagrangian-based EP} (LEP).
In the LEP setup, the energy minimizer is no longer a steady state alone but \emph{the whole physical trajectory}. Crucially, both~\citep{scellierDeepLearningTheory2021} and~\citep{kendall2021gradient} implicitly assumed boundary-value-problem conditions---\emph{i.e.} vanishing variations at both endpoints---yet neither study provided a practical algorithm nor implementation, raising the question of how feasible this assumption actually is. 
More broadly, existing LEP proposals remain theoretical and did not lead to any practical algorithmic prescriptions, which we diagnose as due to the need to \emph{carefully handle boundary conditions} arising in the underlying variational problem. 
This limitation raises our first key question:
\begin{center}
    \textit{Can EP be generalised to design efficient and practically-implementable\\ learning algorithms for time-varying inputs and outputs?}
\end{center}



\paragraph{Hamiltonian-based approaches.} 
In parallel to EP research, learning algorithms grounded in \emph{reversible} 
Hamiltonian dynamics have emerged as another promising direction of research. 
One such algorithm, Hamiltonian Echo Backpropagation (HEB,~\citep{lopez-pastorSelfLearningMachinesBased2023}), was developed 
with theoretical physics tools to train the initial conditions of physical systems 
governed by field equations for static input-output mappings. 
More recently, Recurrent Hamiltonian Echo Learning (RHEL) was introduced as a generalization of HEB 
to time-varying inputs and outputs~\citep{rhel}. 
Like EP, these Hamiltonian-based approaches, which we herein label as \textit{Hamiltonian Echo Learning (HEL)} algorithms, enable a single physical system 
to perform both inference and learning whilst maintaining the theoretical 
equivalence to BPTT.
Interestingly, HEL methods were also independently found to yield local Hebbian learning rules \citep{dauphinRecurrentHamiltonianEcho2025}, and to lend themselves to be agnostic to the underlying physics~\citep{rhel}.
Since HEL algorithms originate from a different formalism than that of LEP,
this motivates our second key question:
\begin{center}
\textit{How do HEL algorithms relate to LEP?}
\end{center}

In this paper,
we address both questions through a theoretical analysis that reveals the connection between these approaches. Our contributions are organized as follows:

\begin{itemize}
    \item We revisit 
          \emph{Lagrangian Equilibrium Propagation} (LEP), which extends the
          variational formulation of EP to temporal trajectories (Section~\ref{subsec:lagrangian-ep}). Our formulation generalizes previous studies \citep{scellierDeepLearningTheory2021,kendall2021gradient} by carefully analyzing the effect of different boundary conditions, explicitly treating both the boundary-value assumption of prior work (CBPVP, Section~\ref{subsec:cfvp}) and the initial-value alternative (CIVP, Section~\ref{subsec:civp}).
    \item We show that the boundary-value formulation (CBPVP) assumed by prior work eliminates boundary residuals from the learning rule but requires an expensive non-causal iterative solver whose cost dominates the overall complexity (Section~\ref{subsec:instantiation}). We then show that the natural causal alternative (CIVP), which restores forward Euler-Lagrange integration, introduces intractable boundary residual terms. Neither formulation leads to a practical algorithm on its own.
    \item We demonstrate that \emph{RHEL can be derived as a special case of LEP} by constructing an associated reversible Lagrangian system with carefully chosen boundary conditions (PFVP) that eliminate the problematic residual terms while preserving causal forward integration---yielding a \emph{first practical implementation of LEP}. Further, we establish the mathematical equivalence of RHEL and LEP through the Legendre transformation (Section~\ref{sec:rhel_is_ep}). We empirically validate this equivalence with a numerical analysis comparing the gradient estimates obtained by LEP and RHEL (Section~\ref{subsec:empirical-validation}).
    \item Finally, we directly leverage the connection between RHEL and LEP to come up with a variant of LEP that applies to dissipative Lagrangians which we call \emph{Dissipative LEP} (Section~\ref{subsec:dissipative_lep}). Provided that the sign of the dissipation term in the dynamics of the system can be arbitrarily controlled (\emph{i.e.} sinking or pumping energy into the system), we empirically show that gradients can be correctly estimated.
\end{itemize}




\section{Preliminaries and problem formulation}

\subsection{The learning problem: supervised learning with time-varying input}
\label{subsec:learning_problem}
We consider the supervised learning problem, where the goal is to predict a target trajectory $\by(t) \in \mathbb{R}^{d_{\by}} $ given an input trajectory $\bx(t) \in \mathbb{R}^{d_{\bx}}$ over a continuous time interval $t \in [0,T]$. The model is parameterised by $\btheta \in \mathbb{R}^{d_{\btheta}}$ and produces predictions through a continuous state trajectory ${\bs}_t(\btheta) \in \mathbb{R}^{d_{\bs}}$ that evolves over time according to the system dynamics. In the context of continuous time systems, the state-trajectory is typically defined as the solution of an Ordinary Differential Equation (ODE). 


The learning objective is to minimize a cost functional $C[\bs(\btheta, \bx), \by]$ that measures the discrepancy between the produced trajectory and the target. Formally,
\begin{equation*}
 C[\bs(\btheta,\bx), \by] \defn \int_0^T c(\bs_t(\btheta, \bx_t), \by_t) \dt\,,
\end{equation*}
where $c(\cdot, \cdot): \mathbb{R}^{d_{\bs}} \times \mathbb{R}^{d_{\by}} \rightarrow \mathbb{R}$ is an instantaneous cost function that evaluates the prediction error at time $t$ and $\bs(\btheta, \bx) \defn \{\bs_t({\btheta},\bx_t) : t \in [0,T]\} $ represents the entire trajectory. Commonly, $c$ takes the form of an $\ell_2$ loss function, 
$c(\bs_t, \by_t) \defn \frac{1}{2}\|\bs_t^{\text{out}} - \by_t\|_2^2$, 
where $\bs_t^{\text{out}} \in \mathbb{R}^{d_{\by}}$ denotes an 
appropriately selected subset of state variables. More generally, $c$ can be 
any differentiable function that quantifies the instantaneous 
prediction error. 

The parameters ${\btheta}$ are optimised to minimise the cost functional $C[\bs(\btheta, \bx), \by]$. 
One popular approach to solve this minimisation problem is to use gradient descent-type optimisation algorithms.
Modern machine learning owes much of its success to the generality and scalability of gradient-based optimization.
This requires computing the gradient of the learning objective with respect to the parameters ${\btheta}$.  
While several methods have been proposed to compute this gradient, most rely on explicit backward passes through computational graphs~\citep{backprop,lecun2015deep}, making them unsuitable for analog hardware implementations or plausible explanations for biological learning.

This limitation has motivated the development of alternative learning paradigms. Among the existing approaches, the \textit{Equilibrium Propagation} (EP,~\citep{scellierEquilibriumPropagationBridging2017}) framework stands out as a particularly promising one for designing a single system that can perform inference and learning.

\subsection{A primer on Lagrangian and Hamiltonian models}
\label{subsec:lagrangian_hamiltonian}

In this paper, the learning algorithms considered are constraining the kind of trajectories that can be used. In particular, we will only consider state trajectories $\bs_t(\btheta)$ that arise from Lagrangian and Hamiltonian dynamics. Both Hamiltonian and Lagrangian dynamics provide frameworks for formulating specific dynamical systems using a scalar-valued function: the Lagrangian or the Hamiltonian, defined as follows:
\begin{itemize}
    \item The Lagrangian $\mathcal{L}(\bs, \dot{\bs}, \bx, \btheta): \mathbb{R}^{d_{\bs}} \times \mathbb{R}^{d_{\bs}} \times \mathbb{R}^{d_{\bx}} \times \mathbb{R}^{d_{\btheta}} \rightarrow \mathbb{R}$ is a function of the state $\bs$, its time derivative $\dot{\bs}$ (velocity), the input $\bx$, and parameters $\btheta$. The dynamics are then defined by the Euler-Lagrange equations:
    \begin{equation*}
        d_t \partial_{\dot{\bs}} \mathcal{L} - \partial_{\bs} \mathcal{L} = 0\,.
    \end{equation*}
    
    \item The Hamiltonian $H(\bs, \bp, \bx, \btheta): \mathbb{R}^{d_{\bs}} \times \mathbb{R}^{d_{\bs}} \times \mathbb{R}^{d_{\bx}} \times \mathbb{R}^{d_{\btheta}} \rightarrow \mathbb{R}$ is a function of the position $\bs$, momentum $\bp$, the input $\bx$, and parameters $\btheta$. The dynamics are defined by Hamilton's equations:
    \begin{equation*}
        \begin{pmatrix} d_t \bs \\ d_t \bp \end{pmatrix} = \bJ \begin{pmatrix} \partial_{\bs} H \\ \partial_{\bp} H \end{pmatrix}\,,
    \end{equation*}
    where $\bJ = \begin{bmatrix} \bm{0} & \bm{I} \\ -\bm{I} & \bm{0} \end{bmatrix}$ is the canonical symplectic matrix.
\end{itemize}



\paragraph{Toy example: Driven coupled harmonic oscillators (3 masses).}
A simple physical system that can be expressed in both Lagrangian and Hamiltonian form
is a set of three coupled harmonic oscillators, depicted in Figure \ref{fig:coupled_oscillators}.
Let $\bs=(s_1,s_2,s_3)^\top$ be the displacements and $\bp=(p_1,p_2,p_3)^\top$ the momenta,
with mass vector $\bm{m}=(m_1,m_2,m_3)^\top$ where $m_i>0$, per-mass spring constants $k_i\ge0$,
and pairwise spring couplings $k_{ij}=k_{ji}\ge0$.
An external input $x(t)$ acts on the first mass (the output is $y(t)=s_3(t)$).
The learnable parameters are $\btheta = (m_1, m_2, m_3, k_1, k_2, k_3, k_{12}, k_{13}, k_{23})^\top$.

The system is described by the Hamiltonian
\begin{equation*}
H(\bs,\bp,x,\btheta)=\frac{1}{2}(\bm{m}^{-1} \odot \bp)^\top \bp
+\frac{1}{2}\sum_{i=1}^3 k_i s_i^2
+\frac{1}{2}\sum_{i<j} k_{ij}(s_j-s_i)^2
+s_1\,x\,,
\end{equation*}
and equivalently by the Lagrangian
\begin{equation*}
\mathcal{L}(\bs,\dot{\bs},x,\btheta)=\frac{1}{2} (\bm{m} \odot \dot{\bs})^\top \dot{\bs}
-\frac{1}{2}\sum_{i=1}^3 k_i s_i^2
-\frac{1}{2}\sum_{i<j} k_{ij}(s_j-s_i)^2
-s_1\,x\,.
\end{equation*}
Both formulations lead to the same second-order dynamics:
\begin{equation*}
\bm{m} \odot \ddot{\bs}+K \bs=-x\,e_1\,,
\end{equation*}
where $\odot$ denotes element-wise multiplication (Hadamard product), the stiffness matrix $K$ has
$K_{ii}=k_i+\sum_{j\ne i}k_{ij}$ and $K_{ij}=-k_{ij}$ for $i\ne j$,
and $e_1=(1,0,0)^\top$ is the first canonical basis vector selecting the first mass (the driven coordinate).

\begin{figure}[t!]
\centering
\includegraphics[width=0.8\textwidth,trim={0 1.6cm 4.5cm 0},clip]{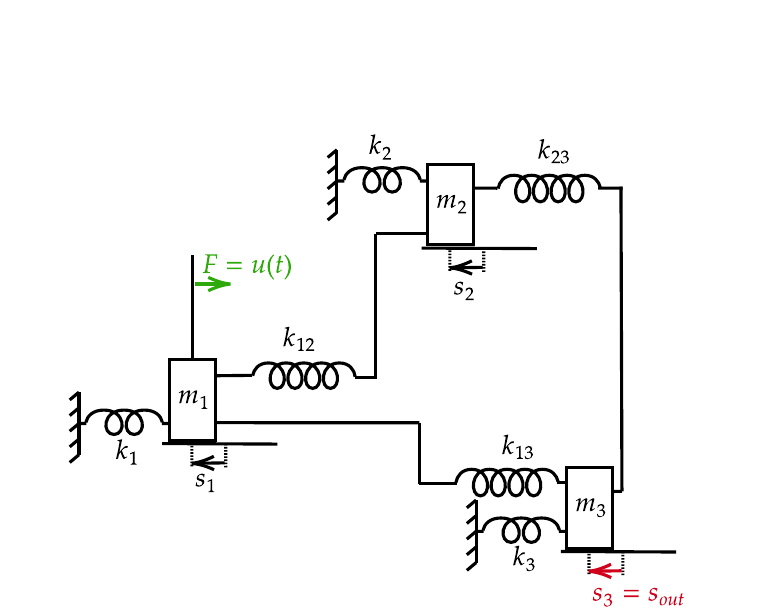}
\caption{Driven coupled harmonic oscillators: A system of three masses connected by springs with an external input $x(t)$ acting on the first mass and output $s_{out}(t) = s_3(t)$ measured from the third mass. The system dynamics $M\ddot{\bs} + K\bs = -x(t)e_1$ can be equivalently described through either a Hamiltonian $H(\bs,\bp,t)$ or Lagrangian $L(\bs,\dot{\bs},t)$ formulation, as detailed in the text above.}
\label{fig:coupled_oscillators}
\end{figure}

\paragraph{Machine learning examples.}

Lagrangian and Hamiltonian formulations are widely used in physics, and correspond to a broad class of physical systems. Recently, they have been applied to machine learning and neuroscience. In machine learning, they have been used to design RNNs with desirable vanishing or exploding gradient properties (UniCORNN,~\citep{ruschUnICORNNRecurrentModel2021}), and to design efficient modern State Space Model (SSM) architectures (LinOSS,~\cite{ruschOscillatoryStateSpaceModels2025}) -- see Table \ref{tab:ml_hamiltonians} for their Lagrangian and Hamiltonian formulations and dynamics. 

More generally, this research aligns with the renewed interest in RNNs as computationally efficient alternatives to Transformers, where state-based dynamical systems eliminate the quadratic cost of attention while maintaining comparable performance on long-range sequence tasks~\citep{orvietoResurrectingRecurrentNeural2023}. In neuroscience, it was proposed that Recurrent Hamiltonian Echo Learning (RHEL) could be implemented in a biologically plausible way via a Hamiltonian inspired by Hopfield energy functions~\citep{dauphinRecurrentHamiltonianEcho2025}.

\begin{table}[h]
\centering
\setlength{\tabcolsep}{5pt}
\renewcommand{\arraystretch}{2.2}
\resizebox{\textwidth}{!}{
\begin{tabular}{|l|c|c|c|c|}
\hline
\textbf{Model} & \textbf{Hamiltonian} ($H$) & \textbf{Lagrangian} ($L$) & \textbf{Dynamics} & \textbf{Constraint} \\[0.2cm]
\hline
\textbf{UniCORNN} \citep{ruschUnICORNNRecurrentModel2021} & 
$\begin{array}{@{}l@{}}\tfrac{1}{2}\|\bp\|^2 + \tfrac{\alpha}{2}\|\bs\|^2 \\ + \,\mathbf{1}^\top \bm{W}^{-1} \log\!\bigl(\cosh(\bm{W}\bs {+} \bm{B}\bx {+} \bm{b})\bigr)\end{array}$ 
& 
$\begin{array}{@{}l@{}}\tfrac{1}{2}\|\dot{\bs}\|^2 - \tfrac{\alpha}{2}\|\bs\|^2 \\ - \,\mathbf{1}^\top \bm{W}^{-1} \log\!\bigl(\cosh(\bm{W}\bs {+} \bm{B}\bx {+} \bm{b})\bigr)\end{array}$
& 
$\ddot{\bs} = \tanh(\bm{W}\bs {+} \bm{B}\bx {+} \bm{b}) - \alpha \bs$
& 
$\bm{W}$ diagonal
\\[0.2cm]
\hline
\textbf{LinOSS} \citep{ruschOscillatoryStateSpaceModels2025} & 
$\begin{array}{@{}l@{}}\tfrac{1}{2}\|\bp\|^2 + \tfrac{1}{2}\bs^\top \bm{W} \bs \\ - \,\bs^\top \bm{B} \bx\end{array}$ 
& 
$\begin{array}{@{}l@{}}\tfrac{1}{2}\|\dot{\bs}\|^2 - \tfrac{1}{2}\bs^\top \bm{W} \bs \\ + \,\bs^\top \bm{B} \bx\end{array}$
& 
$\ddot{\bs} = -\bm{W}\bs + \bm{B}\bx$
& 
$\bm{W}$ symmetric
\\[0.2cm]
\hline
\textbf{Hopfield} \citep{dauphinRecurrentHamiltonianEcho2025} & 
$\begin{array}{@{}l@{}}\tfrac{1}{2}\bp^\top \mathrm{diag}(\bm{\tau})^{-1} \bp {+} \bm{b}^\top \rho(\bs) \\ + \tfrac{1}{2}\rho(\bs)^\top \bm{W} \rho(\bs) {+} \rho(\bs)^\top \bm{B} \rho(\bx)\end{array}$ 
& 
$\begin{array}{@{}l@{}}\tfrac{1}{2}\dot{\bs}^\top \mathrm{diag}(\bm{\tau}) \dot{\bs} {-} \bm{b}^\top \rho(\bs)\\ - \tfrac{1}{2}\rho(\bs)^\top \bm{W} \rho(\bs) {-} \rho(\bs)^\top \bm{B} \rho(\bx) \end{array}$
& 
$\begin{array}{@{}l@{}}\mathrm{diag}(\bm{\tau})\ddot{\bs} = \\ \;-\rho'(\bs) \odot (\bm{W}\rho(\bs) {+} \bm{b} {+} \bm{B}\rho(\bx))\end{array}$
& 
$\bm{W}$ symmetric
\\[0.2cm]
\hline
\end{tabular}}
\caption{Machine learning models with Hamiltonian and Lagrangian formulations.}
\label{tab:ml_hamiltonians}
\end{table}

\subsection{Connecting Lagrangian and Hamiltonian Formulations via the Legendre Transform}

Hamiltonian and Lagrangian formalisms offer complementary perspectives on the same underlying dynamics. Each formalism possesses distinct mathematical structure that favors different proof techniques: the Hamiltonian framework, with its symplectic geometry and phase-space representation, naturally accommodates tools such as Green's functions~\citep{lopez-pastorSelfLearningMachinesBased2023} and adjoint methods~\citep{rhel}. These techniques proved instrumental in deriving HEL. Conversely, the Lagrangian framework foregrounds the variational structure of trajectories, which makes it particularly amenable to Equilibrium Propagation.

The Legendre transform provides a bridge between these two representations and allows the results established in one formalism to be translated into the other.
\begin{proposition}[Legendre transform]
\label{prop:Legendre_transform}
Let $(\bs_t,\dot{\bs}_t)\in \mathbb{R}^{d_{\bs}} \times \mathbb{R}^{d_{\bs}}$ and $(\bs_t,\bp_t)\in \mathbb{R}^{d_{\bs}} \times \mathbb{R}^{d_{\bs}}$ denote tuples of position--velocity
and position--momentum variables, respectively.  
The Legendre transform establishes a pointwise-in-time, locally invertible mapping between the Lagrangian and Hamiltonian representations, with $L, H \in C^2$:

\paragraph{(a) Forward transform ($L\!\rightarrow\!H$).}
\begin{equation*}
\bp_t = \partial_{\dot{\bs}} L(\bs_t,\dot{\bs}_t),
\qquad
H(\bs_t,\bp_t) = \bp_t^\top \dot{\bs}_t - L(\bs_t,\dot{\bs}_t),
\end{equation*}
which is well-defined whenever $\det(\partial^2_{\dot{\bs}, \dot{\bs}}L)\neq 0$.

\paragraph{(b) Backward transform ($H\!\rightarrow\!L$).}
\begin{equation*}
\dot{\bs}_t = \partial_{\bp} H(\bs_t,\bp_t),
\qquad
L(\bs_t,\dot{\bs}_t) = \bp_t^\top \dot{\bs}_t - H(\bs_t,\bp_t),
\end{equation*}
which is well-defined whenever $\det(\partial^2_{\bp, \bp}H)\neq 0$.

\paragraph{}  
Since the Hessians satisfy $\partial^2_{\bp, \bp}H = (\partial^2_{\dot{\bs}, \dot{\bs}}L)^{-1}$, the well-definiteness conditions are equivalent.
\end{proposition}

\textit{Note (Regularity and uniqueness of solutions).}
Since $L\in C^2$ and $\det(\partial^2_{\dot{\bs},\dot{\bs}}L)\neq 0$, the Euler--Lagrange equation can be rewritten as a first-order ODE whose right-hand side is locally Lipschitz, which guarantees uniqueness of solutions to initial value problems.
We invoke this uniqueness property without further comment in the sequel; see Remark~\ref{rmk:regularity} in Appendix for a detailed verification on the models of Table~\ref{tab:ml_hamiltonians}.




\section{Equilibrium Propagation: From static to time-varying input}
\label{sec:ep}
In this paper, we refer to the EP framework
as a general recipe   to design learning algorithms, where the model to be trained admits a variational description~\cite{scellierDeepLearningTheory2021}. 
The core mechanics underpinning EP are fundamentally \emph{contrastive}: EP proceeds by solving two related variational problems:
\begin{itemize}
    \item the \emph{free} problem, which defines model inference, \emph{i.e.} the ``forward pass'' of the model to be trained,
    \item the \emph{nudged} problem, which is a perturbation of the free problem with infinitesimally lower prediction error controlled by some nudging parameter. 
\end{itemize}


Therefore, EP mechanics perform two relaxations to equilibrium, \emph{e.g.} two ``forward passes'', to estimate gradients
without requiring explicit backward passes through the computational graph.

\subsection{
EP: Variational principle in \emph{vector} space}
\label{subsec:energy-ep}
In the original formulation of EP, the nudged problem 
is defined \emph{via} an augmented energy function that linearly combines an energy function with the learning cost function:
\begin{align*}
    E_\beta({\bs}, {\btheta}, \bx_0, \by_0) \defn E({\bs}, {\btheta}, \bx_0) + \beta C({\bs}, \by_0)\,.
\end{align*}
Here, $E({\bs}, \btheta, \bx_0)$ is the energy function, i.e. a scalar-valued function that takes as input a state vector ${\bs} \in \mathbb{R}^{d_{\bs}}$, a learnable parameter vector $\btheta$, and 
a \textbf{static} input $\bx_0 \in \mathbb{R}^{d_{\bx}}$. The cost function $C({\bs}, \by_0)$ in this setup takes as input a \textbf{static} output target $\by_0 \in \mathbb{R}^{d_y}$ and the static state vector. The nudging parameter $\beta \in \mathbb{R}$ controls the influence of the cost on the augmented energy. This augmented energy defines a vector-valued implicit function $(\btheta, \beta) \mapsto {\bs}^\beta(\btheta)$\footnote{For notational simplicity, we omit the explicit dependence of the implicit function $(\btheta, \beta) \mapsto {\bs}^\beta(\btheta)$ on $\bx_0$ and $\by_0$, as we consider the gradient computation for a fixed input-target pair.} through the nudged variational problem. Specifically, it is minimised at
\begin{align*}
    \partial_{{\bs}} E_\beta({\bs}, \btheta, \bx_0, \by_0) = \mathbf{0}\,.
\end{align*}
The model used for the machine learning task is the implicit function $\btheta \mapsto {\bs}^{0}(\btheta)$ defined by the free variational problem $\partial_{{\bs}} E({\bs}, \btheta, \bx_0) = \mathbf{0}$, and the learning objective is to minimize the cost $C({\bs}^{0}(\btheta), \by_0)$ by finding the gradient $\dd_{\btheta} C({\bs}^{0}(\btheta), \by_0)$. The fundamental result of EP
states that this gradient 
can be computed using \citep{scellierDeepLearningTheory2021}
\begin{align}
    \dd_{\btheta} C({\bs}^{0}(\btheta), \by_0) = \lim_{\beta \to 0} \frac{1}{\beta}
    \left[ \partial_{\btheta} E_{\beta}({\bs}^{\beta}(\btheta), \btheta, \bx_0) - 
    \partial_{\btheta} E_0({\bs}^{0}(\btheta), \btheta, \bx_0) \right]\,.
    \label{eq:energy_ep_lr}
\end{align}
This suggests a two-phase procedure for gradient estimation via a finite difference method, illustrated in 
Figure~\ref{fig:ep-learning}A: 
\begin{enumerate}
    \item \textbf{Free phase:} Compute the output value of the implicit function ${\bs}^{0}(\btheta)$ (black cross $\boldsymbol{\times}$) by finding a minimum of the energy function $E({\bs}, \btheta, \bx_0)$ (black curve).
    \item \textbf{Nudged phase:} Compute the output value of the implicit function ${\bs}^{\beta}(\btheta)$ ({\color{blue} blue dot}) for a small positive value of $\beta$ by finding a slightly perturbed minimum of the augmented energy $E_\beta({\bs}, \btheta, \bx_0, \by_0)$ ({\color{blue} blue curve}).
\end{enumerate}
Note that multiple nudged phases with opposite nudging strength ($\pm\beta$) may be needed to reduce the bias of EP-based gradient estimation~\cite{laborieux2021scaling}. 
In practice, these implicit function values 
can be found with a properly chosen root finding algorithm. As done in many past works~\cite{scellierEquilibriumPropagationBridging2017,meulemans2022least}, we pick gradient descent dynamics over the energy function as an example here. 
Simple fixed-point iteration~\cite{laborieux2021scaling,laborieux2022holomorphic,scellier2023energy} or coordinate descent~\cite{scellier2024fast} may also be used depending on the models in use. 
In the free phase ($\beta = 0$), the system evolves according to (Figure~\ref{fig:ep-learning}B, black curve):
\begin{align}
    d_t {\bs}_t = -\partial_{{\bs}} E({\bs}_t, \btheta, \bx_0)\,, \label{eq:energy-ep-gd-free}
\end{align} 
until convergence to ${\bs}^{0}(\btheta)$, i.e., $\lim_{t \to \infty} {\bs}_t = {\bs}^{0}(\btheta)$. This temporal evolution is shown as 
the black curve in Figure~\ref{fig:ep-learning}B. In the nudged phase ($\beta > 0$), starting from the free equilibrium, 
the system follows (Figure~\ref{fig:ep-learning}B, {\color{blue}blue dotted curve}):
\begin{align}
    d_t {\bs}_t = -\partial_{{\bs}} E({\bs}_t, \btheta, \bx_0) - 
    \beta \partial_{{\bs}} C({\bs}_t, \by_0)\,,
    \label{eq:energy-ep-gd-nudged}
\end{align}
until convergence to ${\bs}^{\beta}(\btheta)$, i.e., $\lim_{t \to \infty} {\bs}_t = {\bs}^{\beta}(\btheta)$. The corresponding dynamical 
trajectory is depicted as the blue dotted curve in Figure~\ref{fig:ep-learning}B. Importantly, the gradient descent dynamics in Equation~\eqref{eq:energy-ep-gd-free} and~\eqref{eq:energy-ep-gd-nudged} are \emph{neither physical\footnote{i.e., the physical system does not need to implement gradient-descent dynamics explicitly; it only has to find a minimum of the energy landscape.}}, nor explicitly trained to match a target trajectory. As mentioned earlier, they serve as a computational tool to reach the 
solution of the variational problem. Because of these dynamics, the solutions of these variational problems are often called ``equilibrium states'' or ``fixed points''. The model corresponds to the free equilibrium, while the contrast between 
the nudged and free equilibria provides the necessary information to compute gradients through Equation~\eqref{eq:energy_ep_lr}.

\begin{figure}[t!]
    \centering
\includegraphics[width=1.0\textwidth]{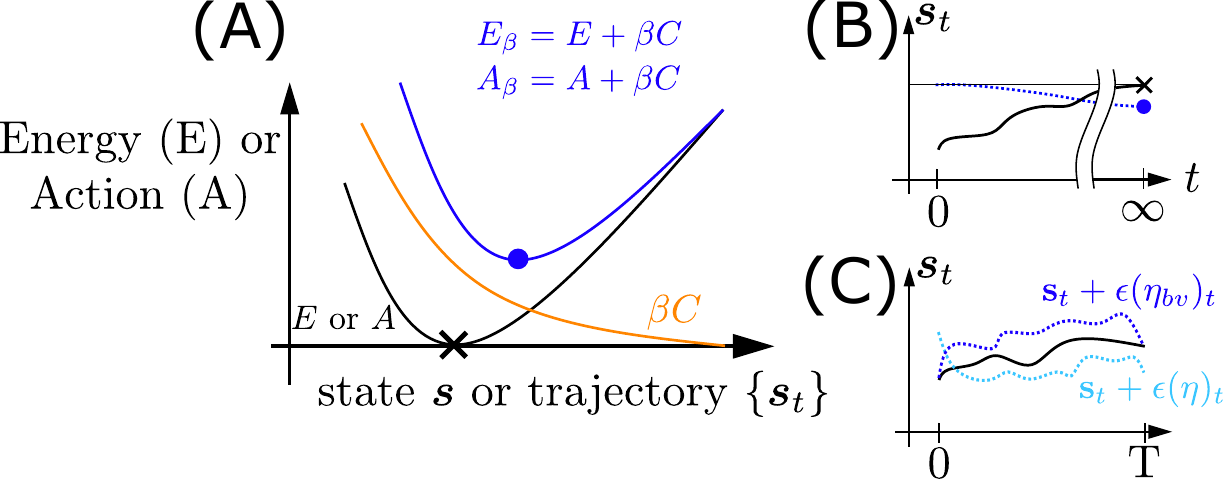}
    \caption{{\bfseries (A) EP trains variational systems.}  
    EP can train models that admit a variational 
        description, whether as a state vector ${\bs}$ or entire state trajectories (function of time) 
        $\{\bs_t\}$ that extremize a scalar function $E$ (or functional $A$), 
        represented by the black cross. To train these models to minimize a loss 
        $C$ ({\color{orange}orange curve}), one computes an extremum of the augmented energy 
        $E_\beta$ (or action $A_\beta$) represented by the blue curve. 
        These two variational objects enable efficient gradient computation, 
        leading to an improved energy (action) after a gradient update (dotted line). \textbf{Important caveat:} for the functional version, the trajectory is a stationary point (not necessarily a minimum) of the action; it is a true minimum only when considering \emph{boundary-vanishing variations} (see Panel C).
        For general variations, additional boundary terms appear in the gradient computation -- adapted from~\citep{ernoultRethinkingBiologicallyInspired2020}. {\bfseries (B) EP.} The free phase (black curve) and nudged phase (dotted {\color{blue}blue curve}) consist of 
        gradient descent on the energy and augmented energy, respectively. 
        These phases run sequentially, the nudged phase starts from the end-state of the free phase, and the learning rule uses only the 
        states at convergence. The trained model corresponds to the free 
        equilibrium state ({\color{black}black cross}). {\bfseries (C) LEP.} The free phase (black curve) corresponds to a trajectory satisfying the Euler-Lagrange 
        equations, while the nudged phase follows the Euler-Lagrange equation associated with the augmented action (dotted curve). For boundary-vanishing variation $\bs_t + \epsilon (\bmeta_{bv})_t$ ({\color{blue}dotted blue curve}), the gradient estimator is easy to compute because the boundary residual vanishes, but the boundary conditions are non-causal (both endpoints are constrained). On the contrary, for causal boundary conditions $\bs_t + \epsilon \bmeta_t$ ({\color{cyan}dotted cyan curve}), trajectories can be efficiently computed via forward integration, but the gradient estimator has boundary residuals that can be hard to compute.}
    \label{fig:ep-learning}
\end{figure}

\paragraph{Limitations.} The fact that we are only training the fixed point of the system highlights a major limitation of EP. It can only be used to train static input-output mappings (from $\bx_0$ to $\by_0$). More precisely, the equilibrium 
state defined by Equation~\eqref{eq:energy-ep-gd-free} 
represents a \emph{time-independent} configuration that encodes an implicit function ${\btheta} \mapsto {\bs}^0({\btheta})$ with static vector input $\bx_0$ and \emph{static} vector output $\by_0$. This fundamental constraint arises because energy function $E({\bs}, \btheta, \bx_0)$ is applied only to instantaneous states rather than temporal trajectories. 

A challenge lies in extending the variational principle underlying the framework of EP from vector spaces (where a single state ${\bs}$ is described variationally) to \emph{functional spaces}, where entire trajectories  $\{\bs_t : t \in [0,T]\}$ are described by a variational principle. Such an extension requires moving from energy functions defined on state vectors to an energy-like quantity defined on complete trajectories.

\subsection{\emph{Lagrangian EP}: variational principle in \emph{functional} space}
\label{subsec:lagrangian-ep}
Now, we generalise EP to describe entire trajectories through 
a variational problem, enabling us to train dynamical systems that map 
time-varying inputs to time-varying outputs. We refer to this extension as 
\emph{Lagrangian EP} (LEP). 
To achieve this extension, we revisit
the concept of augmented energy 
$E_\beta$ to an \emph{augmented action functional} $A_\beta$ that integrates 
over a time-varying ``energy-like'' 
quantity called the \emph{Lagrangian} $L_0$ \citep{scellierDeepLearningTheory2021}:
\begin{align}
    \label{eqdef:L_beta}
    \underbrace{A_\beta[\bs, \btheta, \bx, \by]}_{\text{augmented action }} &\defn \int_0^T \underbrace{(L_{0}(\bs_t, \dot{\bs}_t, \btheta, \bx_t)
    + \beta~c(\bs_t, \by_t))}_{L_{\beta}(\bs_t, \dot{\bs}_t, \btheta, \bx_t, \by_t)} \dt \\
    &= \underbrace{\int_0^T L_{{0}}(\bs_t, \dot{\bs}_t, \btheta, \bx_t) \dt}_{\defn \text{ action }A[\bs, \btheta, \bx]}\,~ \notag
     +~~ \beta \underbrace{\int_0^T c(\bs_t, \by_t)\dt}_{\defn \text{ cost }C[\bs,\by]}\,.
\end{align}
Here $A[\bs,\btheta, \bx]$ is a functional that serves as the temporal counterpart
of the 
energy function $E$, operating on entire trajectories\footnote{Note that we don't have to write $\dot{\bs}$ as input of the action $A$, because it can be derived from $\bs$ via the time derivative transformation}. It integrates the 
Lagrangian $L_{0}(\bs_t, \dot{\bs}_t, \btheta, \bx_t)$ over time, where 
the Lagrangian takes as input the state $\bs_t$, its temporal derivative (velocity) $\dot{\bs}_t$, and the time-varying input $\bx_t$. 

The augmented action functional $A_\beta$ is the temporal analogue of $E_\beta$. 
It integrates the augmented Lagrangian $L_{\beta}$ that extends the Lagrangian by including an additional nudging term 
$\beta c(\bs_t, \by_t)$. The augmented action functional $A_\beta[\bs]$ maps 
a trajectory $\bs \defn \{\bs_t : t \in [0,T]\}$ to a scalar value, 
generalizing the scalar-valued energy functions of EP to 
functional-valued quantities that capture temporal dynamics. For notational simplicity, we omit the dependence on inputs 
$\bx$ and targets $\by$ (or their time-indexed versions $\bx_t$ and $\by_t$) 
whenever the context is clear. 

\paragraph{Variational formulation and functional derivatives.}
The action functional enables us to define the variational problems that 
generalize EP to the temporal domain. Following standard variational calculus~\citep{olverCalculusVariations2022}, we define the \emph{functional derivative} (or \emph{variational derivative}) $\delta_{\bs} A_\beta$ through the directional derivative with respect to trajectory variations $\bmeta \defn \{\bmeta_t : t \in [0,T]\}$:
\begin{align*}
    \delta_{\bs} A_\beta \cdot \bmeta \defn \left. d_\epsilon \right|_{\epsilon=0} A_\beta[\bs + \epsilon \bmeta]\,,
\end{align*}
where $\delta_{\bs} A_\beta$ denotes the functional gradient with respect to the trajectory and $\cdot$ denotes the standard $L^2$ inner product on function space, i.e., $\bmeta \cdot \bmeta' := \int_0^T (\bmeta_t)^\top (\bmeta_t') \, \dt$.
With this notation, the \emph{nudged variational problem} is
\begin{align*}
    \delta_{\bs} A_\beta = 0 \quad \Leftrightarrow \quad \delta_{\bs} A_\beta \cdot \bmeta = 0 \text{ for all smooth variations } \bmeta\ \quad \text{s.t.} \quad \bmeta_0=\bmeta_T=0\,.
\end{align*}
In particular, for $\beta=0$, the \emph{free variational 
problem} is defined as $\delta_{\mathbf{s}} A_0 = 0$, corresponding to the 
system's natural dynamics without nudging. Unlike EP, where the variational problems are typically solved 
through gradient descent dynamics, these functional variational problems can 
be solved more directly using the Euler-Lagrange equations. The corresponding Euler-Lagrange expression is defined as
\begin{align*}
    \EL(t, \btheta,\beta) &:= \partial_{\bs} L_{\beta}(\bs_t, \dot{\bs}_t,\btheta) 
    - d_t\partial_{\dot{\bs}} L_{\beta}(\bs_t, \dot{\bs}_t, \btheta) \\
    &= \partial_{\bs} L_{0}(\bs_t, \dot{\bs}_t, \btheta) 
    - d_t\partial_{\dot{\bs}} L_0(\bs_t, \dot{\bs}_t, \btheta) 
    + \beta \partial_{\bs}c(\bs_t)\,.
\end{align*}
The following classic result, namely the principle of stationary action~\citep{olverCalculusVariations2022}, generalized to arbitrary boundary conditions, establishes the fundamental connection between the variational 
formulation and the Euler-Lagrange equation.
\begin{lemma}[Euler-Lagrange solutions and the action functional~\citep{olverCalculusVariations2022}]
\label{lemma:euler-lagrange}
Let $\mathbf{s}^\beta({\btheta}) \defn \{\mathbf{s}_t^\beta({\btheta}) : t \in [0,T]\}$ be a trajectory solution of the Euler-Lagrange equation $\EL(t, {\btheta}, \beta) = 0$ for 
all $t \in [0,T]$, and let $\bmeta \defn \{\bmeta_t : t \in [0,T]\}$ be any smooth variation. Then (see proof in Appendix~\ref{appx:sec:proof-generalized-lemma}):
\begin{enumerate}
    \item \textbf{Boundary-vanishing variations:} For any variation $\bmeta_{bv}$ that vanishes at the boundaries, i.e., $(\bmeta_{bv})_0 = (\bmeta_{bv})_T = \mathbf{0}$, $\mathbf{s}^\beta({\btheta})$ 
is a critical point of the action functional $A_\beta[\mathbf{s}]$:
\begin{align*}
    \delta_{\bs} A_\beta \cdot \bmeta_{bv} = \left. d_\epsilon \right|_{\epsilon=0} A_\beta[\mathbf{s}^\beta + \epsilon \bmeta_{bv}] = 0\,.
\end{align*}
    \item \textbf{General formula for arbitrary variations:} For an arbitrary variation $\bmeta$ (not necessarily vanishing at the boundaries), the directional derivative of the action is given by:
\begin{align}
    \delta_{\bs} A_\beta \cdot \bmeta = \left. d_\epsilon \right|_{\epsilon=0} A_\beta[\mathbf{s}^\beta + \epsilon \bmeta] = \left[\bmeta_t^\top \partial_{\dot{\bs}} L_\beta(\bs_t^\beta, \dot{\bs}_t^\beta, \btheta)\right]_0^T\,.
    \label{eq:generalized-variation}
\end{align}
When $\bmeta_0 \neq \mathbf{0}$ or $\bmeta_T \neq \mathbf{0}$, $\mathbf{s}^\beta({\btheta})$ is not generally a critical point. The boundary terms must be handled separately depending on the specific boundary conditions imposed on the problem.
\end{enumerate}
\textit{Note (Parametric perturbations).}\label{remark:generalized-variation}
A similar result holds when the linear perturbation $\epsilon \bmeta$ is replaced by a general smooth parametric perturbation $\bmeta(\epsilon)$ with $\bmeta(0) = \mathbf{0}$: the variation $\bmeta$ in Eq.~\eqref{eq:generalized-variation} is simply replaced by $\left.\partial_\epsilon\right|_{\epsilon=0} \bmeta(\epsilon)$ (see proof in Appendix~\ref{appx:sec:proof-generalized-lemma}). This generalization is the result that will be used to prove our central result, Theorem~\ref{thm:general-ep}, where we evaluate $\beta$ and $\btheta$ perturbations of the trajectory $\bs^\beta(\btheta)$.
\end{lemma}

This principle establishes that Euler-Lagrange solutions correspond to critical points of the action functional for boundary-vanishing variations (Case 1). This variational property enables extending EP to temporal domains: instead of computing gradients through explicit differentiation, we can approximate them by contrasting two EL trajectories -- the free trajectory $\bs^0(\btheta)$ and the $\beta$-nudged trajectory $\bs^\beta(\btheta)$ -- analogous to the two phases in EP (Section~\ref{sec:ep}).

However, for arbitrary variations (Case 2), the nudged trajectory must satisfy the same boundary conditions as the free trajectory at both $t=0$ and $t=T$. This defines a two-point boundary value problem that cannot be solved by forward integration from initial conditions. We call boundary conditions that only constrain the initial state \emph{causal}, since they allow forward-in-time computation; conversely, boundary conditions that constrain both endpoints are \emph{non-causal}. Previous work~\citep{scellierDeepLearningTheory2021, kendall2021gradient} implicitly assumed non-causal boundary conditions, leaving this difficulty of satisfying them unaddressed. Relaxing the boundary conditions to causal ones permits efficient trajectory computation, but may introduce additional terms in the gradient formula -- see Theorem~\ref{thm:general-ep}.

To understand this tradeoff between causal trajectory computation and tractable gradient formulas, we derive LEP for \emph{arbitrary boundary conditions}. Theorem~\ref{thm:general-ep} provides our primary result: it explicitly characterizes both the learning rule and the boundary terms that arise for any choice of boundary conditions.

\begin{theorem}[LEP for arbitrary boundary conditions]\label{thm:general-ep}
    Let $t \mapsto \bs_t^\beta(\btheta)$ denote the solution to the 
    Euler-Lagrange equation $\EL(t, \btheta, \beta) = 0$ with arbitrary 
    boundary conditions. The gradient of the objective functional with 
    respect to $\btheta$ is given by (with all $\beta$-derivatives evaluated at $\beta=0$):
    \begin{align}
        d_{\btheta} C[\bs^0({\btheta})] 
        &= d_\beta \left(\int_0^T \partial_{\btheta} L_{\beta}\left( \bs_t^{\beta}, \dot{\bs}_t^{\beta}, \btheta\right) \dt \right)  
        \label{eq:lagrangian-ep-lr}\\
        &+ \underbrace{
            \left[ 
                \left(\partial_{\btheta} \bs_t^{0}\right)^\top  d_\beta\partial_{\dot{\bs}} L_{\beta}\left(\bs_t^{\beta}, \dot{\bs}_t^{\beta}, \btheta\right) 
                - \left(d_{\btheta}\partial_{\dot{\bs}} L_{0}\left(\bs_t^{0}, \dot{\bs}_t^{0}, \btheta\right)\right)^\top \partial_\beta \bs_t^{\beta}
            \right]_0^T
        }_{\text{boundary residuals}}\,.
        \label{eq:boundary_residuals}
    \end{align}
\end{theorem}
Note that we have omitted the explicit $\btheta$ dependence in the state 
trajectories $\bs_t^\beta(\btheta)$ and their derivatives 
$\dot{\bs}_t^\beta(\btheta)$ for notational simplicity. We adopt this convention throughout the remainder of this work.


\paragraph{Gradient formula interpretation.} The first term on the right-hand side of 
\eqref{eq:lagrangian-ep-lr} directly generalizes the EP 
learning rule (Eq.~\ref{eq:energy_ep_lr}): instead of computing differences between energy function parameters
derivatives at two fixed points, we integrate differences between 
Lagrangian parameter derivatives over entire trajectories. This integration reflects 
the fact that we are now training the complete temporal evolution rather 
than an equilibrium state.

The second term, which we call \emph{boundary residuals}, represents a 
fundamental difficulty that arises from extending EP to temporal domains. 
These terms emerge from the integration by parts required in the 
derivation of Theorem~\ref{thm:general-ep} (see proof in Appendix~\ref{appx:sec:proof-general-ep}) and depend on the boundary conditions imposed on the trajectories 
$\bs^\beta(\btheta)$. The fact that we have not yet specified these 
boundary conditions is why we refer to our theorem as a ``generalization to arbitrary boundary conditions''. As we explore in the following sections, different choices of boundary conditions yield different learning algorithms.

\paragraph{Implementation procedure.} Focusing on the first term suggests 
a two-phase procedure analogous to EP, as illustrated in 
Figure~\ref{fig:ep-learning}:
\begin{enumerate}
    \item \textbf{Free phase}: Compute the trajectory $\bs^{0}(\btheta)$ 
    (black cross in Fig~\ref{fig:ep-learning}A) that is a stationary point of the action 
    functional $A_0$ (black curve in Fig~\ref{fig:ep-learning}A) by solving 
    the associated Euler-Lagrange equation $\text{EL}(\btheta,0) = 0$ over 
    the time interval $[0,T]$. The temporal evolution is highlighted by the 
    black curve in Figure~\ref{fig:ep-learning}C.
    \item \textbf{Nudged phase}: Compute the trajectory $\bs^{\beta}({\btheta})$ 
    (blue dot in Fig~\ref{fig:ep-learning}A) for a small positive value of 
    $\beta$ by solving the perturbed Euler-Lagrange equation 
    $\text{EL}(\btheta,\beta) = 0$, corresponding to the minimum of the 
    augmented action $A_\beta$ (blue curve in Fig~\ref{fig:ep-learning}A). 
    The corresponding dynamics are shown as the dotted blue curve in 
    Figure~\ref{fig:ep-learning}C.
    \item \textbf{Learning rule}: Estimate the gradient using the finite 
    difference approximation of the first term in \eqref{eq:lagrangian-ep-lr}, 
combined with appropriate handling of the boundary residuals (see below in Section \ref{subsec:instantiation}).
\end{enumerate}
\paragraph{Computational challenges.}
Unlike standard EP, Lagrangian EP faces two computational challenges controlled by the choice of boundary conditions:
\begin{enumerate}
    \item \textbf{Boundary residuals in the learning rule.} The boundary residuals in Eq.~\eqref{eq:boundary_residuals} involve $\btheta$-derivatives like $\partial_{\btheta} \bs^0_T$ that would require differentiating through the ODE solver -- defeating the purpose of this work.
    \item \textbf{Non-causal boundary conditions.} Even when boundary residuals vanish, as previous work assumed~\citep{scellierDeepLearningTheory2021, kendall2021gradient}, computing the nudged trajectory $\bs^\beta(\btheta)$ presents its own difficulties. For boundary residuals to vanish, the nudged trajectory must satisfy the same boundary conditions as the free trajectory (boundary-vanishing variations). This means one must find a trajectory that both satisfies the Euler-Lagrange equations \emph{and} matches prescribed values at \emph{both} endpoints -- a non-causal boundary value problem (see Section~\ref{subsec:cfvp}).
\end{enumerate}
These challenges motivate the search for boundary conditions that are {\bfseries both causal and free of boundary residuals}, as we explore in Section~\ref{subsec:instantiation}.


\subsection{Instantiations of LEP}
\label{subsec:instantiation}

\begin{figure}[t!]
    \centering
    \includegraphics[width=1.0\textwidth]{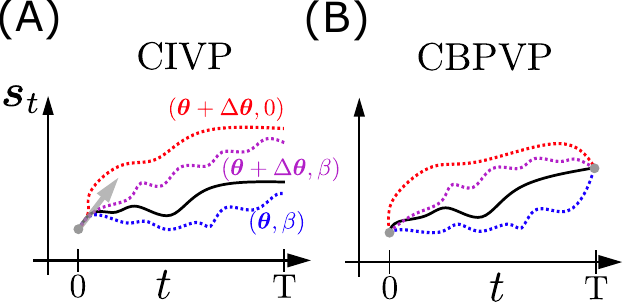}
\caption{\textbf{Different boundary condition formulations for LEP.} 
    The two panels use a consistent color scheme: \textcolor{black}{black curves} represent the free trajectory $\bs^0(\btheta)$ used for inference, and \textcolor{blue}{blue dotted curves} show the $\beta$-nudged trajectories $\bs^\beta(\btheta)$ used for learning. Boundary conditions are depicted in \textcolor{gray}{grey}, with dots for positions and arrows for velocities. To illustrate how boundary conditions constrain the entire family of trajectories, we also display $\btheta$-perturbed trajectories $\bs^0(\btheta + \Delta\btheta)$ (\textcolor{red}{red dotted curve}) and combined perturbations $\bs^\beta(\btheta + \Delta\btheta)$ (\textcolor{violet}{purple dotted curve}).
    \textbf{(A) Constant Initial Value Problem (CIVP).} All trajectories share 
    the same initial conditions $(\bs_0, \dot{\bs}_0)=(\boldsymbol{\alpha}_0, \boldsymbol{\gamma}_0)$, depicted as a \textcolor{gray}{grey dot} for the initial position $\boldsymbol{\alpha}_0$ and a \textcolor{gray}{grey arrow} for the initial velocity $\boldsymbol{\gamma}_0$,
    but evolve differently due to parameters or nudging perturbations.
    \textbf{(B) Constant Boundary Position Value Problem (CBPVP).} All trajectories 
    satisfy boundary conditions requiring fixed positions at $t=0$ and $t=T$, $(\bs_0, \bs_T)=(\boldsymbol{\alpha}_0, \boldsymbol{\alpha}_T)$, depicted as \textcolor{gray}{grey dots}, but their dynamics differ due to parameters or nudging perturbation.}
    \label{fig:perturbations}
\end{figure}

In this section, we demonstrate how to instantiate LEP
by constructing the function 
$t \mapsto \bs_t^\beta(\boldsymbol{{\btheta}})$ through different boundary 
specifications. We first consider the Constant Boundary Position Value Problem (CBPVP), which corresponds to the boundary-value-problem assumption made by~\citep{scellierDeepLearningTheory2021} and~\citep{kendall2021gradient}. We then consider the Constant Initial Value Problem (CIVP) as a natural causal alternative. As we show, each resolves one of the two computational challenges identified above, but not both.
Importantly, boundary conditions must be specified for an entire \emph{family} of trajectories---those corresponding to different values of $\btheta$ and $\beta$. Figure~\ref{fig:perturbations} illustrates how different types of boundary conditions constrain these families: some fix both endpoints, others fix the initial state across all trajectories, and so on.

\subsubsection{Constant Boundary Position Value Problem (CBPVP) on position}
\label{subsec:cfvp}
The boundary-value-problem assumption made by~\citep{scellierDeepLearningTheory2021} and~\citep{kendall2021gradient} corresponds to the \emph{Constant Boundary Position Value Problem}, where 
trajectories are constrained by conditions at both temporal boundaries:
\begin{align*}
    \forall t \in [0,T], \quad t \mapsto  \bs_{\bda,t}^\beta(\boldsymbol{{\btheta}}, (\boldsymbol{\alpha}_0, \boldsymbol{\alpha}_T)) 
    \text{ satisfies:} \quad
    \begin{cases}
        \text{EL}(t, \boldsymbol{{\btheta}}, \beta) = 0 \\
        \bs_{\bda,0}^\beta(\boldsymbol{{\btheta}}) = \boldsymbol{\alpha}_0 \\
        \bs_{\bda,T}^\beta(\boldsymbol{{\btheta}}) = \boldsymbol{\alpha}_T
    \end{cases}
\end{align*}
where $\boldsymbol{\alpha}_0$ and $\boldsymbol{\alpha}_T$ now represent the fixed 
positions at the initial and final times, respectively. This formulation is 
depicted in Figure~\ref{fig:perturbations}B, where all trajectories connect 
the same boundary points but follow different internal dynamics. Applying Theorem~\ref{thm:general-ep} to this boundary condition choice yields 
a direct instantiation of the general gradient formula with significant 
simplification due to the elimination of boundary residual terms.

\begin{corollary}[Gradient estimator for CBPVP]
\label{corollary:CBPVP_gradient}
The gradient of the objective functional for 
$\bs_{\bda}^\beta(\boldsymbol{{\btheta}}, (\boldsymbol{\alpha}_0, \boldsymbol{\alpha}_T))$ is given by:
\begin{align}
    d_{\boldsymbol{{\btheta}}} C[\bs_{\bda}^0(\boldsymbol{{\btheta}}, (\boldsymbol{\alpha}_0, \boldsymbol{\alpha}_T))] 
    &= \lim_{\beta \to 0} \frac{1}{\beta} \Delta^{\text{CBPVP}}(\beta)\,,
    \label{eq:CBPVP_gradient}
\end{align}
where the finite difference gradient estimator simplifies to:
\begin{align*}
    \Delta^{\text{CBPVP}}(\beta) &:= \int_0^T \Big[\partial_{\boldsymbol{{\btheta}}} L_{\beta}(\bs_{\bda,t}^{\beta}, \dot{\bs}_{\bda,t}^{\beta}, \btheta) 
    - \partial_{\boldsymbol{{\btheta}}} L_{0}(\bs_{\bda,t}^{0}, \dot{\bs}_{\bda,t}^0, \btheta)\Big] \dt\,.
\end{align*}
\end{corollary}

\paragraph{No boundary residuals, but non-causal boundary conditions.} 
The CBPVP formulation resolves the boundary residual challenge: both endpoints are fixed independently of $\btheta$ and $\beta$, causing all residual terms to vanish. This yields a simple gradient estimator that only requires integrating differences between Lagrangian derivatives over the two trajectories (Eq.~\eqref{eq:CBPVP_gradient}). However, given only the two endpoint conditions $\boldsymbol{\alpha}_0$ and $\boldsymbol{\alpha}_T$, the Euler-Lagrange equation cannot be solved by forward integration from an initial condition. Instead, one must solve a two-point boundary value problem---finding a trajectory that satisfies both the Euler-Lagrange equations \emph{and} the prescribed endpoint constraints.

As an alternative to Euler-Lagrange forward integration, one can exploit the variational characterization to solve this two-point boundary value problem: by Lemma~\ref{lemma:euler-lagrange}, $\bs_{\bda}^\beta$ is equivalently the minimizer of the action subject to boundary constraints:
\begin{align*}
    \bs_{\bda}^\beta(\boldsymbol{{\btheta}}, (\boldsymbol{\alpha}_0, \boldsymbol{\alpha}_T)) 
    = \arg\min_{\bs} A_\beta[\bs] \quad 
    \text{subject to} \quad \bs_{\bda,0}^\beta = \boldsymbol{\alpha}_0, \; \bs_{\bda,T}^\beta = \boldsymbol{\alpha}_T\,.
\end{align*}
This optimization can be solved via gradient descent (or other root finding algorithm) on the action functional, which takes the form of a partial differential equation~\citep{olverCalculusVariations2022}:
\begin{align*}
    d_\tau  \bs_{\bda} = -\delta_{\bs} A_\beta
    = -\text{EL}(t, \boldsymbol{{\btheta}}, \beta)
    \qquad\text{subject to} \quad \bs_{\bda,0} = \boldsymbol{\alpha}_0, \; \bs_{\bda,T} = \boldsymbol{\alpha}_T\,,
\end{align*}
where $\tau$ is an artificial optimization time and $\delta_{\bs} A_\beta$ is the functional gradient. In practice, the physical time $t \in [0,T]$ is discretized into $N$ bins, turning the trajectory into a vector of size $N \times d_s$. The system then evolves iteratively in $\tau$ -- analogous to the root-finding algorithms used in standard EP, but applied to this much larger state space -- until the trajectory converges to a critical point where $\text{EL}(t, \btheta, \beta) = 0$. As we quantify in Table~\ref{tab:complexity}, this iterative solver dominates the overall cost at $\mathcal{O}(K N d_s^2)$, where $K$ grows with $N$ and $d_s$.

CBPVP eliminates boundary residuals but at the cost of non-causal trajectory computation, making it less appealing than LEP instantiations that would require simple forward passes through an ODE.

\begin{remark}[Unconstrained action minimization]
If one is willing to accept iterative optimization---rather than forward integration via Euler-Lagrange equations---then boundary conditions need not be imposed at all. Minimizing the action functional \emph{without} boundary constraints yields a variational formulation analogous to standard EP, where boundary residuals vanish entirely in Theorem~\ref{thm:general-ep}. However, this approach inherits the same non-causal drawbacks as CBPVP and is in fact more expensive, since the full trajectory \emph{including} its endpoints becomes part of the optimization variables. We elaborate on this observation in Appendix~\ref{app:unconstrained-action}.
\end{remark}

\subsubsection{Constant Initial Value Problem (CIVP)}
\label{subsec:civp}
A natural attempt to restore causality is the \emph{Constant Initial Value Problem} (CIVP), where trajectories are constructed through straightforward forward integration:
\begin{align*}
    \forall t \in [0,T] \quad t \mapsto \bs_{\fwda,t}^\beta(\boldsymbol{{\btheta}}, (\boldsymbol{\alpha}_0, \boldsymbol{\gamma}_0)) 
    \text{ satisfies:} \quad
    \begin{cases}
        \text{EL}(t, \boldsymbol{{\btheta}}, \beta) = 0 \\
        \bs_{\fwda,0}^\beta(\boldsymbol{{\btheta}}) = \boldsymbol{\alpha}_0 \\
        \dot{\bs}_{\fwda,0}^\beta(\boldsymbol{{\btheta}}) = \boldsymbol{\gamma}_0
    \end{cases}
\end{align*}
where $\boldsymbol{\alpha}_0 \in \mathbb{R}^d$ and $\boldsymbol{\gamma}_0 \in \mathbb{R}^d$
are the initial position and velocity conditions at $t=0$, respectively. This formulation 
defines a family of trajectories that all originate from the same initial state 
but evolve according to different dynamics due to parameter or nudging perturbations, 
as illustrated in Figure~\ref{fig:perturbations}A. Unlike CBPVP, the Euler-Lagrange equation can be directly integrated forward from the initial conditions---the trajectory computation is therefore causal and efficient at $\mathcal{O}(N d_s^2)$. Applying Theorem~\ref{thm:general-ep} to this boundary condition choice yields 
a direct instantiation of the general gradient formula with some simplification due to the fixed initial conditions.

\begin{corollary}[Gradient estimator for CIVP]
\label{corollary:civp_gradient}
The gradient of the objective functional for 
$\bs_{\fwda}^0(\boldsymbol{{\btheta}}, (\boldsymbol{\alpha}_0, \boldsymbol{\gamma}_0))$ is given by:
\begin{align*}
    d_{\boldsymbol{{\btheta}}} C[\bs_{\fwda}^0(\boldsymbol{{\btheta}}, (\boldsymbol{\alpha}_0, \boldsymbol{\gamma}_0))]
    &= \lim_{\beta \to 0} \Delta^{\text{CIVP}}(\beta)\,,
\end{align*}
where
\begin{align}
    \Delta^{\text{CIVP}}(\beta) &:= \frac{1}{\beta} \Bigg[ 
    \int_0^T \Big[\partial_{\boldsymbol{{\btheta}}} L_{\beta}(\bs_{\fwda,t}^{\beta}, \dot{\bs}_{\fwda,t}^{\beta}, \btheta) 
    - \partial_{\boldsymbol{{\btheta}}} L_{ 0}(\bs_{\fwda,t}^{0}, \dot{\bs}_{\fwda,t}^0, \btheta)\Big] \dt \nonumber \\
    &\quad + \underbrace{\left(\partial_{\boldsymbol{{\btheta}}} \bs_{\fwda,T}^{0}\right)^{\top}}_{\text{costly residual}}
    \Big(\partial_{\dot{\bs}} L_{\beta}(\bs_{\fwda,T}^{\beta}, \dot{\bs}_{\fwda,T}^{\beta}, \btheta) 
    - \partial_{\dot{\bs}} L_{ 0}(\bs_{\fwda,T}^{0}, \dot{\bs}_{\fwda,T}^{0}, \btheta)\Big) \nonumber \\
    &\quad - \underbrace{\left(d_{\boldsymbol{{\btheta}}}\partial_{\dot{\bs}} L_{0}(\bs_{\fwda,T}^{0}, \dot{\bs}_{\fwda,T}^{0}, \btheta)\right)^{\top}}_{\text{costly residual}}  
    \left(\bs_{\fwda,T}^{\beta} - \bs_{\fwda,T}^{0}\right) \Bigg]\,.
    \label{eq:civp_estimator}
\end{align}
\end{corollary}

\paragraph{Causal boundary conditions, but intractable boundary residuals.} 
While CIVP restores causal forward integration, it suffers from 
significant computational limitations due to the boundary residual terms in 
Eq.~\eqref{eq:civp_estimator}. In particular, the remaining residuals at time $T$ involve derivatives of the 
trajectory with respect to parameters ($\partial_{\boldsymbol{{\btheta}}} \bs_{\fwda,T}^{0}$) 
and mixed derivatives of the Lagrangian 
($d_{\boldsymbol{{\btheta}}}\partial_{\dot{\bs}} L_{0}$), 
which cannot be efficiently computed using finite differences due to the 
high dimensionality of the parameter space (see Section~\ref{app:complexity:civp} for a detailed complexity analysis showing these terms require $\mathcal{O}(N d_s^3)$ time and $\mathcal{O}(N d_s)$ memory). The only simplification occurs 
at $t=0$, where the boundary residuals vanish due to the fixed initial conditions, 
but this is insufficient to yield a practical learning algorithm.




\subsubsection{Towards a practical implementation of LEP}
\paragraph{Designing efficient algorithms.} Table~\ref{tab:complexity} quantifies the trade-off between CBPVP and CIVP in terms of computational complexity, where $N$ denotes the number of discrete time steps, $d_s$ the state dimension, $d_\theta$ the number of learnable parameters, and $K$ the number of iterations required for the 
boundary value problem 
solver convergence. For CBPVP, gradient computation is efficient at $\mathcal{O}(N d_\theta)$ with only $\mathcal{O}(d_\theta)$ memory, but the iterative BVP solver dominates at $\mathcal{O}(K N d_s^2)$ time, where $K$ can be expected to be a growing quantity of $N$ and $d_s$. For CIVP, trajectory computation is efficient at $\mathcal{O}(N d_s^2)$, but evaluating the boundary residuals requires a complexity of $\mathcal{O}(N d_s^3)$ and storing intermediate states, incurring $\mathcal{O}(N d_s)$ memory---when done using backpropagation through time (see Appendix~\ref{app:complexity} for details).

This motivates the search for boundary conditions that are both causal and free of boundary residuals. In the following sections, we demonstrate that the Parametric Final Value Problem (PFVP) formulation, which underlies the RHEL algorithm, achieves both properties for time-reversible systems—attaining efficient $\mathcal{O}(N d_s^2)$ dynamics and $\mathcal{O}(N d_\theta)$ gradient computation without the bottlenecks of either CIVP or CBPVP.

\begin{table}[h]
\centering
\renewcommand{\arraystretch}{1.6}
\resizebox{\textwidth}{!}{
\begin{tabular}{@{}l cc cc c c l@{}}
\toprule
& \multicolumn{2}{c}{\textbf{Time Complexity}} & \multicolumn{2}{c}{\textbf{Memory}} & & & \\
\cmidrule(lr){2-3} \cmidrule(lr){4-5}
\textbf{Method} & Dynamics & Gradient & Dynamics & Gradient & \textbf{Forward-only} & \textbf{Streaming} & \textbf{Bottleneck} \\
\midrule
CIVP & $\mathcal{O}(N d_s^2)$ & $\mathcal{O}(N d_s^3)$ & $\mathcal{O}(d_s)$ & $\dom{\mathcal{O}(N d_s)}$ & \textcolor{dominant}{\textsf{x}} & \textcolor{good}{\checkmark} & BPTT memory \\
CBPVP & $\dom{\mathcal{O}(K N d_s^2)}$ & $\mathcal{O}(N d_\theta)$ & $\dom{\mathcal{O}(N d_s)}$ & $\mathcal{O}(d_\theta)$ & \textcolor{good}{\checkmark} & \textcolor{dominant}{\textsf{x}} & BVP iterations \\
PFVP/RHEL & $\good{\mathcal{O}(N d_s^2)}$ & $\good{\mathcal{O}(N d_\theta)}$ & $\good{\mathcal{O}(d_s)}$ & $\good{\mathcal{O}(d_\theta)}$ & \textcolor{good}{\checkmark} & \textcolor{good}{\checkmark} & None \\
\bottomrule
\end{tabular}}
\caption{Computational complexity comparison. \textcolor{dominant}{\textbf{Red}} indicates the dominant cost that makes the method impractical. \textcolor{good}{Green} indicates efficient scaling. See Appendix~\ref{app:complexity} for detailed derivation.}
\label{tab:complexity}
\end{table}

\paragraph{Designing easy-to-implement algorithms.}
Beyond computational efficiency in time and memory, a central appeal of LEP (and EP) is that, under certain conditions, it can be \emph{forward-only}.

An algorithm is \emph{forward-only} if it only requires running the same physical system forward in time---no separate backward pass through a computational graph is needed.
In practice, gradient computation reuses the same dynamical system as inference, requiring only two forward passes: a free phase and a nudged phase.

As summarized in Table~\ref{tab:complexity}, CIVP is \emph{not} forward-only: it requires an explicit backward pass through the stored computational graph to evaluate the boundary residual terms of the gradient estimator.
CBPVP is forward-only, since both phases run the same iterative boundary-value-problem solver and no separate backward circuit is needed, but at the cost of an expensive iterative procedure.
As we show in Section~\ref{sec:rhel_is_ep}, PFVP/RHEL satisfies the forward-only property while avoiding this overhead: both the free and echo phases consist of pure forward integration, with no iterative solver required (see Appendix~\ref{app:complexity} for a detailed comparison).

In LEP, a further refinement of the forward-only property matters: \emph{streaming}. An algorithm is streaming if it can process temporal data sequentially from $t=0$ to $t=T$ without requiring access to the entire time horizon at once. As shown in Table~\ref{tab:complexity}, causal boundary conditions (CIVP and PFVP/RHEL) naturally enable streaming, while CBPVP's non-causal boundary conditions, despite being forward-only, require all $N$ time steps to be processed simultaneously, precluding streaming operation.

\vspace{0.5em}





\section{Recurrent Hamiltonian Echo Learning}
\label{sec:rhel}
Recurrent Hamiltonian Echo Learning (RHEL) presents a fundamentally different 
approach to temporal credit assignment compared to the variational formulations 
discussed in the previous section. Unlike 
EP methods that 
rely on variational principles and careful specification of boundary conditions, 
RHEL operates directly on the dynamics of Hamiltonian physical systems without 
requiring an underlying action functional or boundary value problem formulation.

\subsection{Hamiltonian system formulation}

In RHEL, the system to be trained is described by a Hamiltonian function 
$H(\bPhi_t, {\btheta}, \bx_t)$, where $\bPhi_t({\btheta}) \in \mathbb{R}^{2d}$ represents 
the complete state of the system at time $t$. This state vector is composed of 
both position and momentum coordinates:
\begin{equation*}
\bPhi_t := \begin{pmatrix} \bs_t \\ \bp_t \end{pmatrix} \in \mathbb{R}^{2d}\,,
\end{equation*}
where $\bs_t \in \mathbb{R}^d$ represents the position coordinates and 
$\bp_t \in \mathbb{R}^d$ represents the momentum coordinates.

The evolution of the system follows Hamilton's equations of motion:
\begin{equation*}
d_t \bPhi_t = \bJ \cdot \partial_{\bPhi} H(\bPhi_t, {\btheta}, \bx_t)\,,
\end{equation*}
where $\bJ$ is the canonical symplectic matrix:
\begin{equation*}
\bJ := \begin{bmatrix} 
\bm{0} & \bm{I} \\ 
-\bm{I} & \bm{0} 
\end{bmatrix} \in \mathbb{R}^{2d \times 2d}\,.
\end{equation*}

A crucial requirement for RHEL is that the Hamiltonian must be time-reversible, 
meaning it satisfies:
\begin{equation*}
H(\bSigma_z \bPhi_t, {\btheta}, \bx_t) = H(\bPhi_t, {\btheta}, \bx_t)\,,
\end{equation*}
where $\bSigma_z$ is the momentum-flipping operator:
\begin{equation*}
\bSigma_z := \begin{bmatrix} 
\bm{I} & \bm{0} \\ 
\bm{0} & -\bm{I} 
\end{bmatrix}\,.
\end{equation*}
This time-reversibility property ensures that the system can exactly retrace 
its trajectory when the momentum is reversed, which is fundamental to the 
echo mechanism.

\subsection{Two-phase learning procedure}

RHEL implements a two-phase learning procedure that leverages the time-reversible 
nature of Hamiltonian systems. Notably, this procedure does not require solving 
variational problems or specifying complex boundary conditions.

\textbf{Forward phase:} The first phase computes the natural evolution of the 
system from initial conditions. For $t \in [0, T]$, the trajectory 
$t \mapsto \bPhi_t({\btheta}, (\balpha_0, \bmu_0)^\top)$ satisfies:
\begin{equation*}
\begin{cases}
\partial_t \bPhi_t = \bJ ~ \partial_{\bPhi} H(\bPhi_t, {\btheta}, \bx_t) \\
\bPhi_0 = \begin{pmatrix} \balpha_0 \\ \bmu_0 \end{pmatrix}
\end{cases}
\end{equation*}
This phase corresponds to the system's natural dynamics without any learning 
signal and produces the model's prediction.

\textbf{Echo phase:} The second phase begins by flipping the momentum of the 
final state and then evolving the system backward in time with a small nudging 
perturbation. For $t \in [0, T]$, the echo trajectory 
$t \mapsto \bPhi^e_t({\btheta}, \bSigma_z \bPhi_T({\btheta}))$ satisfies:
\begin{equation}
\begin{cases}
\partial_t \bPhi^e_t = \bJ ~ \partial_{\bPhi} H(\bPhi^e_t, {\btheta}, \bx_{T-t})
- \beta \bJ ~ \partial_{\bPhi} c(\bPhi^e_t, \by_{T-t}) \\
\bPhi^e_0 = \bSigma_z \bPhi_T({\btheta})
\end{cases}
\label{def:hes_echo}
\end{equation}
where $\beta > 0$ is a small nudging parameter.

The key insight is that without the perturbation term ($\beta = 0$), the system 
would exactly retrace its forward trajectory due to time-reversibility, returning 
to the initial state $\bPhi_0$. However, the nudging perturbation breaks this 
symmetry, and the resulting deviation encodes gradient information.

Contrary to the Lagrangian formulation, where we defined a function 
$t \mapsto \bs_t({\btheta}, \beta)$ through a unified boundary value problem, 
RHEL operates with two distinct trajectories. We refer to this pair as a 
\emph{Hamiltonian Echo System} (HES): $t \mapsto(\bPhi_t({\btheta}, (\balpha_0, \bmu_0)^\top), \bPhi^e_t({\btheta}, \bSigma_z \bPhi_T({\btheta})))$. We also note that RHEL  is also valid in the more general case where the cost function also depends on the momentum of the system (see Equation~\eqref{def:hes_echo}). 

\subsection{Gradient computation}

The fundamental result of RHEL shows that gradients can be computed through 
finite differences between the perturbed and unperturbed Hamiltonian evaluations:

\begin{theorem}[Gradient estimator from RHEL with parametrized initial state~\cite{rhel}] 
\label{thm:rhel}
The gradient of the objective functional is given by:\footnote{We present the unidirectional formulation; the bidirectional version (centered differences) provides $O(\beta^2)$ accuracy. See Appendix~\ref{app:bidirectional-nudging}.}
\begin{align*}
\dtheta C[\bPhi({\btheta}, (\balpha_0(\btheta), \bmu_0(\btheta))^\top)] = 
\lim_{\beta \to 0} \Delta^{\text{RHEL}}(\beta, \balpha_0(\btheta), \bmu_0(\btheta))\,,
\end{align*}
where the finite difference gradient estimator is:
\begin{align}
\label{eq:Delta^RHEL}
\Delta^{\text{RHEL}}(\beta, \balpha_0(\btheta), \bmu_0(\btheta)) := \frac{1}{\beta} \Bigg[-\int_0^T 
&\left[ \partial_{\btheta} H(\bPhi^e_t, {\btheta}, \bx_{T-t}) 
- \partial_{\btheta} H(\bPhi_t, {\btheta}, \bx_t) \right] \dt \notag \\
&+ \left(\partial_{\btheta} \begin{pmatrix} \balpha_0 \\ \bmu_0 \end{pmatrix} \right)^\top \bSigma_x \left(\begin{pmatrix} \bs^e_T \\ \bp^e_T \end{pmatrix} - \begin{pmatrix} \balpha_0 \\ -\bmu_0 \end{pmatrix}\right)\Bigg]\,,
\end{align}
where $\bPhi^e_t$ is the echo trajectory at time $t$, and $\bPhi_t$ represents the forward trajectory 
evaluated at time $t$. We also used the helper matrix $\bSigma_x$ defined as:
\begin{align*}
    \bSigma_x = \begin{pmatrix} \mathbf{0} & \mathbf{I} \\ \mathbf{I} & \mathbf{0} \end{pmatrix}\,.
\end{align*}

When the initial conditions $\begin{pmatrix} \balpha_0 \\ \bmu_0 \end{pmatrix}$ are 
independent of the parameters ${\btheta}$ (i.e., $\partial_{\btheta} \begin{pmatrix} \balpha_0 \\ \bmu_0 \end{pmatrix} = 0$), 
the boundary term vanishes and the estimator reduces to the integral term only.
\end{theorem}

\begin{proof}[Proof sketch]
This result follows from Theorem 3.1 in \citep{rhel}. The detailed derivation, showing how to recover this result from \citep{rhel}, is provided in Appendix~\ref{app:proof-rhel-parametrized}.
\end{proof}

\subsection{Contrast with Variational Approaches}

RHEL was originally derived without requiring a variational principle. Instead, it relies on establishing a direct mapping between the system dynamics and adjoint methods~\citep{rhel}. The central requirement in this approach is finding the correct mapping, which requires insight or good intuition about the structure of the problem. Attempts to generalize RHEL to the broader class of port-Hamiltonian systems~\citep{schaftPortHamiltonianSystemsTheory2014} using this mapping strategy have shown that the original mapping does not straightforwardly extend to such systems~\citep[Appendix A.3.1]{rhel}.

The key insight, already exploited by RHEL, is that time-reversibility of Hamiltonian dynamics combined with a specific choice of boundary conditions can resolve the boundary residual problem identified in Section~\ref{subsec:civp}. Specifically, the initial condition of the echo phase is defined as the momentum-flipped final state of the forward phase, allowing the system to approximately retrace its trajectory in reverse. We call this construction the \emph{bouncing-backward kick} (formalized in Proposition~\ref{prop:solution-pfvp-reversibility}). Since Lagrangian systems also exhibit time-reversibility, the same construction carries over naturally to the LEP framework, where the kick acts on velocity rather than momentum. In the following section, we demonstrate that RHEL emerges as a special case of LEP.

Interestingly, LEP offers a more systematic derivation. Rather than relying on guesses about the correct mapping to adjoint methods, LEP starts from variational principles and lets the mathematical structure dictate the learning algorithm. This generality enables extensions that would be difficult to derive from the RHEL perspective alone. In particular, while the direct mapping approach struggled to handle dissipative systems such as port-Hamiltonians, the variational perspective naturally accommodates dissipation, as we demonstrate in Section~\ref{sec:dissipative}.



\section{RHEL is a particular case of the Lagrangian EP} 
\label{sec:rhel_is_ep}

In this section, we demonstrate that RHEL can be recast as a particular 
instance of LEP when the system exhibits time-reversibility and the nudged 
trajectories are defined through a \emph{Parametric Final Value Problem} (PFVP). 
This connection reveals the fundamental relationship between these seemingly 
different approaches to temporal credit assignment.

\subsection{Instantiation of the Lagrangian EP as a PFVP}

\subsubsection{Definition of the Parametric Final Value Problem (PFVP)}
\label{subsubsec:pfvp-definition}

We now introduce a novel boundary condition formulation that enables tractable trajectory generation while eliminating problematic boundary residuals. The key idea is to define \emph{parametric} final boundary conditions $\balpha_T(\btheta)$ and $\bgamma_T(\btheta)$ that depend on the parameters $\btheta$. This defines the \emph{Parametric Final Value Problem} (PFVP):
\begin{align}
    \forall t \in [0,T] \quad t \mapsto \bs_{\bwda,t}^\beta(\btheta, (\balpha_T(\btheta), \bgamma_T(\btheta))) \text{ satisfies:}\quad
    \begin{cases}
        \EL_{r}(t, \btheta, \beta) = 0 \\
        \bs_{\bwda,T}^\beta(\btheta) = \balpha_T(\btheta) \\
        \dot{\bs}_{\bwda,T}^\beta(\btheta) = \bgamma_T(\btheta)
    \end{cases}\,,
    \label{def:PFVP-general}
\end{align}
where $\text{EL}_r(t, \btheta, \beta)$ denotes the time-indexed Euler-Lagrange equation with reversible Lagrangian $L_r$:
\begin{align*}
    \EL_{r}(t, \btheta, \beta) &:= \partial_{\bs} L_{\beta}(\bs_t, \dot{\bs}_t,\btheta, \bx_t, \by_t) 
    - d_t\partial_{\dot{\bs}} L_{\beta}(\bs_t, \dot{\bs}_t, \btheta, \bx_t, \by_t)\,.
\end{align*}
A reversible Lagrangian satisfies the time-symmetry condition:
\begin{align*}
    L_r(\bs_t, \dot{\bs}_t, \btheta, \bx_t, \by_t) = L_r(\bs_t, -\dot{\bs}_t, \btheta, \bx_t, \by_t)\,.
\end{align*}
This ensures that solutions of the associated Euler-Lagrange equations are time-reversible: forward evolution followed by momentum reversal exactly retraces the original trajectory.

In our instantiation, the parametric boundary conditions $\balpha_T(\btheta)$ and $\bgamma_T(\btheta)$ are defined with a Constant Initial Value Problem (CIVP) with $\beta=0$ (that will then be used for practically running the free phase, see Section~\ref{subsubsec:pfvp-computation}). Specifically, they correspond to the final position and velocity of this CIVP:
\begin{equation}
\left\{
\begin{aligned}
    \balpha_T(\btheta) &:= \bs_{\fwda,T}^0(\btheta, (\balpha_0, \bgamma_0)) \\
    \bgamma_T(\btheta) &:= \dot{\bs}_{\fwda,T}^0(\btheta, (\balpha_0, \bgamma_0))
\end{aligned}
\right.\,,
\label{def:PFVP-boundary-from-CIVP}
\end{equation}
where $\bs_{\fwda,T}^0(\btheta, (\balpha_0, \bgamma_0))$ and $\dot{\bs}_{\fwda,T}^0(\btheta, (\balpha_0, \bgamma_0))$ are the final position and velocity from the CIVP solution without nudging (see Section~\ref{subsec:civp}). This choice ensures that the free trajectory ($\beta=0$) satisfies both the CIVP initial conditions and the PFVP final conditions simultaneously (see Figure~\ref{fig:rhel}A).

\subsubsection{Practical Computation of the PFVP}
\label{subsubsec:pfvp-computation}

Final value problems are generally difficult to solve, as one must find initial conditions that produce prescribed final states—typically requiring iterative root-finding or constrained optimization (see Section~\ref{subsec:cfvp}). However, the PFVP formulation admits efficient computation by converting both phases into Initial Value Problems (IVPs).

\paragraph{Free phase.} By construction, the free trajectory is obtained directly from the CIVP. The FVP $\bs_{\bwda,t}^0(\btheta, (\balpha_T(\btheta), \bgamma_T(\btheta)))$ is equivalent to the CIVP $\bs_{\fwda,t}^0(\btheta, (\balpha_0, \bgamma_0))$ with constant initial conditions $\balpha_0$ and $\bgamma_0$ (See Proposition~\ref{prop:equivalence_IVP_EVP} for details). This trajectory can be computed via standard forward integration from $t=0$ to $t=T$.

\paragraph{Nudged phase: the bouncing-backward kick.} For the nudged trajectory ($\beta \neq 0$), we exploit the time-reversibility of the system to convert the PFVP into an Initial Value Problem (IVP). The key insight is that applying a \emph{velocity kick}---reversing the velocity at the final boundary---allows us to integrate the \emph{same} \footnote{
    Both phases integrate the \emph{same} Euler\textendash Lagrange equations, unlike the adjoint state method~\citep{chenNeuralOrdinaryDifferential2018a}, which often uses time-reversibility but integrates a different ODE to recompute activations during the backward pass, on top of integrating the adjoint equations themselves.
    } dynamical system forward in time rather than solving a final value problem. We call this the \emph{bouncing-backward kick}: the system ``bounces'' off the final state of the free phase and retraces its path backward in physical time, using only forward integration. In the Lagrangian formulation, the kick acts on the \emph{velocity} ($\bgamma_T \to -\bgamma_T$); in the equivalent Hamiltonian formulation (RHEL), it acts on the \emph{momentum} ($\bp \to -\bp$, the $\Sigma_z$ flip).

\begin{proposition}[Bouncing-backward kick: PFVP-to-IVP reduction] 
\label{prop:solution-pfvp-reversibility}
The solution of the time-reversible PFVP~\eqref{def:PFVP-general} with boundary conditions $\balpha_T(\btheta)$ and $\bgamma_T(\btheta)$ satisfies:
\begin{align*}
    \forall t \in [0,T] \qquad\bs_{\bwda,t}^\beta(\btheta, \left(\balpha_T(\btheta), \bgamma_T(\btheta))\right) &= \bs_{\fwda,t'}^\beta\left(\btheta, \left(\balpha_T(\btheta), -\bgamma_T(\btheta)\right)\right) \quad \text{with } t' = T-t\,,
\end{align*}
where $t' \mapsto \bs_{\fwda,t'}^\beta\left(\btheta, \left(\balpha_T(\btheta), -\bgamma_T(\btheta)\right)\right)$ is the solution of the IVP with velocity-reversed initial conditions, integrated forward in time $t'$ from $0$ to $T$ (where $t' = T-t$ relates the integration time $t'$ to the time $t$):
\begin{align*}
    \forall t' \in [0,T] \quad t' \mapsto \bs_{\fwda,t'}^\beta(\btheta, \left(\balpha_T(\btheta), -\bgamma_T(\btheta)\right)) \text{ satisfies:}\quad
    \begin{cases}
        \EL_{r}(t', \btheta, \beta) = 0 \\
        \bs_{\fwda,0}^\beta(\btheta) = \balpha_T(\btheta) \\
        \dot{\bs}_{\fwda,0}^\beta(\btheta) = -\bgamma_T(\btheta)
    \end{cases}
\end{align*}
\end{proposition}

The proposition states that the PFVP solution $\bs_{\bwda,t}^\beta$ at physical time $t$ equals the IVP solution $\bs_{T-t}^\beta$ at integration time $T-t$. Crucially, the Euler-Lagrange equation $\EL_r(T-t, \btheta, \beta)$ is evaluated with the input $\bx_{T-t}$ and target $\by_{T-t}$ corresponding to physical time $T-t$, meaning the input and target sequences are played backward during integration.

In practice, this gives a simple algorithm for the nudged phase: (1) start from the final state of the free phase with reversed velocity $(\balpha_T(\btheta), -\bgamma_T(\btheta))$, and (2) introduce a new integration time variable $t' = T-t$ and integrate the IVP \emph{forward in time} $t'$ from $t'=0$ to $t'=T$ (corresponding to physical time $t$ going backward from $T$ to $0$) while feeding the inputs and targets in reverse temporal order. The resulting IVP trajectory $\bs_{t'}^\beta$, yields the desired PFVP solution $\bs_{\bwda,t}^\beta$ (see Figure~\ref{fig:rhel}B).

\begin{figure}[t!]
    \centering
    \includegraphics[width=1.0\textwidth]{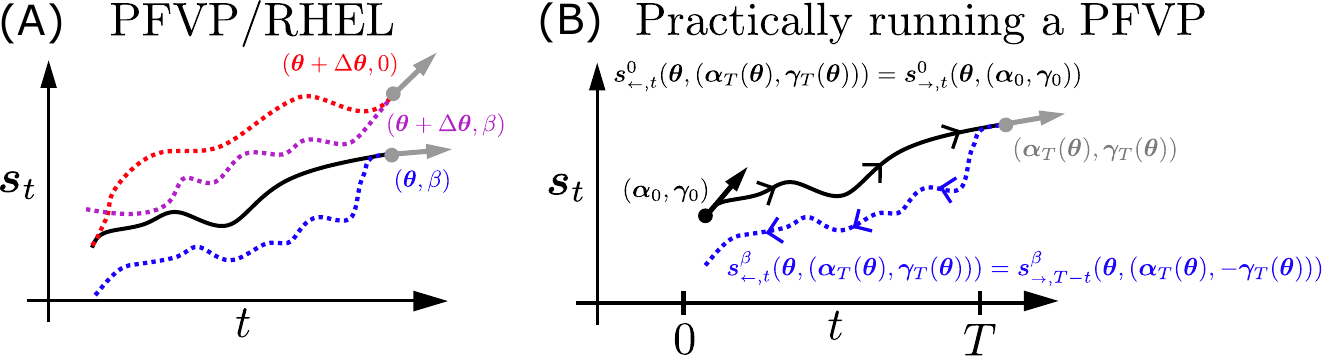}
    \caption{\textbf{Parametric Final Value Problem (PFVP) for LEP.} 
    The two panels use a consistent color scheme: \textcolor{black}{black curves} represent the free trajectory $\bs^0(\btheta)$ used for inference, and \textcolor{blue}{blue dotted curves} show the $\beta$-nudged trajectories $\bs^\beta(\btheta)$ used for learning. The boundary conditions (parametric final value problem) are depicted in \textcolor{gray}{grey}, with a dot for the position $\balpha_T(\btheta)=\bs_{\fwda,T}^0(\btheta)$ and arrow $\bgamma_T(\btheta)=\dot{\bs}_{\fwda,T}^0(\btheta)$. To illustrate how boundary conditions constrain the entire family of trajectories, we also display $\btheta$-perturbed trajectories $\bs^0(\btheta + \Delta\btheta)$ (\textcolor{red}{red dotted curve}) and combined perturbations $\bs^\beta(\btheta + \Delta\btheta)$ (\textcolor{violet}{purple dotted curve}).
    \textbf{(A)} We observe the effect of the \emph{parametric} final value condition: only trajectories that share the same $\btheta$ (\textcolor{blue}{blue} and \textcolor{black}{black} vs \textcolor{red}{red} and \textcolor{violet}{purple}, respectively) satisfy the same position (\textcolor{gray}{grey dot}) and velocity final value conditions (\textcolor{gray}{grey arrow}). 
    \textbf{(B)} The arrows on the curves (\textcolor{blue}{blue} and \textcolor{black}{black}) indicate the direction of integration of their respective IVPs. Although Final Value Problems (FVPs) are generally difficult to solve, both the free phase (\textcolor{black}{black curve}) and the nudged phase (\textcolor{blue}{blue curve}) can be efficiently computed by reformulating them as Initial Value Problems (IVPs). For the free phase, the FVP $\bs_{\bwda,t}^0(\btheta, (\balpha_T, \bgamma_T))$ is equivalent to the Constant Initial Value Problem (CIVP) $\bs_{\fwda,t}^0(\btheta, (\balpha_0, \bgamma_0))$ with initial conditions $\balpha_0$ and $\bgamma_0$ shown in \textcolor{black}{black (dot and arrow, respectively)}. For the nudged phase, we exploit time-reversibility (Proposition~\ref{prop:solution-pfvp-reversibility}): the PFVP $\bs_{\bwda,t}^\beta(\btheta, (\balpha_T, \bgamma_T))$ becomes the Parametric Initial Value Problem (PIVP) ${\bs}^\beta_{t'}(\btheta, (\balpha_T, -\bgamma_T))$ starting from the momentum-reversed final conditions $(\balpha_T, -\bgamma_T)$. This PIVP is then integrated \emph{forward} in integration time $t'=T-t$ from $0$ to $T$ (corresponding to $t$ going backward from $T$ to $0$), as illustrated by the \textcolor{blue}{blue arrows}. This PFVP formulation, expressed through Lagrangian mechanics, corresponds exactly to the Hamiltonian formulation of RHEL after applying the forward Legendre transform (see Theorem~\ref{thm:lep-rhel-equivalence}).}
    \label{fig:rhel}
\end{figure}

\subsection{Boundary Residual Cancellation in PFVP}
Applying Theorem~\ref{thm:general-ep} to this parametric boundary condition 
choice yields a remarkable instantiation of the general gradient formula 
where both the boundary conditions and the time-reversibility cause the boundary residuals to partially cancel.

\begin{theorem}[PFVP Boundary Residual Cancellation]
\label{thm:pfvp-cancellation}
Recall that the parametric boundary conditions $\balpha_T(\btheta) := \bs_{\fwda,T}^0(\btheta, (\balpha_0, \bgamma_0))$ and $\bgamma_T(\btheta) := \dot{\bs}_{\fwda,T}^0(\btheta, (\balpha_0, \bgamma_0))$ are defined as the final position and velocity of the free-phase CIVP (Equation~\eqref{def:PFVP-boundary-from-CIVP}). The boundary residuals in 
Theorem~\ref{thm:general-ep} vanish at $t=T$ and reduce to easy-to-compute terms at $t=0$ for the PFVP 
formulation $\bs_{\bwda}^\beta(\btheta, (\balpha_T, \bgamma_T))$. 
The gradient of the objective functional is given by:
\begin{align*}
    d_{\btheta} C[\bs_{\bwda}^0(\btheta, (\balpha_T, \bgamma_T))]
    = \lim_{\beta \to 0} \Delta^{\text{PFVP}}(\beta, \balpha_0(\btheta),\bgamma_0(\btheta))\,,
\end{align*}
where the PFVP gradient estimator simplifies to:
\begin{align*}
    \Delta^{\text{PFVP}}(\beta, \balpha_0(\btheta),\bgamma_0(\btheta)) := \frac{1}{\beta} \Bigg[
    \int_0^T &\left(\partial_{\btheta} L_{\beta}
    \left(\bs_{\bwda,t}^{\beta}, \dot{\bs}_{\bwda,t}^{\beta}, \btheta\right) 
    - \partial_{\btheta} L_{0}
    \left(\bs_{\bwda,t}^{0}, \dot{\bs}_{\bwda,t}^{0}, \btheta\right)\right) \dt \\
    &+ \left(d_{\btheta\dot{\bs}}L_{0}
    (\balpha_0(\btheta), \bgamma_0(\btheta), \btheta) \right)^\top 
    ~ \left(\bs_{\bwda,0}^\beta - \balpha_0(\btheta)\right) \\ 
    &- \left(\partial_{\btheta} \balpha_0\right)^\top \left( \partial_{\dot{\bs}} L_{0}\left(\bs_{\bwda,0}^\beta , \dot{\bs}_{\bwda,0}^\beta,\btheta\right)- \partial_{\dot{\bs}} L_{0}\left(\balpha_0(\btheta) , \bgamma_0(\btheta),\btheta\right)\right) \Bigg]\,.
\end{align*}
\textit{Note:} When the initial conditions $\balpha_0$ and $\bgamma_0$ are 
independent of $\btheta$ (i.e., $\partial_{\btheta} \balpha_0 = 0$), 
the boundary residual simplifies to a single term:
$\left(\partial_{\btheta\dot{\bs}}L_{0}(\balpha_0, \bgamma_0, \btheta) \right)^\top 
\left(\bs_{\bwda,0}^\beta - \balpha_0\right)$.
\end{theorem}

\textbf{Computational advantages.} The PFVP formulation resolves both computational challenges identified earlier. Unlike CIVP, it avoids intractable boundary residuals that would require backpropagation-like computations. Unlike CBPVP, it uses causal boundary conditions—trajectories are computed via simple forward integration rather than iterative solvers, enabling efficient streaming computation.

Table~\ref{tab:complexity} confirms these advantages quantitatively. PFVP achieves efficient trajectory generation at $\mathcal{O}(N d_s^2)$ time with only $\mathcal{O}(d_s)$ memory, matching CIVP's forward integration cost. Simultaneously, its gradient computation scales as $\mathcal{O}(N d_\theta)$ with $\mathcal{O}(d_\theta)$ memory, matching CBPVP's efficient gradient estimation.

\textbf{Comparison with previous work.} Recently, \citep{massarEquilibriumPropagationLearning2025} proposed a Lagrangian EP formulation; however, their work considers only \emph{fixed} boundary conditions (such as our CBPVP). The central novelty of our PFVP is making the final boundary \emph{parametric}: the terminal constraints $\bs_{\bwda,T}^\beta(\btheta)=\balpha_T$ and $\dot{\bs}_{\bwda,T}^\beta(\btheta)=\bgamma_T$ depend on $\btheta$ through the free-phase CIVP (Equation~\eqref{def:PFVP-boundary-from-CIVP}). 

Fixing $\balpha_T$ and $\bgamma_T$ independently of $\btheta$ would make the system less expressive and make the initial state depend on $\btheta$, forcing the initial conditions to change at every training step. To run an inference with the input in the forward direction, one would need to recompute it after each training step. 


\subsection{Hamiltonian-Lagrangian Equivalence via Legendre Transform}

We now establish the precise mathematical relationship between the PFVP 
formulation of LEP and RHEL. We first introduce the Legendre transform and its condition of well-definiteness to define a bijection between two pairs of variables. We use this to show the equivalence between LEP and RHEL.

This transform is important for our work because it allows us to map solutions of the Euler-Lagrange equations bijectively to solutions of the Hamiltonian equations.

\begin{theorem}[LEP-RHEL Equivalence via Legendre Transform]
\label{thm:lep-rhel-equivalence}
The time-local Legendre transform (Proposition~\ref{prop:Legendre_transform}), applied pointwise along trajectories, creates an equivalence between LEP and RHEL at the level of trajectories (1) and gradient estimators (2).

\textbf{(1) Trajectory Equivalence.} The PFVP formulation of LEP and the HES formulation of RHEL establish a bijection between solutions of Euler-Lagrange and Hamiltonian equations:
\begin{equation*}
t \mapsto \bs_{\bwda,t}^\beta(\btheta, (\balpha_T, \bgamma_T)) \quad \longleftrightarrow \quad t \mapsto \left(\bPhi_t(\btheta, (\balpha_0, \bmu_0)^\top), \bPhi_t^e(\btheta, \Sigma_z \bPhi_T(\btheta))\right)\,,
\end{equation*}
where the Legendre transformation induces the invertible relation between $(\balpha_0(\btheta), \bgamma_0(\btheta))$ and $\begin{pmatrix} \balpha_0(\btheta) \\ \bmu_0(\btheta) \end{pmatrix}$:
\begin{align}
    \begin{pmatrix} \balpha_0(\btheta) \\ \bmu_0(\btheta) \end{pmatrix} = \begin{pmatrix} \balpha_0(\btheta) \\ \partial_{\dot{\bs}} L_0(\balpha_0(\btheta), \bgamma_0(\btheta), \btheta) \end{pmatrix}
    \quad \text{and} \quad
    \begin{pmatrix} \balpha_0(\btheta) \\ \bgamma_0(\btheta) \end{pmatrix} = \begin{pmatrix} \balpha_0(\btheta) \\ \partial_{\bp} H_0(\balpha_0(\btheta), \bmu_0(\btheta), \btheta) \end{pmatrix}\,,
    \label{eq:mapping-ini}
\end{align}
where $\balpha_0, \bgamma_0$ are the Lagrangian initial conditions (position and velocity at $t=0$), and $\begin{pmatrix} \balpha_0 \\ \bmu_0 \end{pmatrix}$ are the Hamiltonian initial conditions (position and momentum at $t=0$), related via the bijective mapping of Equation~\eqref{eq:mapping-ini}.

\textbf{(2) Gradient Equivalence.} Under the respective Legendre transforms, the gradient estimators are identical:
\begin{equation*}
\Delta^{\text{PFVP}}(\beta, \balpha_0(\btheta), \bgamma_0(\btheta)) = \Delta^{\text{RHEL}}(\beta, \balpha_0(\btheta), \bmu_0(\btheta))\,.
\end{equation*}


\makeatletter
\newcommand{\biggg}{\bBigg@{3.5}}
\newcommand{\Biggg}{\bBigg@{4.5}}
\makeatother

{\small
\begin{center}
\textbf{LEP (Lagrangian)}
\[
\begin{aligned}
\Delta^{\text{PFVP}}(\beta, \balpha_0, \bgamma_0)
&= \frac{1}{\beta}\Bigg[\textcolor{blue}{\int_0^T \big(\partial_{\btheta} L_\beta(\bs_{\bwda,t}^\beta, \dot{\bs}_{\bwda,t}^\beta, \btheta) - \partial_{\btheta} L_0(\bs_{\bwda,t}^0, \dot{\bs}_{\bwda,t}^0, \btheta)\big) \, dt}
\\[6pt]
&\qquad + \textcolor{red}{\Biggg(d_{\btheta}\biggg(\begin{matrix} \balpha_0 \\ \partial_{\dot{\bs}} L_0(\balpha_0, \bgamma_0, \btheta)
\end{matrix}\biggg)\Biggg)^\top} \bSigma_x
\textcolor[RGB]{0,140,0}{\Biggg(\biggg(\begin{matrix}
\bs_{\bwda,0}^\beta \\
-\partial_{\dot{\bs}} L_\beta(\bs_{\bwda,0}^\beta, \dot{\bs}_{\bwda,0}^\beta, \btheta)
\end{matrix}\biggg) - \biggg(\begin{matrix}
\balpha_0 \\
-\partial_{\dot{\bs}} L_0(\balpha_0, \bgamma_0, \btheta)
\end{matrix}\biggg)\Biggg)}\Bigg].
\end{aligned}
\]

\medskip
\textbf{RHEL (Hamiltonian)}
\[
\begin{aligned}
\Delta^{\text{RHEL}}(\beta, \balpha_0, \bmu_0)
&= \frac{1}{\beta}\Bigg[\textcolor{blue}{-\int_0^T \big(\partial_{\btheta} H_\beta(\bPhi_t^e, \btheta) - \partial_{\btheta} H_0(\bPhi_t, \btheta)\big) \, dt}
\\[6pt]
&\qquad + \textcolor{red}{\Biggg(\partial_{\btheta} \biggg(\begin{matrix} \balpha_0 \\ \bmu_0 \end{matrix}\biggg)\Biggg)^\top} \bSigma_x
\textcolor[RGB]{0,140,0}{\Biggg(\biggg(\begin{matrix} \bs_T^e \\ \bp_T^e \end{matrix}\biggg) - \biggg(\begin{matrix} \balpha_0 \\ -\bmu_0 \end{matrix}\biggg)\Biggg)}\Bigg].
\end{aligned}
\]
\end{center}
}

where the $\btheta$ dependencies on $\balpha_0, \bgamma_0$ and $\balpha_0, \bmu_0$ --- which are constrained by Equation~\eqref{eq:mapping-ini} --- were dropped for readability. The color coding highlights terms that are equal between LEP and RHEL: \textcolor{blue}{blue} for the integral terms, \textcolor{red}{red} for the parameter derivatives before $\bSigma_x$, and \textcolor[RGB]{0,140,0}{green} for the state differences after $\bSigma_x$.

\end{theorem}

\begin{proof}[Sketch of the proof]
The proof proceeds in three steps.

\paragraph{(1) Legendre correspondence.}
We first show that the Legendre transform establishes a bijection between solutions of the Euler--Lagrange and Hamilton equations. Since the transform itself depends on the parameters~$\btheta$, it not only maps entire trajectories between the two formalisms but also reparametrizes their initial conditions in a $\btheta$-dependent manner.

\paragraph{(2) PFVP--HES construction.}
For both $\beta=0$ and $\beta\neq0$, we construct the HES from the PFVP through a sequence of maps (including the Legendre transform), each of which is bijective.

\paragraph{(3) Gradient equivalence.}
Finally, applying the Legendre transform to the PFVP gradient estimator yields the RHEL gradient expression.  
Term by term, the Lagrangian estimator in LEP matches the Hamiltonian estimator in RHEL, establishing full gradient equivalence.

\end{proof}




\textbf{Theoretical significance.} The combination of 
Theorems~\ref{thm:pfvp-cancellation} and~\ref{thm:lep-rhel-equivalence} establishes 
a fundamental result: RHEL can be derived from first principles using 
variational methods of EP.
Theorem~\ref{thm:pfvp-cancellation} demonstrates that 
the PFVP formulation is a solution instance of LEP, the first one we found 
that does not have problematic boundary residuals, and thus can be used 
to train Lagrangian systems. Furthermore, we can also recover the RHEL learning 
rule for Hamiltonian systems: Theorem~\ref{thm:lep-rhel-equivalence} shows that 
this computationally viable LEP formulation is mathematically equivalent 
to RHEL through the Legendre transformation. This equivalence provides a 
new theoretical foundation for RHEL, revealing that its distinctive 
properties---forward-only computation, scalability independent of model 
size, and local learning---emerge naturally from the variational structure 
of physical systems rather than being only the consequence of specific Hamiltonian dynamics.



\subsection{Empirical validation}
\label{subsec:empirical-validation}
We now provide numerical validation of Theorem~\ref{thm:lep-rhel-equivalence} by training a Hopfield-inspired dynamical system using both RHEL (Hamiltonian formulation) and LEP (Lagrangian formulation), demonstrating that the two approaches yield identical gradients.

\subsubsection{Example of equivalence: fixed Hamiltonian initial conditions}

\paragraph{Learning rule analysis.}
Consider the case where the Hamiltonian initial conditions $\balpha_0$ and $\bmu_0$ are fixed independently of $\btheta$, i.e., $\partial_{\btheta} \balpha_0 = 0$ and $\partial_{\btheta} \bmu_0 = 0$. In this setting, the \textcolor{red}{red boundary term} in Theorem~\ref{thm:lep-rhel-equivalence} vanishes, and both gradient estimators reduce to the \textcolor{blue}{blue integral term} only:
\begin{align*}
    \partial_{\btheta} \begin{pmatrix} \balpha_0 \\ \bmu_0 \end{pmatrix} = \mathbf{0} \quad \Rightarrow \quad \Delta^{\text{RHEL}}(\beta, \balpha_0, \bmu_0) &= -\frac{1}{\beta} \int_0^T \left[ \partial_{\btheta} H_\beta(\bPhi^e_t, \btheta) - \partial_{\btheta} H_0(\bPhi_t, \btheta) \right] \dt \\
    &= \Delta^{\text{PFVP}}(\beta, \balpha_0, \bgamma_0(\btheta))\,,
\end{align*}
where $\bgamma_0(\btheta) = \partial_{\bp} H_0(\balpha_0, \bmu_0, \btheta)$ is the corresponding Lagrangian initial velocity. The LEP gradient estimator takes the equivalent form:
\begin{align*}
    \Delta^{\text{PFVP}}(\beta, \balpha_0, \bgamma_0(\btheta)) = \frac{1}{\beta} \int_0^T \left[ \partial_{\btheta} L_\beta(\bs_{\bwda,t}^{\beta}, \dot{\bs}_{\bwda,t}^{\beta}, \btheta) - \partial_{\btheta} L_0(\bs_{\bwda,t}^{0}, \dot{\bs}_{\bwda,t}^{0}, \btheta) \right] \dt\,.
\end{align*}
Both learning rules compare parameter derivatives along the free and nudged trajectories, differing only in whether Hamiltonian or Lagrangian variables are used.

\paragraph{Initial condition analysis.}
Crucially, fixing the \emph{Hamiltonian} initial conditions induces \emph{parametric} Lagrangian initial conditions. Through the Legendre transform (Equation~\eqref{eq:mapping-ini}), the initial velocity in the Lagrangian formulation is:
\begin{align*}
    \bgamma_0(\btheta) = \partial_{\bp} H_0(\balpha_0, \bmu_0, \btheta)\,.
\end{align*}
When $H_0$ depends on $\btheta$ (e.g., through a mass matrix or time constant parameters), the initial velocity $\bgamma_0$ becomes $\btheta$-dependent even though the Hamiltonian initial conditions are fixed. This subtlety is illustrated in Figure~\ref{fig:empirical_validation}B, where the Lagrangian phase portraits show varying initial velocities across training epochs as parameters evolve.

\begin{remark}[Simplification for zero initial momentum]
\label{remark:zero_initial_momentum}
In practice, if one wishes to avoid implementing the boundary term in the learning rule, one can set $\bmu_0 = \mathbf{0}$. This yields $\bgamma_0 = \partial_{\bp} H_0(\balpha_0, \mathbf{0}, \btheta) = \mathbf{0}$ for standard kinetic energies, making both initial conditions non-parametric.
\end{remark}

\subsubsection{Hopfield-inspired system with learnable time constants}

We validate our theoretical results on a Hopfield-inspired dynamical system, based on the Hopfield model in Table~\ref{tab:ml_hamiltonians}. For simplicity, we set $\alpha = 0$ and $b = \mathbf{0}$ (no regularization or bias in the potential). The Lagrangian takes the form (see Table~\ref{tab:ml_hamiltonians}):
\begin{align}
    L_0(\bs, \dot{\bs}, \btheta, \bx) = \frac{1}{2}\dot{\bs}^\top \mathrm{diag}(\bm{\tau}) \dot{\bs} - \frac{1}{2}\rho(\bs)^\top \bm{W} \rho(\bs) - \bm{B}^\top \rho(\bs) - \rho(\bx)^\top \rho(\bs)\,,
    \label{eq:hopfield_lagrangian}
\end{align}
where $\bs \in \mathbb{R}^d$ is the state, $\rho(\cdot)$ is an element-wise activation function (e.g., $\tanh$), $\bm{\tau} \in \mathbb{R}^d_{>0}$ is a vector of learnable time constants, $\bm{W} \in \mathbb{R}^{d \times d}$ is the symmetric recurrent weight matrix, $\bm{B} \in \mathbb{R}^d$ is a bias vector, and $\bx_t \in \mathbb{R}^d$ is the time-varying input. The learnable parameters are $\btheta = (\bm{W}, \bm{B}, \bm{\tau})$.

\paragraph{Parameter gradients.} Table~\ref{tab:learning_rules} summarizes the parameter gradients in both formalisms. Notably, the gradient with respect to $\bm{W}$ takes the form $\rho(\bs)\rho(\bs)^\top$, which corresponds to a \emph{Hebbian learning rule}---one of the most famous and oldest learning rules in neuroscience~\citep{dauphinRecurrentHamiltonianEcho2025}. Additionally, the gradient with respect to $\bm{\tau}$ takes different forms in each formulation: in the Lagrangian it depends on velocities $\dot{\bs}$, while in the Hamiltonian it depends on momenta $\bp$. These are related through the Legendre transform and yield identical learning signals.

\begin{table}[h]
\centering
\renewcommand{\arraystretch}{1.8}
\begin{tabular}{@{}l c c@{}}
\toprule
\textbf{Parameter} & \textbf{LEP:} $\partial_{\btheta} L_0$ & \textbf{RHEL:} $\partial_{\btheta} H_0$ \\
\midrule
$\bm{W}$ & $-\frac{1}{2}\rho(\bs) \rho(\bs)^\top$ & $\frac{1}{2}\rho(\bs) \rho(\bs)^\top$ \\
$\bm{B}$ & $-\rho(\bs)$ & $\rho(\bs)$ \\
$\bm{\tau}$ & $\frac{1}{2}\dot{\bs} \odot \dot{\bs}$ & $-\frac{1}{2}\bp \odot \bp \odot \bm{\tau}^{-2}$ \\
\bottomrule
\end{tabular}
\caption{\textbf{Parameter gradients for the Hopfield-inspired system.} The symbol $\odot$ denotes element-wise multiplication. The relation $\partial_{\btheta} H_0 = -\partial_{\btheta} L_0$ (Lemma~\ref{lemma:param_grad}) is verified for each parameter. For $\bm{W}$, the gradient simplifies due to its symmetry. For the time constant $\bm{\tau}$, using $\dot{\bs} = \mathrm{diag}(\bm{\tau})^{-1}\bp$ confirms that $\frac{1}{2}\dot{\bs} \odot \dot{\bs} = \frac{1}{2}\bp \odot \bp \odot \bm{\tau}^{-2}$.}
\label{tab:learning_rules}
\end{table}

\paragraph{Initial condition mapping.} For fixed Hamiltonian initial conditions $(\balpha_0, \bmu_0)$, the corresponding Lagrangian initial conditions are:
\begin{align*}
    \text{Position:} \quad &\balpha_0 \quad \text{(unchanged)} \\
    \text{Velocity:} \quad &\bgamma_0 = \mathrm{diag}(\bm{\tau})^{-1} \bmu_0 \quad \text{($\btheta$-dependent through $\bm{\tau}$)}\,.
\end{align*}
This $\btheta$-dependence of the Lagrangian initial velocity through the learnable time constants $\bm{\tau}$ is what makes the initial conditions parametric in the LEP formulation, as illustrated in Figure~\ref{fig:empirical_validation}B where the initial velocity changes across training epochs.

\subsubsection{Experimental setup}

\paragraph{Task.} We consider a teacher-student learning setup with a 6-dimensional system ($d=6$). The input signal $\bx_t$ is injected into neuron 0 and consists of a superposition of 10 random sine waves:
\begin{align*}
    x_t = \frac{1}{n_{\text{waves}}} \sum_{k=1}^{n_{\text{waves}}} a_k \sin(2\pi f_k t + \phi_k)\,,
\end{align*}
where frequencies $f_k$ are uniformly sampled from $[10^{-2}, 1]$ Hz, phases $\phi_k$ from $[0, 2\pi]$, and amplitudes $a_k$ from $[0.5, 1.5]$. The target output $\by_t$ is generated by a teacher network with the same architecture but different random initialization. The cost function is the squared error on neuron 5: $c(\bs_t, \by_t) = \frac{1}{2}(s_t^{(5)} - y_t^{(5)})^2$. We use Euler integration with time step $\dt = 0.001$, total duration $T = 10$, and nudging strength $\beta = 0.01$. 

\paragraph{Parameter initialization.} The weight matrix $\bm{W}$ is initialized via QR decomposition: a random orthogonal matrix $\bm{U}$ is obtained from the QR factorization of a Gaussian matrix, and eigenvalues are sampled uniformly from $[0.1, 1.0]$, yielding $\bm{W} = \bm{U} \, \mathrm{diag}(\bm{\lambda}) \, \bm{U}^\top$ for controlled spectral properties. Time constants $\bm{\tau}$ are sampled uniformly from $[0.5, 1.0]$. We use the Adam optimizer with learning rate $0.005$ and random seed $50$. Full hyperparameter details are given in Appendix~\ref{appx:experimental_details}.

We perform two separate training runs of 100 epochs each, both starting from the same initial parameter values $\btheta_0$: (1) a RHEL training run using Hamiltonian parameterization with state variables $(\bs, \bp)$ and learning rules from Table~\ref{tab:learning_rules} (right column), and (2) a LEP training run using Lagrangian parameterization with state variables $(\bs, \dot{\bs})$ and learning rules from Table~\ref{tab:learning_rules} (left column). The Hamiltonian initial conditions $(\balpha_0, \bmu_0)$ are fixed and identical for both runs; in the LEP run, these map to Lagrangian initial conditions $(\balpha_0, \bgamma_0)$ where $\bgamma_0 = \mathrm{diag}(\bm{\tau})^{-1}\bmu_0$ evolves as $\bm{\tau}$ changes during training. During LEP training, at every gradient update we also compute the gradient provided by automatic differentiation (BPTT) for comparison. 

\begin{figure}[t!]
    \centering
    \includegraphics[width=1.0\textwidth]{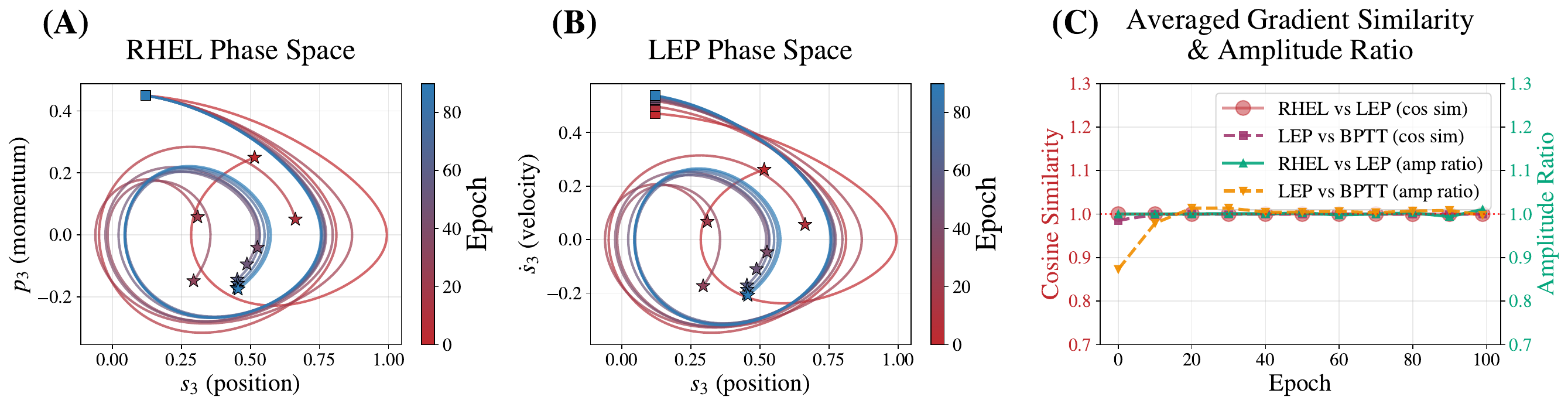}
    \caption{\textbf{Numerical validation of Theorem~\ref{thm:lep-rhel-equivalence} with two separate training runs.} 
    A 6-dimensional Hopfield-inspired system (Equations~\eqref{eq:hopfield_lagrangian}) is trained in two separate runs from the same initial parameters: one with Hamiltonian parameterization + RHEL, one with Lagrangian parameterization + LEP. The phase portrait of ``hidden layer'' neuron 3 (position $s^{(3)}$ vs.\ momentum/velocity) is shown across training epochs (colored trajectories from \textcolor{red}{red} to \textcolor{blue}{blue}); squares mark initial conditions, stars mark final conditions.
    \textbf{(A)}~RHEL training run with Hamiltonian parameterization $(\bs, \bp)$.
    \textbf{(B)}~LEP training run with Lagrangian parameterization $(\bs, \dot{\bs})$.
    \textbf{(C)}~Cosine similarity and amplitude ratio between gradient estimates: LEP vs.\ BPTT (\textcolor[RGB]{162,59,114}{purple}, \textcolor[RGB]{241,143,1}{orange}) and RHEL vs.\ LEP (\textcolor[RGB]{193,41,46}{red}, \textcolor[RGB]{6,167,125}{green}).}
    \label{fig:empirical_validation}
\end{figure}

The experiment confirms the predictions of Theorem~\ref{thm:lep-rhel-equivalence}. The two separate training runs---one with Hamiltonian parameterization and RHEL learning rule, one with Lagrangian parameterization and LEP learning rule---both start from the same initial parameter values $\btheta_0$ and evolve the parameters independently. In the RHEL run (Figure~\ref{fig:empirical_validation}A), the Hamiltonian initial conditions $(\balpha_0, \bmu_0)$ remain fixed across training epochs while the input signal (a superposition of sine waves) drives complex oscillatory dynamics. In the LEP run (Figure~\ref{fig:empirical_validation}B), the same fixed Hamiltonian initial conditions $(\balpha_0, \bmu_0)$ map to Lagrangian initial conditions where the initial velocity $\bgamma_0 = \mathrm{diag}(\bm{\tau})^{-1}\bmu_0$ shifts across epochs as the time constant parameters $\bm{\tau}$ evolve during training, illustrating the $\btheta$-dependence of boundary conditions under the Legendre transform. Despite these two independent training runs using different parameterizations and learning rules, the LEP and RHEL gradient estimates agree nearly perfectly throughout training (cosine similarity $\approx 1$, amplitude ratio $\approx 1$), and both closely match the ground-truth BPTT gradients obtained via automatic differentiation.

\section{From LEP to Dissipative LEP}
\label{sec:dissipative}

The non-dissipative nature of standard Hamiltonian/Lagrangian systems has been recognized as a limitation in both the LEP and HEL literatures, on two fronts. From a hardware perspective, energy conservation restricts the class of physical systems where LEP can be implemented; to address this, \citet{kendall2021gradient} proposed using fractional calculus to extend Lagrangian mechanics to dissipative dynamics. From a machine learning perspective, the absence of dissipation means that, like Unitary RNNs before them \citep{jingGatedOrthogonalRecurrent2017}, Lagrangian/Hamiltonian systems cannot forget \citep{rhel, lopez-pastorSelfLearningMachinesBased2023, boyerLearningDissipateEnergy2025}.

In this section, we take a first step toward addressing this limitation by extending LEP to dissipative systems. We show that dissipation can be introduced through an exponential integrating factor in the Lagrangian, and made practical via the PFVP formulation: during the free phase, the system genuinely dissipates energy, while during the nudge phase, energy is pumped back in.

\subsection{Energy Conservation in Standard Lagrangian Systems}

To understand the non-dissipative nature of standard Lagrangian systems, we first consider an \emph{isolated system} without external input. Let $L_0^{\mathrm{iso}}(\bs_t, \dot{\bs}_t, \btheta)$ denote the Lagrangian of the isolated system, obtained by setting $\bx_t = 0$ in the full Lagrangian $L_0(\bs_t, \dot{\bs}_t, \btheta, \bx_t)$. For any Lagrangian system, there exists a conserved quantity~\cite{olverCalculusVariations2022}:
\begin{equation}
    E = \dot{\bs}_t^\top \partial_{\dot{\bs}} L_0^{\mathrm{iso}} - L_0^{\mathrm{iso}}\,.
    \label{eq:energy_definition}
\end{equation}
This quantity $E$ is the \emph{physical energy} of the system: kinetic energy plus internal potential energy, corresponding to the standard notion of mechanical energy in classical physics.

For the isolated system satisfying the Euler-Lagrange equations $\partial_{\bs} L_0^{\mathrm{iso}} - d_t\partial_{\dot{\bs}} L_0^{\mathrm{iso}} = 0$, this energy is conserved: $d_t E = 0$. This can be verified by direct computation:
\begin{align*}
    d_t E &= d_t\left(\dot{\bs}_t^\top \partial_{\dot{\bs}} L_0^{\mathrm{iso}}\right) - d_t L_0^{\mathrm{iso}} \\
    &= \dot{\bs}_t^\top \left(d_t\partial_{\dot{\bs}} L_0^{\mathrm{iso}} - \partial_{\bs} L_0^{\mathrm{iso}}\right) = 0\,.
\end{align*}

Note that when the external input $\bx_t$ is applied, it introduces in the Lagrangian $L_0(\bs_t, \dot{\bs}_t, \btheta, \bx_t)$ a time dependence that breaks this energy conservation. The system can exchange energy with its environment through the input, but does not dissipate energy by itself. 


\subsection{Dissipative LEP}
\label{subsec:dissipative_lep}

To address the limitation identified above, we extend LEP to dissipative systems by introducing an explicitly time-dependent Lagrangian through an exponential integrating factor. This approach generalizes a known method for simulating dissipation~\cite{rieweNonconservativeLagrangianHamiltonian1996a} to the multivariate case.

\paragraph{Construction of the dissipative Lagrangian.}
We scale the standard physical Lagrangian $L_0$ by an exponential factor, yielding the \emph{dissipative Lagrangian}:
\begin{equation}
    L^{\mathrm{diss}}_\beta(\bs_t, \dot{\bs}_t, \btheta, \bx_t, \by_t) := \expw{\exp(\bzeta t)} \cdot L_0(\bs_t, \dot{\bs}_t, \btheta, \bx_t) + \beta \, c(\bs_t, \by_t)\,,
    \label{eq:dissipative_lagrangian}
\end{equation}
where $\bzeta > 0$ is a scalar damping coefficient. The exponential factor $\expw{\exp(\bzeta t)}$ acts as an integrating factor that introduces dissipation into the dynamics while maintaining the variational structure needed for gradient estimation.

\paragraph{Dissipative gradient estimator.}
We now present the dissipative counterpart of Theorem~\ref{thm:pfvp-cancellation}. The structure remains similar, but with additional terms arising from the exponential time-weighting.

\begin{theorem}[Dissipative LEP with PFVP]
\label{thm:dissipative_LEP}
Let $t \mapsto \bs_{\bwda,t}^\beta(\btheta)$ denote the solution to the dissipative Euler-Lagrange equation:
\begin{equation}
    \mathrm{EL}_{\mathrm{diss}}(t, \btheta, \beta) := \partial_{\bs} L_0 - d_t\partial_{\dot{\bs}} L_0 - \bzeta \, \partial_{\dot{\bs}} L_0 + \beta \, \expw{\exp(-\bzeta t)} \, \partial_{\bs} c = 0\,,
    \label{eq:dissipative_EL}
\end{equation}
with PFVP boundary conditions. Then the gradient of the objective functional is given by:
\begin{align*}
    \dtheta \mathcal{C}[\bs_{\bwda}^0(\btheta)] &= \lim_{\beta \to 0} {\Delta}_{\mathrm{PFVP}}(\beta, \balpha_0(\btheta), \bgamma_0(\btheta))\,,
\end{align*}
where the dissipative PFVP gradient estimator is:
\begin{align}
    {\Delta}_{\mathrm{PFVP}}(\beta, \balpha_0, \bgamma_0) &:= \frac{1}{\beta} \Bigg[ \textcolor{blue}{\int_0^T \expw{\exp(\bzeta t)} \cdot \left( \partial_{\btheta} L_\beta(\bs_{\bwda,t}^\beta, \dot{\bs}_{\bwda,t}^\beta, \btheta) - \partial_{\btheta} L_0(\bs_{\bwda,t}^0, \dot{\bs}_{\bwda,t}^0, \btheta) \right) \, \mathrm{d}t} \nonumber \\
    &\quad + \textcolor{red}{\left( \dtheta \partial_{\dot{\bs}} L_0(\balpha_0, \bgamma_0, \btheta) \right)^\top} \textcolor[RGB]{0,140,0}{\left( \bs_{\bwda,0}^\beta - \balpha_0 \right)} \nonumber \\
    &\quad - \textcolor{red}{(\partial_{\btheta} \balpha_0)^\top} \textcolor[RGB]{0,140,0}{\left( \partial_{\dot{\bs}} L_\beta(\bs_{\bwda,0}^\beta, \dot{\bs}_{\bwda,0}^\beta, \btheta) - \partial_{\dot{\bs}} L_0(\balpha_0, \bgamma_0, \btheta) \right)} \Bigg]\,,
    \label{eq:dissipative_PFVP_gradient}
\end{align}
with $\balpha_0 = \balpha_0(\btheta)$ and $\bgamma_0 = \bgamma_0(\btheta)$ the initial conditions. The \textcolor{blue}{blue integral term} weights the Lagrangian difference by the exponential factor, the \textcolor{red}{red terms} involve parameter derivatives of initial conditions, and the \textcolor[RGB]{0,140,0}{green terms} measure state and momentum differences at the initial time.
\end{theorem}

\begin{proof}
See Appendix~\ref{appx:dissipative_LEP}.
\end{proof}

\paragraph{Interpretation: dissipative terms.}
Compared to the conservative case, the dissipative formulation introduces a new \expw{exponentially-weighted term} (shown in \expw{orange}) in both the Euler-Lagrange equation and the gradient estimator:
\begin{itemize}[leftmargin=*]
    \item In the \textbf{Euler-Lagrange equation}~\eqref{eq:dissipative_EL}: The term $-\bzeta \, \partial_{\dot{\bs}} L_0$ introduces friction-like damping, while the cost term acquires a down-weighting factor $\expw{\exp(-\bzeta t)}$ that reduces nudging strength at later times.
    \item In the \textbf{gradient estimator}~\eqref{eq:dissipative_PFVP_gradient}: The \textcolor{blue}{integral term} is weighted by $\expw{\exp(\bzeta t)}$, emphasizing later time steps. This reflects that dissipative dynamics progressively "forget" early information, so gradients appropriately emphasize recent observations.
    \item For the \textbf{free phase} ($\beta = 0$): The Euler-Lagrange equation reduces to $\partial_{\bs} L_0 - d_t\partial_{\dot{\bs}} L_0 - \bzeta \, \partial_{\dot{\bs}} L_0 = 0$, which is identical to applying the standard Euler-Lagrange equation to the exponentially-weighted Lagrangian $\exp(\bzeta t) L_0$.
\end{itemize}
The boundary terms (\textcolor{red}{red} and \textcolor[RGB]{0,140,0}{green}) remain identical to the conservative PFVP case.

\paragraph{Verification: energy dissipation.}
To confirm that the exponential integrating factor introduces a dissipative system with energy decay (rather than merely rescaling time), we analyze how energy evolves under the dissipative dynamics. We again consider the isolated system ($\bx_t = 0$) to cleanly isolate the effect of dissipation. We find that for a trajectory $t \mapsto \bs_t$ satisfying the dissipative Euler-Lagrange equation~\eqref{eq:dissipative_EL} with $\beta = 0$ and $\bx_t = 0$, the physical energy $E$ (defined as in~\eqref{eq:energy_definition}) evolves as (see Proposition~\ref{prop:energy_dissipation} in Appendix):
$$d_t E = - \bzeta \, \dot{\bs}_t^\top \partial_{\dot{\bs}} L_0^{\mathrm{iso}}$$
In the special case with quadratic kinetic energy $E_{\mathrm{kin}}(\dot{\bs}_t) = \frac{1}{2} \|\dot{\bs}_t\|^2$, this reduces to:
$$d_t E = - \bzeta \|\dot{\bs}_t\|^2 \leq 0$$
Since $\bzeta > 0$, energy is strictly dissipated whenever $\dot{\bs}_t \neq 0$, confirming the physically expected behavior of a dissipative system.

\subsection{Empirical Validation}
\label{subsec:empirical}

We now validate the dissipative LEP framework empirically. Our goals are twofold: first, to confirm that the exponential integrating-factor mechanism genuinely introduces dissipation, with energy transfers consistent with Proposition~\ref{prop:harmonic_energy_evolution}; second, to verify that the dissipative LEP gradient estimator accurately recovers parameter gradients, using autodiff/BPTT as a ground-truth baseline. We conduct these experiments on a system of $d=6$ coupled damped harmonic oscillators (Figure~\ref{fig:dissipative_oscillators}), extending the undamped system of Section~\ref{subsec:lagrangian_hamiltonian} with damping forces via the exponential integrating-factor introduced above.

\paragraph{System description.}
Consider a $d$-dimensional system of coupled harmonic oscillators with mass vector $\bm{m} \in \mathbb{R}^{d}_{>0}$, symmetric stiffness matrix $\bm{K} \in \mathbb{R}^{d \times d}$, and scalar damping coefficient $\bzeta > 0$. The damping vector is $\bgamma := \zeta \bm{m}$, making the damping force proportional to mass (for independent per-dimension damping coefficients $\gamma_i$ decoupled from mass, see Appendix~\ref{app:anisotropic}). An external input $x_t$ drives the first oscillator, and the output is measured from the last oscillator $y_t = s_{d,t}$. The learnable parameters are $\btheta = \{\bm{m}, \bm{K}, \zeta\}$. We use fixed initial conditions $(\bs_0, \dot{\bs}_0) = (\balpha_0, \mathbf{0})$ (zero initial velocity), ensuring boundary terms vanish as explained in Remark~\ref{remark:zero_initial_momentum}. The Lagrangian of the undamped, input-driven system is:
\begin{equation}
    L_0(\bs_t, \dot{\bs}_t, \btheta, x_t) = \frac{1}{2} (\bm{m} \odot \dot{\bs}_t) \cdot \dot{\bs}_t - \frac{1}{2} \bs_t^\top \bm{K} \bs_t - \bm{e}_1^\top \bs_t \, x_t\,,
    \label{eq:harmonic_lagrangian}
\end{equation}
where $\bm{e}_1 = (1, 0, \ldots, 0)^\top$ and $\odot$ denotes element-wise multiplication. The dissipative Lagrangian is $L^{\mathrm{diss}}_\beta = \exp(\bzeta t) \cdot L_0 + \beta \, c(\bs_t, y_t)$ with cost $c(\bs_t, y_t) = \frac{1}{2}(s_{d,t} - y_t)^2$.

\paragraph{Dynamics and gradient estimator: contrast with classical LEP.}
Table~\ref{tab:dissipative_summary} summarizes the dissipative LEP equations. Both the free and nudged dynamics are integrated \emph{forward in time} as Initial Value Problems (IVPs). Compared to the classical (non-dissipative) LEP gradient estimator (Theorem~\ref{thm:pfvp-cancellation}), the dissipative formulation introduces two key modifications:
\begin{enumerate}[leftmargin=*]
    \item \textbf{Sign-flipped damping in the nudge phase.} For the nudged phase, we apply the bouncing-backward kick (Proposition~\ref{prop:solution-pfvp-reversibility}): solving the PFVP backward in time from final conditions $(\bs^\beta_T, \dot{\bs}^\beta_T) = (\bs^0_T, \dot{\bs}^0_T)$ is equivalent to integrating forward with velocity-reversed initial conditions $(\bs^0_T, -\dot{\bs}^0_T)$ and, crucially, \emph{sign-flipped damping} ($+\bgamma \to -\bgamma$). This sign flip reverses the energy flow: while the free phase dissipates energy, the nudged phase \emph{pumps energy back} (see Appendix~\ref{appx:dissipative_oscillators} and Proposition~\ref{prop:dissipative_time_reversal}).
    \item \textbf{Exponential weighting in nudging and learning rule.} The cost nudging term in the Euler-Lagrange equation~\eqref{eq:dissipative_EL} acquires a down-weighting factor $\exp(-\bzeta t)$, and the gradient estimator~\eqref{eq:dissipative_PFVP_gradient} is weighted by $\exp(\bzeta t)$. These exponential factors arise from the integrating-factor construction and are essential for correct gradient estimation.
\end{enumerate}
With fixed initial conditions and PFVP final-condition matching, all boundary terms in Theorem~\ref{thm:dissipative_LEP} vanish, leaving only the integral term that compares trajectories.

\paragraph{Classical LEP baseline.}
To isolate the importance of these modifications, we also evaluate a classical LEP baseline that correctly performs the sign-flipped damping ($+\bgamma \to -\bgamma$) during the nudge phase, but \emph{omits both exponential factors}: it uses the standard nudging $\beta \, \partial_{\bs} c$ instead of the down-weighted $\beta \, \exp(-\bzeta t) \, \partial_{\bs} c$ in the dynamics~\eqref{eq:dissipative_EL}, and the standard unweighted integral $\int_0^T [\cdots] \, \mathrm{d}t$ instead of the exponentially-weighted $\int_0^T [\cdots] \exp(\bzeta t) \, \mathrm{d}t$ in the gradient estimator~\eqref{eq:dissipative_PFVP_gradient}. In other words, this baseline accounts for the dissipative dynamics (including the sign flip) but not for the effect of dissipation on the variational gradient formula.

\begin{table}[h]
\centering
\renewcommand{\arraystretch}{2.2}
\resizebox{\textwidth}{!}{
\begin{tabular}{@{}l l l l@{}}
\toprule
\textbf{Phase} & \textbf{Dynamics (IVP)} & \textbf{Time} & \textbf{Initial Conditions} \\
\midrule
Free ($\beta=0$) & $\bm{m} \odot \ddot{\bs}^0_t + \bgamma \odot \dot{\bs}^0_t + \bm{K} \bs^0_t = -x_t \bm{e}_1$ & $t \in [0,T]$ & $(\bs^0_0, \dot{\bs}^0_0) = (\balpha_0, \mathbf{0})$ \\[0.5em]
Nudged ($\beta > 0$) & $\bm{m} \odot \ddot{\bs}^\beta_{t'} - \bgamma \odot \dot{\bs}^\beta_{t'} + \bm{K} \bs^\beta_{t'} = -x_{T-t'} \bm{e}_1$ & $t' \in [0,T]$ & $(\bs^\beta_0, \dot{\bs}^\beta_0) = (\bs^0_T, -\dot{\bs}^0_T)$ \\
& \qquad $- \beta e^{-\bzeta(T-t')} \bm{e}_d (s^\beta_{d,t'} - y_{T-t'})$ & & \\[0.5em]
\midrule
\multicolumn{4}{@{}l@{}}{\textbf{Gradient estimator:} \quad $\displaystyle \dtheta \mathcal{C}[\bs^0(\btheta)] = \lim_{\beta \to 0} \frac{1}{\beta} \int_0^T \left[ \partial_{\btheta} L_\beta(\bs^\beta_t, \dot{\bs}^\beta_t, \btheta, x_t) - \partial_{\btheta} L_0(\bs^0_t, \dot{\bs}^0_t, \btheta, x_t) \right] \exp(\bzeta t) \, \mathrm{d}t$} \\[0.8em]
\multicolumn{4}{@{}l@{}}{\quad with $\partial_{m_i} L_0 = \frac{1}{2} \dot{s}_{i,t}^2$, \quad $\partial_{\bm{K}} L_0 = -\frac{1}{2} \bs_t \bs_t^\top$, \quad $e^{-\bzeta t}\partial_{\bzeta} [e^{\bzeta t} L_0] = t \cdot L_0(\bs_t, \dot{\bs}_t, \btheta, x_t)$} \\[0.5em]
\midrule
\multicolumn{4}{@{}l@{}}{\textbf{Energy definition:} \quad $\displaystyle E(t) = \underbrace{\frac{1}{2} (\bm{m} \odot \dot{\bs}_t) \cdot \dot{\bs}_t}_{E_{\mathrm{kin}}(t)} + \underbrace{\frac{1}{2} \bs_t^\top \bm{K} \bs_t}_{U_{\mathrm{int}}(t)}$} \\[0.8em]
\multicolumn{4}{@{}l@{}}{\textbf{Energy transfer during inference (free phase)} (Prop.~\ref{prop:harmonic_energy_evolution}): \quad $E(t) = E(0) + W_{\mathrm{input}}(t) - D_{\mathrm{diss}}(t)$} \\[0.3em]
\multicolumn{4}{@{}l@{}}{\quad $\bullet$ \textbf{Input work:} $W_{\mathrm{input}}(t) = -\int_0^t \dot{s}_{1,\tau} \, x_\tau \, \mathrm{d}\tau$ \quad (power = force $\times$ velocity; can inject or extract energy)} \\[0.3em]
\multicolumn{4}{@{}l@{}}{\quad $\bullet$ \textbf{Dissipation:} $D_{\mathrm{diss}}(t) = \int_0^t \bgamma \cdot \dot{\bs}_\tau^2 \, \mathrm{d}\tau \geq 0$ \quad (always removes energy)} \\
\bottomrule
\end{tabular}}
\caption{\textbf{Summary of dissipative LEP for coupled harmonic oscillators.} Both the free and nudged phases are integrated forward in time as IVPs from their respective initial conditions; for the nudged phase, the backward PFVP is implemented through the bouncing-backward kick with reversed initial velocity and sign-flipped damping. The gradient estimator contains \emph{only integral terms}---all boundary terms cancel due to fixed initial conditions and PFVP matching of final conditions. During inference (free phase), the energy $E(t) = E_{\mathrm{kin}}(t) + U_{\mathrm{int}}(t)$ (kinetic plus internal potential) evolves through two mechanisms: work done by the external input (force $\times$ velocity), and dissipation (proportional to $\bgamma \cdot \dot{\bs}_t^2$, always removes energy).}
\label{tab:dissipative_summary}
\end{table}

Figure~\ref{fig:dissipative_oscillators} reports the outcomes of both validations.
Regarding energy dynamics (Panels~A--B), the free-phase energy decomposition shows kinetic and potential energy oscillating out of phase as energy transfers between modes, while the total energy increases over time due to the external input driving the system---kinetic energy starts at zero since $\dot{\bs}_0 = \mathbf{0}$.
The cumulative dissipation and input work are of comparable magnitude, confirming that dissipation is balanced by the energy injected by the external drive.
The energy conservation relation $E(t) = E(0) + W_{\mathrm{input}}(t) - D_{\mathrm{diss}}(t)$ (Proposition~\ref{prop:harmonic_energy_evolution}) is verified numerically, confirming that the integrating-factor mechanism produces physically consistent dissipative behavior.
Regarding gradient accuracy (Panel~C), the dissipative LEP gradient estimates for all parameters ($\bm{m}$, $\bm{K}$, $\bzeta$) closely match the autodiff/BPTT ground truth, with relative Euclidean distance below $0.10$.
In contrast, the classical LEP baseline (which performs the sign flip but omits the exponential factors) yields relative distances above $1$ for $\bm{m}$ and $\bm{K}$, demonstrating that the sign-flipped damping alone is not sufficient---the exponential weighting in both the nudging and the learning rule is essential for correct gradient estimation in dissipative systems. Note that the classical LEP baseline produces an identically zero gradient for $\bzeta$: without the exponential weighting $e^{\bzeta t}$ in the learning rule, the damping coefficient does not appear in the gradient estimator, so no LEP bar is shown for $\bzeta$ in Panel~C.


\paragraph{Comparison with other approaches.} Kendall \textit{et al.} \citep{kendall2021gradient} proposed using fractional calculus to extend Lagrangian systems to dissipative dynamics, building on earlier work \citep{rieweNonconservativeLagrangianHamiltonian1996a}. While promising, this approach is limited to non-standard fractional dissipative elements, and they implicitly assumed fixed boundary conditions (equivalent to CBPVP), which would need to be reformulated as a PFVP for practical use. Another approach is to use periodic systems driven by periodic inputs \citep{bernemanEquilibriumPropagationPeriodic2025}, which also simplifies the boundary terms at the cost of restricting applicability to periodic systems (there is no need to match final conditions, but the input must be repeated multiple times). \citep{massarEquilibriumPropagationLearning2025} also proposed leveraging periodic systems (rather than the bouncing-backward kick used in our work), but introduced dissipation via the same exponential integrating factor as ours. Interestingly, all these approaches start from the mathematically guiding Lagrangian framework (see Section~\ref{subsec:lagrangian_hamiltonian}). Finally, a method to simulate dissipativity within a non-dissipative system was proposed by \citep{lopez-pastorSelfLearningMachinesBased2023}, where part of the system serves as an ancilla (an auxiliary subsystem that preserves information to maintain reversibility, a concept from reversible computing) for task-irrelevant information, yet can still be exploited for leveraging bouncing-backward kick.

\begin{figure}[t!]
    \centering
    \includegraphics[width=1.0\textwidth]{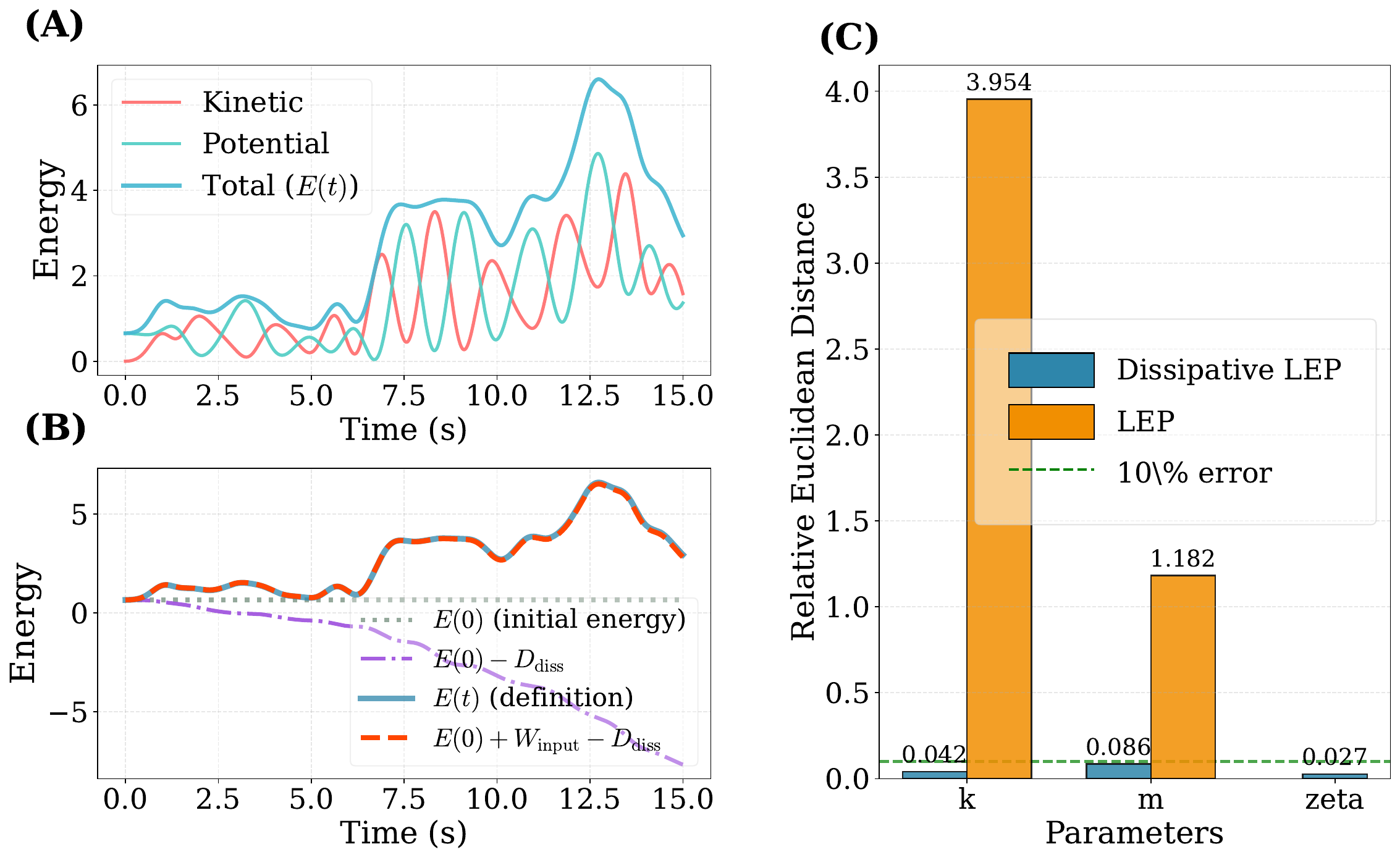}
    \caption{\textbf{Empirical validation of dissipative LEP} on $d=6$ coupled damped harmonic oscillators.
    \textbf{(A)}~Internal energy decomposition during the free phase: kinetic energy $E_{\mathrm{kin}}$, internal potential $U_{\mathrm{int}}$, and total energy $E$.
    \textbf{(B)}~Cumulative energy balance: initial energy $E(0)$ (grey dashed), cumulative dissipation $D_{\mathrm{diss}}$ (purple), input work $W_{\mathrm{input}}$ (red), and total energy $E(t)$.
    \textbf{(C)}~Relative Euclidean distance of gradient estimates to autodiff/BPTT for dissipative LEP (blue) and the classical LEP baseline (orange, which performs the sign-flipped damping but omits the exponential weighting in nudging and learning rule) across parameters $\bm{m}$, $\bm{K}$, and $\bzeta$.}
    \label{fig:dissipative_oscillators}
\end{figure}


\section{Discussions and Future Works}

\paragraph{Summary.} This work sets out to address two questions. 

(a) The first is whether \textit{EP can be generalised to design efficient and practically-implementable learning algorithms for time-varying inputs and outputs}. We show that it can, through Lagrangian Equilibrium Propagation (LEP) that extends EP's variational principles from steady states to entire physical trajectories, provided that the boundary conditions are chosen carefully. 

(b) The second question is \textit{how Hamiltonian Echo Learning algorithms relate to this generalised EP} framework. We show that RHEL is a special case of LEP obtained by combining the PFVP boundary conditions with the Legendre transformation.

A central finding is that \textit{the choice of boundary conditions has a decisive impact on whether the resulting learning algorithm is practical}. We show that the most natural choices lead to a trade-off between tractability of the gradient estimator and tractability of the trajectory computation. On one hand, the Constant Initial Value Problem (CIVP) yields causal, easy-to-simulate trajectories but introduces boundary residual terms that are hard to compute and requires explicit backward passes. On the other hand, the Constant Boundary Position Value Problem (CBPVP) eliminates these residuals but imposes non-causal boundary conditions that require an iterative boundary value solver. The Parametric Final Value Problem (PFVP), combined with time-reversibility, resolves this trade-off: it eliminates boundary residuals entirely while preserving causal, forward-only, streaming computation with no iterative solver overhead.

By combining this PFVP formulation with the Legendre transformation, we establish that RHEL is a special case of LEP. This reveals that RHEL's distinctive properties, namely local learning rules, forward-only computation, and the ``bouncing-backward'' echo phase, are not artifacts of Hamiltonian mechanics but arises naturally from the underlying variational structure.

Finally, we show that the variational framework of LEP provides guiding principles to extend these algorithms beyond conservative systems. By introducing an exponential integrating factor in the Lagrangian, dissipative dynamics can be accommodated within the PFVP framework provided the sign of the damping can be flipped during the echo phase. The variational derivation prescribes both the correct exponential weighting in the nudging and in the learning rule. Empirical validation on coupled damped harmonic oscillators confirms that this dissipative LEP gradient estimator accurately recovers BPTT gradients, and that omitting the prescribed weighting leads to incorrect gradients even when the sign-flipped damping is correctly applied.

\paragraph{Limitations and future directions.}
First, the PFVP formulation still requires an echo phase, \emph{i.e.} a second forward pass that can only begin after the free phase completes, making the algorithm inherently offline in the sense that gradients are not available during inference. Developing an online variant that eliminates this echo phase would bring LEP closer to Real-Time Recurrent Learning~\citep{williamsLearningAlgorithmContinually1989}, potentially offering more efficient alternatives to its notoriously high computational cost. 
Second, the elimination of boundary residuals relies on time-reversibility, which restricts applicability to conservative (or sign-controllable dissipative) systems. Extending the PFVP formulation beyond time-reversible systems, or identifying weaker sufficient conditions for boundary residual cancellation, would broaden its applicability. 
Third, while RHEL has been tested at larger scale on state-space models~\citep{rhel}, neither LEP nor dissipative LEP have been validated on real physical systems. 
These advances would further solidify the theoretical foundation for physics-based learning algorithms that unify inference and training within single physical systems, offering promising alternatives to conventional digital computing paradigms for future neuromorphic and analog computing architectures.

\paragraph{Reproducibility.} Code to reproduce the experiments will be made available.

\bibliographystyle{alpha}
\bibliography{sample}

\newpage
\appendix 
\part{Appendix}
\parttoc
\newpage
\section{Derivative and shape conventions}
\label{appx:sec:derivative_conventions}

Throughout the paper, we adopt a fixed coordinate-based matrix convention. 
State variables are represented as column vectors
$s \in \mathbb{R}^{d_s}$ and parameters as column vectors $\theta \in \mathbb{R}^{d_\theta}$.

Gradients with respect to state variables (including velocities) are column vectors:
\begin{equation*}
\partial_s L \in \mathbb{R}^{d_s},
\qquad
\partial_{\dot s} L \in \mathbb{R}^{d_s}\,.
\end{equation*}

Jacobians with respect to parameters are matrices acting on parameter variations:
\begin{equation*}
\partial_\theta s \in \mathbb{R}^{d_s \times d_\theta},
\qquad
\partial^2_{\theta, \dot s} L \in \mathbb{R}^{d_s \times d_\theta}\,.
\end{equation*}

\paragraph{Total derivatives vs partial derivatives.}
We distinguish between \emph{partial derivatives} and \emph{total derivatives}:
\begin{itemize}[leftmargin=*]
    \item $\partial_\theta$ denotes the \emph{partial derivative} with respect to $\theta$, holding all other explicit arguments fixed.
    \item $d_\theta$ denotes the \emph{total derivative} with respect to $\theta$, accounting for both explicit and implicit dependencies through the chain rule.
\end{itemize}
For example, $d_\theta \partial_{\dot s} L$ is a total derivative (Jacobian) with shape $\mathbb{R}^{d_s \times d_\theta}$ that accounts for how $\partial_{\dot s} L$ changes with $\theta$ through all dependencies, including implicit ones through the state variables.

Scalar or parameter-wise quantities are obtained via standard matrix multiplication.
In particular, for any $v \in \mathbb{R}^{d_s}$,
\begin{equation*}
(\partial^2_{\theta, \dot s} L)^\top v \;\in\; \mathbb{R}^{d_\theta}\,,
\end{equation*}
where the transpose denotes the usual matrix transpose and ensures dimensional consistency.
Equivalently, this corresponds componentwise to
\begin{equation*}
\big[(\partial^2_{\theta, \dot s} L)^\top v\big]_j
=
\sum_{i=1}^{d_s}
\partial^2_{\theta_j, \dot{s}_i} L \, v_i \,.
\end{equation*}

All transposes appearing in the paper are genuine matrix transposes introduced to make
matrix products well-defined; no implicit row/column conventions are assumed.
Under this convention, all gradient expressions and boundary residual terms are dimensionally consistent.

\paragraph{Functional (variational) derivative.}
We denote by $\delta_s A$ the \emph{functional derivative} (or \emph{variational derivative}) of a functional $A[s]$ with respect to the trajectory $s$. It is defined implicitly through the directional derivative: for any smooth variation $\eta$, $\delta_s A \cdot \eta \defn \left. d_\epsilon \right|_{\epsilon=0} A[s + \epsilon \eta]$. Informally, $\delta_s A$ is the infinite-dimensional analogue of a gradient: it captures how $A$ responds to infinitesimal deformations of the trajectory.

\section{Glossary}
\label{appx:sec:glossary}

Table~\ref{tab:lep_formulations} summarizes the different boundary condition formulations for Lagrangian Equilibrium Propagation (LEP) discussed in this paper. Each formulation defines how trajectories $\bs_t^\beta(\btheta)$ are specified through boundary conditions, leading to distinct computational properties and practical implications.

\begin{table}[H]
\centering
\caption{\textbf{Summary of LEP Boundary Condition Formulations.} Comparison of the three main boundary condition formulations for Lagrangian Equilibrium Propagation.}
\label{tab:lep_formulations}
\begin{tabular}{lp{10cm}l}
\toprule
\textbf{Acronym} & \textbf{Definition} & \textbf{Section} \\
\midrule
CIVP & \textbf{Constant Initial Value Problem:} All trajectories share fixed initial conditions independent of $\btheta$ and $\beta$. Defined by:
$$\forall t \in [0,T] \quad t \mapsto \bs_{\fwda,t}^\beta(\boldsymbol{{\btheta}}, (\boldsymbol{\alpha}_0, \boldsymbol{\gamma}_0)) 
\text{ satisfies:} \quad
\begin{cases}
\text{EL}(t, \boldsymbol{{\btheta}}, \beta) = 0 \\
\bs_{\fwda,0}^\beta(\boldsymbol{{\btheta}}) = \boldsymbol{\alpha}_0 \\
\dot{\bs}_{\fwda,0}^\beta(\boldsymbol{{\btheta}}) = \boldsymbol{\gamma}_0
\end{cases}$$
Causal boundary conditions: forward integration from $t=0$. Suffers from intractable boundary residuals. & \ref{subsec:civp} \\
\midrule
CBPVP & \textbf{Constant Boundary Position Value Problem:} All trajectories satisfy fixed position boundary conditions at both endpoints, independent of $\btheta$ and $\beta$. Defined by:
$$\forall t \in [0,T], \quad t \mapsto \bs_{\bda,t}^\beta(\boldsymbol{{\btheta}}, (\boldsymbol{\alpha}_0, \boldsymbol{\alpha}_T)) 
\text{ satisfies:} \quad
\begin{cases}
\text{EL}(t, \boldsymbol{{\btheta}}, \beta) = 0 \\
\bs_{\bda,0}^\beta(\boldsymbol{{\btheta}}) = \boldsymbol{\alpha}_0 \\
\bs_{\bda,T}^\beta(\boldsymbol{{\btheta}}) = \boldsymbol{\alpha}_T
\end{cases}$$
Non-causal boundary conditions: requires solving a two-point boundary value problem. Eliminates boundary residuals but computationally expensive. & \ref{subsec:cfvp} \\
\midrule
PFVP & \textbf{Parametric Final Value Problem:} Final boundary conditions depend on parameters $\btheta$. Defined by:
$$\forall t \in [0,T] \quad t \mapsto \bs_{\bwda,t}^\beta(\btheta, (\balpha_T(\btheta), \bgamma_T(\btheta))) \text{ satisfies:}\quad
\begin{cases}
\text{EL}_{r}(t, \btheta, \beta) = 0 \\
\bs_{\bwda,T}^\beta(\btheta) = \balpha_T(\btheta) \\
\dot{\bs}_{\bwda,T}^\beta(\btheta) = \bgamma_T(\btheta)
\end{cases}$$
where parametric boundaries are derived from CIVP free phase: $\balpha_T(\btheta) = \bs_{\fwda,T}^0(\btheta, (\balpha_0, \bgamma_0))$ and $\bgamma_T(\btheta) = \dot{\bs}_{\fwda,T}^0(\btheta, (\balpha_0, \bgamma_0))$. \\
\bottomrule
\end{tabular}
\end{table}

\section{Preparatory results}

\begin{remark}[Regularity and uniqueness of solutions]
\label{rmk:regularity}
Several proofs in this paper invoke uniqueness of solutions to initial value problems.
The classical sufficient condition (the uniqueness part of the Picard--Lindel\"of theorem) is that the Euler--Lagrange equation, once rewritten as a first-order system
$\dot{\bm{z}}=\bm{f}(t,\bm{z})$, has a right-hand side $\bm{f}$ that is locally Lipschitz in~$\bm{z}$.
Two ingredients are needed:
\begin{enumerate}
    \item \textbf{Mass-matrix invertibility.} The Hessian $\partial^2_{\dot{\bs},\dot{\bs}}L$ must be invertible so that $\ddot{\bs}$ can be expressed as a function of $(\bs,\dot{\bs},t)$.
    This is already required for the Legendre transform (Proposition~\ref{prop:Legendre_transform}).
    \item \textbf{Local Lipschitz continuity of the resulting right-hand side.}
    When $L\in C^2$, the implicit-function theorem guarantees that the map $(\bs,\dot{\bs},t)\mapsto\ddot{\bs}$ is $C^1$, hence locally Lipschitz.
\end{enumerate}
\textbf{Verification for the models of Table~\ref{tab:ml_hamiltonians}.}
For all three models the mass matrix is either the identity (\textbf{UniCORNN}, \textbf{LinOSS}) or $\mathrm{diag}(\bm{\tau})$ with $\tau_i>0$ (\textbf{Hopfield}), hence invertible.
Moreover, the right-hand sides are in fact \emph{globally} Lipschitz:
\textbf{LinOSS} is linear in~$\bm{z}$;
\textbf{UniCORNN} involves $\tanh$, which is globally Lipschitz (derivative bounded by~$1$) composed with a linear map;
\textbf{Hopfield} involves products of $\tanh$ and $\tanh'$, both globally bounded, composed with linear maps---the Lipschitz constant depends on $\|W\|$ but remains finite for any fixed~$W$.
Global Lipschitz continuity ensures that solutions exist and are unique on any interval $[0,T]$.
\end{remark}

\begin{lemma}[Odd derivative property of reversible Lagrangian]\label{lemma:odd_derivative}
For a reversible Lagrangian $L_{\beta}$ that satisfies $L_{\beta}(\bs,\dot{\bs},\btheta) = L_{\beta}(\bs,-\dot{\bs},\btheta)$, the derivative with respect to $\dot{\bs}$ satisfies:
\begin{equation*}
    \partial_{\dot{\bs}} L_{\beta}(\bs,-\dot{\bs},\btheta) = -\partial_{\dot{\bs}} L_{\beta}(\bs,\dot{\bs},\btheta)\,.
\end{equation*}
\end{lemma}
\begin{proof}
Since the Lagrangian $L_{\beta}$ is reversible, it satisfies
\begin{equation*}
    L_{\beta}(\bs,\dot{\bs},\btheta) = L_{\beta}(\bs,-\dot{\bs},\btheta)\,,
\end{equation*}
i.e.\ it is even in $\dot{\bs}$. Consequently, its derivative with respect to $\dot{\bs}$ is odd, yielding the stated result.
\end{proof}





\begin{proposition}[Least Action principle for parametrized perturbations]
\label{prop:chain-rule-varvect}
Let $A[\bs(\btheta),\btheta]=\int_0^T L(t,\btheta,\bs_t(\btheta),\dot{\bs}_t(\btheta))dt$ be a scalar functional of an arbitrary function $\bs(\btheta)$ that depends on some parameter vector $\btheta \in \R^p$. Further, $A$ also has an explicit dependence on $\btheta$. Here, $\btheta$ is a non-time-varying parameter.

If $\bs(\btheta)$ satisfies the Euler-Lagrange equations $\partial_{\bs} L - d_t\partial_{\dot{\bs}} L=0$, then the implicit variation of $A$ through $\btheta$ via $\bs$ reduces to boundary terms:
\begin{align*}
\delta_{\bs} A[\bs(\btheta),\btheta] \delta_{\btheta} \bs(\btheta) = \left[\left(\partial_{\btheta} \bs_t(\btheta)\right)^\top \cdot\partial_{\dot{\bs}}L\left(t,\btheta, \bs_t(\btheta), \dot{\bs}_t(\btheta)\right) \right]_0^T \,.
\end{align*}
\textbf{Implicit variation (definition):} The implicit variation along each component $\theta_i$ is defined as the change in $A$ due to $\theta_i$ acting \emph{only} through $\bs$, with the explicit $\btheta$-dependence of $A$ held fixed:
\begin{align}
    \delta_{\bs} A[\bs(\btheta),\btheta] \delta_{\theta_i} \bs(\btheta) \defn \left. d_\epsilon \right|_{\epsilon=0} A[\bs(\btheta + \epsilon e_i), \btheta] \,.
    \label{eq:implicit-variation-def}
\end{align}
\textbf{Notation:} Here $e_i$ denotes the $i$-th canonical basis vector in $\R^p$ (the parameter space of $\btheta$). Each $\delta_{\bs} A \delta_{\theta_i} \bs$ is a scalar, and the full vector form $\delta_{\bs} A[\bs(\btheta),\btheta] \delta_{\btheta} \bs(\btheta) \in \R^{p}$ is the vector obtained by concatenation: $\delta_{\bs} A \delta_{\btheta} \bs = \left(\delta_{\bs} A \delta_{\theta_1} \bs, \ldots, \delta_{\bs} A \delta_{\theta_p} \bs\right)^\top$.
\end{proposition}

\begin{proof}
We prove the result component-wise using Lemma~\ref{lemma:euler-lagrange}. Fix a component $\theta_i$ and consider the perturbation $\bmeta(\epsilon) \defn \bs(\btheta + \epsilon e_i) - \bs(\btheta)$, which satisfies $\bmeta(0) = \mathbf{0}$ and $\left.\partial_\epsilon\right|_{\epsilon=0} \bmeta_t(\epsilon) = \partial_{\theta_i} \bs_t(\btheta)$. From definition~\eqref{eq:implicit-variation-def}:
\begin{align*}
    \delta_{\bs} A[\bs(\btheta),\btheta] \delta_{\theta_i} \bs(\btheta) = \left. d_\epsilon \right|_{\epsilon=0} A[\bs(\btheta + \epsilon e_i), \btheta] = \left. d_\epsilon \right|_{\epsilon=0} A[\bs(\btheta) + \bmeta(\epsilon), \btheta] \,.
\end{align*}
Applying Lemma~\ref{lemma:euler-lagrange} to parametric perturbations (see note after the Lemma and proof in Appendix~\ref{appx:sec:proof-generalized-lemma}), since $\bs(\btheta)$ satisfies the Euler-Lagrange equations:
\begin{align*}
    \delta_{\bs} A[\bs(\btheta),\btheta] \delta_{\theta_i} \bs(\btheta) = \left[\left(\partial_{\theta_i} \bs_t(\btheta)\right)^\top \cdot \partial_{\dot{\bs}} L\left(t, \btheta, \bs_t(\btheta), \dot{\bs}_t(\btheta)\right)\right]_0^T \,.
\end{align*}
Concatenating over all components $i = 1, \ldots, p$ yields the full vector result. The same analysis applies with respect to any other parameter (e.g., $\beta$).
\end{proof}



\begin{proposition}[Equivalence between CIVP and PFVP]\label{prop:equivalence_IVP_EVP}
The PFVP free solution that terminates at the final state of the corresponding CIVP free solution
coincides with that CIVP free solution. Namely, for any $(\balpha_0,\bgamma_0)$ and all $t\in[0,T]$,
\begin{equation*}
\bs_{\bwda,t}^0\!\Bigl(\btheta,\bigl(\bs_{\fwda,T}^0(\btheta,(\balpha_0,\bgamma_0)),
\dot{\bs}_{\fwda,T}^0(\btheta,(\balpha_0,\bgamma_0))\bigr)\Bigr)
=
\bs_{\fwda,t}^0(\btheta,(\balpha_0,\bgamma_0))\,.
\end{equation*}
\end{proposition}

\begin{proof}
Define the terminal state of the CIVP trajectory by
\begin{equation*}
(\balpha_T,\bgamma_T)
:=
\bigl(\bs_{\fwda,T}^0(\btheta,(\balpha_0,\bgamma_0)), \dot{\bs}_{\fwda,T}^0(\btheta,(\balpha_0,\bgamma_0))\bigr)\,.
\end{equation*}
By definition of the PFVP (Definition~\ref{def:PFVP-boundary-from-CIVP}), the trajectory
$t\mapsto \bs_{\bwda,t}^0(\btheta,(\balpha_T,\bgamma_T))$ is a solution of the same Euler--Lagrange
dynamics as $t\mapsto \bs_{\fwda,t}^0(\btheta,(\balpha_0,\bgamma_0))$ and satisfies the terminal condition
\begin{equation*}
\bigl(\bs_{\bwda,T}^0(\btheta,(\balpha_T,\bgamma_T)),
\dot{\bs}_{\bwda,T}^0(\btheta,(\balpha_T,\bgamma_T))\bigr)
=
(\balpha_T,\bgamma_T)\,.
\end{equation*}

Therefore, the two trajectories share the same state at time $T$ while solving the same ODE.
By uniqueness of solutions (Remark~\ref{rmk:regularity}), they must coincide on $[0,T]$, which proves the claim.
\end{proof}

\begin{lemma}[IVP-FVP equivalence for reversible Hamiltonian systems]\label{lemma:equivalence_IVP_EVPhamiltonian}
For a reversible Hamiltonian system, the IVP solution starting from momentum-flipped 
initial conditions is equivalent to the time-reversed FVP solution:
\begin{align*}
    \forall t \in [0,T] \quad \bPhi_{IVP,t}(\btheta, \bSigma_z\blambda_0) = \bSigma_z\bPhi_{FVP,T-t}(\btheta, \blambda_0) \,,
\end{align*}
where $\bSigma_z = \begin{pmatrix} I & 0 \\ 0 & -I \end{pmatrix}$ is the momentum-flipping operator.
\end{lemma}
\begin{proof}[Proof by reference and relation to Proposition~2]
\textbf{(1) Analogy with Proposition~2.}
Proposition~2 proves reversibility of the time-reversible PFVP solution in the Lagrangian formulation: the FVP can be reduced to an IVP by applying the time-reversal symmetry (reverse time and flip the time-odd variable, namely the velocity).
The present lemma (Lemma~\ref{lemma:equivalence_IVP_EVPhamiltonian}) is the Hamiltonian analogue of that statement: in Hamiltonian coordinates, time reversal acts by leaving the position unchanged and flipping the conjugate momentum. Hence the same ``FVP $\leftrightarrow$ IVP'' conversion holds, with velocity flip replaced by momentum flip.

\textbf{(2) Proof is contained in the RHEL paper.}
This reversibility property is established in the RHEL paper~\cite{rhel}, Appendix~A.1.
Specifically, Lemma~A.3 in~\cite{rhel} proves time-reversal invariance of the Hamiltonian dynamics under the momentum-flip involution (with the appropriate time-reversal of the forcing/input), and Corollary~A.3 in~\cite{rhel} deduces the corresponding reversibility/echo (trajectory retracing) property.
The present lemma is exactly the specialization of these results to the IVP--FVP equivalence stated here.
\end{proof}


\begin{lemma}[Parameter-gradient relation under Legendre transform]
\label{lemma:param_grad}
Let $t\mapsto \bPhi_t = (\bs_t, \bp_t)$ be a solution of Hamilton's equations with Hamiltonian $H(\bPhi_t, \btheta)$. Let $t \mapsto (\bs_t,\dot{\bs}_t)$ be the associated Lagrangian trajectory defined through the backward Legendre transform (Proposition~\ref{prop:Legendre_transform}). Then:
\begin{equation*}
    \partial_{\btheta} H(\bPhi_t,\btheta) = - \partial_{\btheta} L(\bs_t,\dot{\bs}_t,\btheta)\,.
\end{equation*}
\end{lemma}

\begin{proof}
The momentum is defined by the (forward) Legendre transform (Proposition~\ref{prop:Legendre_transform}\hyperlink{prop:Legendre_transform:a}{(a)}):
\begin{align}
    \bp_t := \partial_{\dot{\bs}} L(\bs_t, \dot{\bs}_t, \btheta)\,.
    \label{eq:momentum_definition_appxA}
\end{align}

The Hamiltonian is constructed as:
\begin{align*}
    H(\bPhi_t, \btheta) := \bp_t^\top \dot{\bs}_t - L(\bs_t, \dot{\bs}_t, \btheta)\,,
\end{align*}
where $\dot{\bs}_t$ is determined implicitly by $(\bs_t, \bp_t, \btheta)$ through Eq.~\eqref{eq:momentum_definition_appxA}. In particular, when the Legendre transform is well-defined (i.e., $\partial^2_{\dot{\bs}, \dot{\bs}} L$ is invertible), we may invert $\bp_t = \partial_{\dot{\bs}} L(\bs_t, \dot{\bs}_t, \btheta)$ to obtain a (local) implicit function $\dot{\bs}_t = \dot{\bs}(\bs_t, \bp_t, \btheta)$. Thus $\dot{\bs}_t$ is not an independent variable here: once $\bPhi_t=(\bs_t,\bp_t)$ is held fixed, the right-hand side is understood as a function of $(\bPhi_t,\btheta)$ only.

Taking the derivative with respect to $\btheta$ (holding $\bPhi_t = (\bs_t, \bp_t)$ fixed):
\begin{align*}
    \partial_{\btheta} H(\bPhi_t, \btheta) 
    &= \partial_{\btheta}\left[\bp_t^\top \dot{\bs}_t - L(\bs_t, \dot{\bs}_t, \btheta)\right] \\
    &= \bp_t^\top \partial_{\btheta}\dot{\bs}_t - \partial_{\btheta} L(\bs_t, \dot{\bs}_t, \btheta) - (\partial_{\dot{\bs}} L)^\top \partial_{\btheta}\dot{\bs}_t\,.
\end{align*}

Here, $\partial_{\btheta}\dot{\bs}_t$ represents the derivative of the implicit function $\dot{\bs}_t(\bs_t, \bp_t, \btheta)$ with respect to $\btheta$. This derivative captures how the velocity $\dot{\bs}_t$ changes with $\btheta$ while keeping the Hamiltonian state $(\bs_t, \bp_t)$ fixed.

The key observation is that the first and third terms cancel:
\begin{align*}
    \bp_t^\top \partial_{\btheta}\dot{\bs}_t - (\partial_{\dot{\bs}} L)^\top \partial_{\btheta}\dot{\bs}_t 
    &= \bp_t^\top \partial_{\btheta}\dot{\bs}_t - \bp_t^\top \partial_{\btheta}\dot{\bs}_t \quad \text{(using Eq.~\eqref{eq:momentum_definition_appxA})} \\
    &= 0\,.
\end{align*}

Therefore:
\begin{align*}
    \partial_{\btheta} H(\bPhi_t, \btheta) 
    &= -\partial_{\btheta} L(\bs_t, \dot{\bs}_t, \btheta)\,.
\end{align*}
\end{proof}

\begin{lemma}[Independence of augmented Lagrangian and Hamiltonian derivatives]
\label{lemma:beta_independent_momentum}
Let $L_\beta$ be the augmented Lagrangian defined in Equations~\eqref{eqdef:L_beta} where the cost term $c$ does not depend on $\dot{\bs}$ or $\btheta$. Let $H_\beta$ be the corresponding Hamiltonian obtained via Legendre transform. Then, for all
$(\bs,\dot{\bs},\btheta)$ (or $(\bPhi,\btheta)$ for Hamiltonian) and all $\beta$:
\begin{enumerate}
    \item $\partial_{\dot{\bs}} L_\beta(\bs,\dot{\bs},\btheta) = \partial_{\dot{\bs}} L_0(\bs,\dot{\bs},\btheta)$ \quad (Lagrangian velocity derivative)
    \item $\partial_{\btheta} L_\beta(\bs,\dot{\bs},\btheta) = \partial_{\btheta} L_0(\bs,\dot{\bs},\btheta)$ \quad (Lagrangian parameter derivative)
    \item $\partial_{\btheta} H_\beta(\bPhi,\btheta) = \partial_{\btheta} H_0(\bPhi,\btheta)$ \quad (Hamiltonian parameter derivative)
\end{enumerate}

\textit{Note:} The result also holds for the Hamiltonian momentum derivative: $\partial_{\bp} H_\beta(\bs,\bp,\btheta) = \partial_{\bp} H_0(\bs,\bp,\btheta)$, though this property is not needed in this paper.
\end{lemma}
\begin{proof}
Since $L_\beta = L_0 + \beta c$ where $c$ depends only on $\bs$ (not on $\dot{\bs}$ or $\btheta$), the first two properties follow immediately. For the third property, since $H_\beta$ is obtained from $L_\beta$ via Legendre transform and the transform preserves parameter derivatives (as shown in Lemma~\ref{lemma:param_grad}), we have $\partial_{\btheta} H_\beta = -\partial_{\btheta} L_\beta = -\partial_{\btheta} L_0 = \partial_{\btheta} H_0$.
\end{proof}

\section{Proof of Proposition~\ref{prop:Legendre_transform}: Legendre transform}
\begin{proof} [Proof of Proposition~\ref{prop:Legendre_transform}]
We first justify the forward transform, then the backward one, and finally the equivalence of the Hessian conditions.

\medskip\noindent
\phantomsection\hypertarget{prop:Legendre_transform:a}{}%
	\textbf{(a) Forward transform.}
Fix $t$ and regard $L$ as a function of $(\bs_t,\dot{\bs}_t)$. Define
\begin{equation*}
\bp_t \;=\; \partial_{\dot{\bs}} L(\bs_t,\dot{\bs}_t)\,.
\end{equation*}
For fixed $\bs_t$, the Jacobian of the map
\begin{equation*}
\dot{\bs}_t \mapsto \bp_t
\end{equation*}
is precisely the Hessian $\partial^2_{\dot{\bs}, \dot{\bs}}L(\bs_t,\dot{\bs}_t)$.

If
\begin{equation*}
\det\big(\partial^2_{\dot{\bs}, \dot{\bs}}L(\bs_t,\dot{\bs}_t)\big)\neq 0\,,
\end{equation*}
then, by the inverse function theorem, this map is locally invertible: there exists a unique smooth function
\begin{equation*}
\dot{\bs}_t = \dot{\bs}_t(\bs_t,\bp_t)
\end{equation*}
in a neighbourhood of $(\bs_t,\dot{\bs}_t)$. We then define the Hamiltonian
\begin{equation*}
H(\bs_t,\bp_t) \;=\; \bp_t^\top \dot{\bs}_t(\bs_t,\bp_t) - L\big(\bs_t,\dot{\bs}_t(\bs_t,\bp_t)\big)\,,
\end{equation*}
which yields the forward transform
\begin{equation*}
(\bs_t,\dot{\bs}_t) \mapsto (\bs_t,\bp_t), 
\qquad
\bp_t = \partial_{\dot{\bs}} L, 
\quad
H = \bp_t^\top \dot{\bs}_t - L\,,
\end{equation*}
and is locally one-to-one with smooth inverse $(\bs_t,\bp_t)\mapsto (\bs_t,\dot{\bs}_t)$.

\medskip\noindent
\textbf{(b) Backward transform.}
Conversely, fix $t$ and regard $H$ as a function of $(\bs_t,\bp_t)$, and define
\begin{equation*}
\dot{\bs}_t \;=\; \partial_{\bp} H(\bs_t,\bp_t)\,.
\end{equation*}
For fixed $\bs_t$, the Jacobian of the map
\begin{equation*}
\bp_t \mapsto \dot{\bs}_t
\end{equation*}
is the Hessian $\partial^2_{\bp, \bp}H(\bs_t,\bp_t)$. 

If
\begin{equation*}
\det\big(\partial^2_{\bp, \bp}H(\bs_t,\bp_t)\big)\neq 0\,,
\end{equation*}
then, again by the inverse function theorem, this map is locally invertible, so there exists a unique smooth function
\begin{equation*}
\bp_t = \bp_t(\bs_t,\dot{\bs}_t)\,,
\end{equation*}
and we can define the Lagrangian via
\begin{equation*}
L(\bs_t,\dot{\bs}_t) \;=\; \bp_t(\bs_t,\dot{\bs}_t)^\top \dot{\bs}_t - H\big(\bs_t,\bp_t(\bs_t,\dot{\bs}_t)\big)\,,
\end{equation*}
which yields the backward transform
\begin{equation*}
(\bs_t,\bp_t) \mapsto (\bs_t,\dot{\bs}_t),
\qquad
\dot{\bs}_t = \partial_{\bp} H,
\quad
L = \bp_t^\top \dot{\bs}_t - H\,,
\end{equation*}
again locally one-to-one with smooth inverse.

\medskip\noindent
\textbf{Equivalence of Hessian conditions.}
Assume $L$ and $H$ are related by the Legendre transform as above. For fixed $\bs_t$, we have
\begin{equation*}
\bp_t \;=\; \partial_{\dot{\bs}}L(\bs_t,\dot{\bs}_t),
\qquad
\dot{\bs}_t \;=\; \partial_{\bp} H(\bs_t,\bp_t)\,.
\end{equation*}
Differentiate the first relation w.r.t.\ $\dot{\bs}_t$:
\begin{equation*}
\partial_{\dot{\bs}_t} \bp_t
\;=\;
\partial^2_{\dot{\bs}, \dot{\bs}}L(\bs_t,\dot{\bs}_t)\,.
\end{equation*}
Differentiate the second relation w.r.t.\ $\bp_t$:
\begin{equation*}
\partial_{\bp_t} \dot{\bs}_t
\;=\;
\partial^2_{\bp, \bp}H(\bs_t,\bp_t)\,.
\end{equation*}
Since the maps $\dot{\bs}_t \mapsto \bp_t$ and $\bp_t \mapsto \dot{\bs}_t$ are inverse to each other (for fixed $\bs_t$), their Jacobians are matrix inverses:
\begin{equation*}
\partial_{\bp_t} \dot{\bs}_t
=
\left(\partial_{\dot{\bs}_t} \bp_t\right)^{-1}\,.
\end{equation*}
Hence
\begin{equation*}
\partial^2_{\bp, \bp}H(\bs_t,\bp_t)
=
\big(\partial^2_{\dot{\bs}, \dot{\bs}}L(\bs_t,\dot{\bs}_t)\big)^{-1}\,.
\end{equation*}
In particular,
\begin{equation*}
\det(\partial^2_{\dot{\bs}, \dot{\bs}}L)\neq 0
\quad\Longleftrightarrow\quad
\det(\partial^2_{\bp, \bp}H)\neq 0\,,
\end{equation*}
so the forward and backward non-singularity conditions are equivalent, and the Legendre transform is invertible in both directions.
\end{proof}

\section{Proof of Lemma~\ref{lemma:euler-lagrange}. Euler-Lagrange solutions and the action functional}
\label{appx:sec:proof-generalized-lemma}
\begin{proof}[Proof of Lemma~\ref{lemma:euler-lagrange}]
We prove the result for a general smooth perturbation $\epsilon \mapsto \bmeta(\epsilon)$ with $\bmeta(0) = \mathbf{0}$ (as noted after Lemma~\ref{lemma:euler-lagrange}); the linear case $\bmeta(\epsilon) = \epsilon \bmeta$ follows by specialization. 

Expanding the action functional along the perturbation:
\begin{align*}
    A_\beta[\mathbf{s}^\beta + \bmeta(\epsilon)] 
    &= \int_0^T L_\beta\!\left(\bs_t^\beta + \bmeta_t(\epsilon),\; \dot{\bs}_t^\beta + \dot{\bmeta}_t(\epsilon),\; \btheta\right) \dt \,,
\end{align*}
where $\dot{\bmeta}_t(\epsilon) \defn d_t \bmeta_t(\epsilon)$ denotes the time derivative at fixed $\epsilon$. Differentiating with respect to $\epsilon$ and evaluating at $\epsilon = 0$, the chain rule gives:
\begin{align*}
    \left. d_\epsilon \right|_{\epsilon=0} A_\beta[\mathbf{s}^\beta + \bmeta(\epsilon)] 
    &= \int_0^T \left[\left(\left.\partial_\epsilon\right|_{\epsilon=0} \bmeta_t(\epsilon)\right)^\top \partial_{\bs} L_\beta(\bs_t^\beta, \dot{\bs}_t^\beta, \btheta) 
    + \left(d_t \left.\partial_\epsilon\right|_{\epsilon=0} \bmeta_t(\epsilon)\right)^\top \partial_{\dot{\bs}} L_\beta(\bs_t^\beta, \dot{\bs}_t^\beta, \btheta) \right] \dt \,.
\end{align*}
Here we used $\bmeta(0) = \mathbf{0}$ (i.e., $\bmeta_t(0) = \mathbf{0}$ and $\dot{\bmeta}_t(0) = \mathbf{0}$ for all $t$) to ensure that the partial derivatives of $L_\beta$ are evaluated at the original trajectory $(\bs_t^\beta, \dot{\bs}_t^\beta)$. We also used the commutativity of $\partial_\epsilon$ and $d_t$ (valid by smoothness) to write $\left.\partial_\epsilon\right|_{\epsilon=0} \dot{\bmeta}_t(\epsilon) = d_t \left.\partial_\epsilon\right|_{\epsilon=0} \bmeta_t(\epsilon)$.
Applying integration by parts to the second term:
\begin{align*}
    &= \int_0^T \left(\left.\partial_\epsilon\right|_{\epsilon=0} \bmeta_t(\epsilon)\right)^\top \left[\partial_{\bs} L_\beta - d_t \partial_{\dot{\bs}} L_\beta\right] \dt 
    + \left[\left(\left.\partial_\epsilon\right|_{\epsilon=0} \bmeta_t(\epsilon)\right)^\top \partial_{\dot{\bs}} L_\beta(\bs_t^\beta, \dot{\bs}_t^\beta, \btheta)\right]_0^T \,.
\end{align*}
Since $\mathbf{s}^\beta$ satisfies the Euler-Lagrange equation $\EL(t, \btheta, \beta) = \partial_{\bs} L_\beta - d_t \partial_{\dot{\bs}} L_\beta = 0$, the integral vanishes, yielding:
\begin{align}
    \left. d_\epsilon \right|_{\epsilon=0} A_\beta[\mathbf{s}^\beta + \bmeta(\epsilon)] 
    = \left[\left(\left.\partial_\epsilon\right|_{\epsilon=0} \bmeta_t(\epsilon)\right)^\top \partial_{\dot{\bs}} L_\beta(\bs_t^\beta, \dot{\bs}_t^\beta, \btheta)\right]_0^T \,.
    \label{eq:proof-general-variation}
\end{align}
This establishes the general case for parametric perturbations (see note after Lemma~\ref{lemma:euler-lagrange}). For the linear perturbation $\bmeta(\epsilon) = \epsilon \bmeta$, we have $\left.\partial_\epsilon\right|_{\epsilon=0} \bmeta(\epsilon) = \bmeta$, and Eq.~\eqref{eq:proof-general-variation} becomes:
\begin{align*}
    \delta_{\bs} A_\beta \cdot \bmeta = \left. d_\epsilon \right|_{\epsilon=0} A_\beta[\mathbf{s}^\beta + \epsilon \bmeta] 
    = \left[\bmeta_t^\top \partial_{\dot{\bs}} L_\beta(\bs_t^\beta, \dot{\bs}_t^\beta, \btheta)\right]_0^T \,.
\end{align*}
This establishes Case~2 (the general formula for arbitrary $\bmeta$). Case~1 follows immediately: when $\bmeta_0 = \bmeta_T = \mathbf{0}$, the boundary terms vanish and $\delta_{\bs} A_\beta \cdot \bmeta = 0$.
\end{proof}

\section{Proof Theorem~\ref{thm:general-ep}. Lagrangian EP 
gradient estimator}
\label{appx:sec:proof-general-ep}
\begin{proof}[Proof of Theorem~\ref{thm:general-ep}]

We consider the cross-derivatives of the action functional 
$A_{\beta}[\bs^\beta(\btheta), \btheta]$. 
Since $A_\beta$ depends on $\btheta$ both \emph{explicitly} (through $L_\beta(\cdot, \cdot, \btheta)$) and \emph{implicitly} (through the trajectory $\bs^\beta(\btheta)$), the total derivative decomposes as $d_{\btheta} A_\beta = \partial_{\btheta} A_\beta + \delta_{\bs} A_\beta \, \delta_{\btheta} \bs^\beta$, where $\partial_{\btheta}$ acts on the explicit $\btheta$-dependence at fixed trajectory and $\delta_{\bs} A_\beta \, \delta_{\btheta} \bs^\beta$ captures the implicit variation through $\bs^\beta(\btheta)$ (see Proposition~\ref{prop:chain-rule-varvect} and the notation conventions in Appendix~\ref{appx:sec:derivative_conventions}).

First, differentiating with respect to $\btheta$ then $\beta$ (at $\beta=0$):
\begin{align}
    d_{\beta\btheta} A_{\beta}[\bs^\beta(\btheta),\btheta] 
    &= d_\beta |_{\beta=0} \left(\partial_{\btheta} A_{\beta}[\bs^\beta(\btheta),\btheta] + \delta_{\bs} A_{\beta} \delta_{\btheta} \bs^\beta \right) \nonumber\\
    &= d_\beta |_{\beta=0} \int_0^T \partial_{\btheta} L_{\beta}\left( \bs_t^{\beta}, \dot{\bs}_t^{\beta}, \btheta\right) \dt 
       + d_\beta |_{\beta=0}(\delta_{\bs} A_{\beta} \delta_{\btheta} \bs^\beta) \,.
    \label{eq:beta-theta-crossderivative}
\end{align}

Second, differentiating with respect to $\beta$ (at $\beta=0$) then $\btheta$:
\begin{align}
    d_{\btheta\beta} A_{\beta}[\bs^\beta(\btheta),\btheta] 
    &= d_{\btheta}\left(\partial_\beta |_{\beta=0} A_{\beta}[\bs^\beta(\btheta), \btheta] + \delta_{\bs} A_{0} \delta_\beta |_{\beta=0} \bs^\beta \right) \nonumber \\
    &= d_{\btheta}\left(C[\bs^0(\btheta)] + \delta_{\bs} A_{0} \delta_\beta  |_{\beta=0}\bs^\beta \right) \,,
    \label{eq:theta-beta-crossderivative}
\end{align}
where we used the fact that $\partial_\beta |_{\beta=0} A_{\beta}[\bs^\beta(\btheta),\btheta] = \int_0^T c(\bs_t^0(\btheta)) \dt = C[\bs^0(\btheta)]$.

By Schwarz's theorem (symmetry of mixed partial derivatives), we have (since $\beta$ and $\btheta$ are independent, we can use $d$ instead of $\partial$):
\begin{align*}
    d_{\beta\btheta} A_{\beta}[\bs^\beta(\btheta),\btheta] = d_{\btheta\beta} A_{\beta}[\bs^\beta(\btheta),\btheta] \,.
\end{align*}
This requires the composite map $(\beta, \btheta) \mapsto A_\beta[\bs^\beta(\btheta), \btheta]$ to be $C^2$. This holds whenever the Lagrangian is $C^2$ in all its arguments and the trajectory map $(\beta, \btheta) \mapsto \bs^\beta(\btheta)$ is $C^2$, which follows from the standard smooth dependence of ODE solutions on parameters.

Equating the right-hand sides of equations~\eqref{eq:beta-theta-crossderivative} 
and~\eqref{eq:theta-beta-crossderivative}, we obtain:
\begin{align}
    d_{\btheta} C[\bs^0] 
    &= d_\beta \int_0^T \partial_{\btheta} L_{\beta}\left(\bs_t^{\beta}, \dot{\bs}_t^{\beta}, \btheta\right) \dt
    + \left(d_\beta(\delta_{\bs} A_{\beta} \delta_{\btheta} \bs^\beta) 
                  - d_{\btheta}(\delta_{\bs} A_{0} \delta_\beta \bs^\beta) \right) \,.
    \label{eq:ep_residual_before_cancellation}
\end{align}

From Proposition~\ref{prop:chain-rule-varvect} (derived from Lemma~\ref{lemma:euler-lagrange}), the variation through implicit dependence gives:
\begin{align}
    d_\beta(\delta_{\bs} A_{\beta} \delta_{\btheta} \bs^\beta) 
    &= d_\beta\left[\left(\partial_{\btheta} \bs_t^\beta\right)^\top \cdot 
       \partial_{\dot{\bs}} L_{\beta} \left(\bs_t^{\beta}, \dot{\bs}_t^{\beta}, \btheta\right) \right]_0^T \,.
    \label{eq:beta-residual}
\end{align}

Applying the product rule of differentiation:
\begin{align}
    &d_\beta(\delta_{\bs} A_{\beta} \delta_{\btheta} \bs^\beta) = \left[\left(\partial_{\beta\btheta} \bs_t^\beta\right)^\top \cdot 
      \partial_{\dot{\bs}} L_{0}\left(\bs_t^{0}, \dot{\bs}_t^{0}, \btheta\right) 
      + \left(\partial_{\btheta} \bs_t^0\right)^\top \cdot 
        d_\beta \partial_{\dot{\bs}} L_{\beta}\left(\bs_t^{\beta}, \dot{\bs}_t^{\beta}, \btheta\right)
    \right]_0^T \,.
    \label{eq:residual_la_beta}
\end{align}
By the same reasoning applied to~\eqref{eq:beta-residual} with the roles of $\beta$ and $\btheta$ exchanged:
\begin{align}
    &d_{\btheta}(\delta_{\bs} A_{0} \delta_\beta \bs^\beta) = \left[\left(\partial_{\btheta\beta} \bs_t^\beta\right)^\top \cdot 
      \partial_{\dot{\bs}} L_{0}\left(\bs_t^{0}, \dot{\bs}_t^{0}, \btheta\right) 
      + \left(d_{\btheta} \partial_{\dot{\bs}} L_{0}\left(\bs_t^{0}, \dot{\bs}_t^{0}, \btheta\right) \right)^\top 
        \cdot \left(\partial_\beta \bs_t^\beta\right) 
    \right]_0^T \,.
    \label{eq:residual_la_theta}
\end{align}
Using the symmetry of cross-derivatives, $\partial_{\btheta\beta} \bs_t^\beta = \partial_{\beta\btheta} \bs_t^\beta$, 
the first terms in equations~\eqref{eq:residual_la_beta} and~\eqref{eq:residual_la_theta} cancel:
\begin{align}
    &d_\beta(\delta_{\bs} A_{\beta} \delta_{\btheta} \bs^\beta) - d_{\btheta}(\delta_{\bs} A_{0} \delta_\beta \bs^\beta) = \left[ \left(\partial_{\btheta} \bs_t^{0}\right)^\top \cdot 
       d_\beta\partial_{\dot{\bs}} L_{\beta}\left(\bs_t^{\beta}, \dot{\bs}_t^{\beta}, \btheta\right)  - \left(d_{\btheta}\partial_{\dot{\bs}} L_{0}\left(\bs_t^{0}, \dot{\bs}_t^{0}, \btheta\right)\right)^\top 
              \cdot \partial_\beta \bs_t^{\beta}\right]_0^T \,.
    \label{eq:boundary_residuals_appendix}
\end{align}

Substituting equation~\eqref{eq:boundary_residuals_appendix} into 
equation~\eqref{eq:ep_residual_before_cancellation} yields the final result.

\end{proof} 
\section{Proof of LEP Corollaries}
\subsection{Proof of Corollary~\ref{corollary:civp_gradient} :Gradient estimator for CIVP}
\begin{proof}[Proof of Corollary~\ref{corollary:civp_gradient}]
We apply Theorem~\ref{thm:general-ep} to the CIVP formulation and analyze 
the boundary residual terms. From Theorem~\ref{thm:general-ep}, the boundary 
residual term is:
\begin{align*}
    \left[\left(\partial_{\boldsymbol{{\btheta}}} \bs_{\fwda,t}^0\right)^\top d_\beta \partial_{\dot{\bs}} L_\beta(\bs_{\fwda,t}^\beta, \dot{\bs}_{\fwda,t}^\beta, \boldsymbol{{\btheta}}) 
    - \left(d_{\boldsymbol{{\btheta}}} \partial_{\dot{\bs}} L_0(\bs_{\fwda,t}^0, \dot{\bs}_{\fwda,t}^0, \boldsymbol{{\btheta}})\right)^\top \partial_\beta \bs_{\fwda,t}^\beta\right]_0^T \,.
\end{align*}

We examine the boundary conditions at both temporal endpoints.

\textbf{Analysis at $t = 0$:} The boundary residual vanishes due to the 
constant initial value constraints. By the CIVP construction, all trajectories 
satisfy the boundary conditions $\bs_{\fwda,0}^\beta(\boldsymbol{{\btheta}}) = \balpha_0$ and 
$\dot{\bs}_{\fwda,0}^\beta(\boldsymbol{{\btheta}}) = \bgamma_0$, which are independent of both $\boldsymbol{{\btheta}}$ and $\beta$.

The left term vanishes because:
\begin{align*}
    \partial_{\boldsymbol{{\btheta}}} \bs_{\fwda,0}^0 = \partial_{\boldsymbol{{\btheta}}} \balpha_0 = \mathbf{0} \,.
\end{align*}

The right term vanishes because:
\begin{align*}
    \partial_\beta \bs_{\fwda,0}^\beta = \partial_\beta \balpha_0 = \mathbf{0} \,.
\end{align*}

Therefore, both boundary residual terms are zero at $t = 0$.

\textbf{Analysis at $t = T$:} The boundary residual does not cancel due to 
the absence of constraints at the final time. Unlike at the initial conditions, 
no boundary value constraints are imposed at $t = T$, so both 
$\partial_{\boldsymbol{{\btheta}}} \bs_{\fwda,T}^0$ and $\partial_\beta \bs_{\fwda,T}^\beta$ are generally non-zero.
Notably, since $\beta$ is scalar, the $\beta$ derivatives can easily be estimated via finite differences. To emphasize this, we can rewrite the left term as: 
\begin{align*}
    \left(\partial_{\boldsymbol{{\btheta}}} \bs_{\fwda,T}^0\right)^\top d_\beta \partial_{\dot{\bs}} L_\beta(\bs_{\fwda,T}^\beta, \dot{\bs}_{\fwda,T}^\beta, \boldsymbol{{\btheta}})
    &= \lim_{\beta \to 0} \frac{1}{\beta}\left(\partial_{\boldsymbol{{\btheta}}} \bs_{\fwda,T}^0\right)^\top 
    \left[\partial_{\dot{\bs}} L_\beta(\bs_{\fwda,T}^\beta, \dot{\bs}_{\fwda,T}^\beta, \boldsymbol{{\btheta}}) 
    - \partial_{\dot{\bs}} L_0(\bs_{\fwda,T}^0, \dot{\bs}_{\fwda,T}^0, \boldsymbol{{\btheta}})\right] \,.
\end{align*}
Similarly, the right term becomes:
\begin{align*}
    \left(d_{\boldsymbol{{\btheta}}} \partial_{\dot{\bs}} L_0(\bs_{\fwda,T}^0, \dot{\bs}_{\fwda,T}^0, \boldsymbol{{\btheta}})\right)^\top \partial_\beta \bs_{\fwda,T}^\beta
    &= \lim_{\beta \to 0} \frac{1}{\beta}\left(d_{\boldsymbol{{\btheta}}} \partial_{\dot{\bs}} L_0(\bs_{\fwda,T}^0, \dot{\bs}_{\fwda,T}^0, \boldsymbol{{\btheta}})\right)^\top 
    \left(\bs_{\fwda,T}^\beta - \bs_{\fwda,T}^0\right) \,.
\end{align*}

\textbf{Final result:} Combining the integral term (in finite difference form) from Theorem~\ref{thm:general-ep} 
with the boundary analysis and applying the finite difference approximation, we obtain:
\begin{align*}
    d_{\boldsymbol{{\btheta}}} C[\bs_{\fwda}^0(\boldsymbol{{\btheta}})] &= \lim_{\beta \to 0} \frac{1}{\beta} \Bigg[ 
    \int_0^T \Big[\partial_{\boldsymbol{{\btheta}}} L_{\beta}(\bs_{\fwda,t}^{\beta}, \dot{\bs}_{\fwda,t}^{\beta}, \btheta) 
    - \partial_{\boldsymbol{{\btheta}}} L_{0}(\bs_{\fwda,t}^{0}, \dot{\bs}_{\fwda,t}^0, \btheta)\Big] \dt \\
    &\quad + \left(\partial_{\boldsymbol{{\btheta}}} \bs_{\fwda,T}^{0}\right)^{\top}
    \Big(\partial_{\dot{\bs}} L_{\beta}(\bs_{\fwda,T}^{\beta}, \dot{\bs}_{\fwda,T}^{\beta}, \btheta) 
    - \partial_{\dot{\bs}} L_{0}(\bs_{\fwda,T}^{0}, \dot{\bs}_{\fwda,T}^{0}, \btheta)\Big) \\
    &\quad - \left(d_{\boldsymbol{{\btheta}}}\partial_{\dot{\bs}} L_{0}(\bs_{\fwda,T}^{0}, \dot{\bs}_{\fwda,T}^{0}, \btheta)\right)^{\top}  
    \left(\bs_{\fwda,T}^{\beta} - \bs_{\fwda,T}^{0}\right) \Bigg] \,.
\end{align*}
The boundary residuals at $t = T$ remain due to the absence of final time 
constraints.
\end{proof}

\subsection{Proof of Corollary~\ref{corollary:CBPVP_gradient}: Gradient estimator for CBPVP}

\begin{proof}[Proof of Corollary~\ref{corollary:CBPVP_gradient}]
We apply Theorem~\ref{thm:general-ep} to the CBPVP formulation and analyze 
the boundary residual terms. From Theorem~\ref{thm:general-ep}, the boundary 
residual term is:
\begin{align*}
    \left[\left(\partial_{\boldsymbol{{\btheta}}} \bs_{\bda,t}^0\right)^\top d_\beta \partial_{\dot{\bs}} L_\beta(\bs_{\bda,t}^\beta, \dot{\bs}_{\bda,t}^\beta, \boldsymbol{{\btheta}}) 
    - \left(d_{\boldsymbol{{\btheta}}} \partial_{\dot{\bs}} L_0(\bs_{\bda,t}^0, \dot{\bs}_{\bda,t}^0, \boldsymbol{{\btheta}})\right)^\top \partial_\beta \bs_{\bda,t}^\beta\right]_0^T \,.
\end{align*}

We examine the boundary conditions at both temporal endpoints.

\textbf{Analysis at $t = 0$:} The boundary residual vanishes due to the 
constant initial position constraint. By the CBPVP construction, all trajectories 
satisfy the boundary condition $\bs_{\bda,0}^\beta(\boldsymbol{{\btheta}}) = \balpha_0$, which is 
independent of both $\boldsymbol{{\btheta}}$ and $\beta$.

The left term vanishes because:
\begin{align*}
    \partial_{\boldsymbol{{\btheta}}} \bs_{\bda,0}^0 = \partial_{\boldsymbol{{\btheta}}} \balpha_0 = \mathbf{0} \,.
\end{align*}

The right term vanishes because:
\begin{align*}
    \partial_\beta \bs_{\bda,0}^\beta = \partial_\beta \balpha_0 = \mathbf{0} \,.
\end{align*}

Therefore, both boundary residual terms are zero at $t = 0$.

\textbf{Analysis at $t = T$:} The boundary residual also vanishes due to the 
constant final position constraint. By the CBPVP construction, all trajectories 
satisfy the boundary condition $\bs_{\bda,T}^\beta(\boldsymbol{{\btheta}}) = \balpha_T$, which is 
independent of both $\boldsymbol{{\btheta}}$ and $\beta$.

The left term vanishes because:
\begin{align*}
    \partial_{\boldsymbol{{\btheta}}} \bs_{\bda,T}^0 = \partial_{\boldsymbol{{\btheta}}} \balpha_T = \mathbf{0} \,.
\end{align*}

The right term vanishes because:
\begin{align*}
    \partial_\beta \bs_{\bda,T}^\beta = \partial_\beta \balpha_T = \mathbf{0} \,.
\end{align*}

Therefore, both boundary residual terms are zero at $t = T$.

\textbf{Final result:} Since the boundary residuals vanish at both endpoints, 
combining with the integral term from Theorem~\ref{thm:general-ep} and 
applying the finite difference approximation, we obtain:
\begin{align*}
    d_{\boldsymbol{{\btheta}}} C[\bs_{\bda}^0(\boldsymbol{{\btheta}})] &= \lim_{\beta \to 0} \frac{1}{\beta} 
    \int_0^T \Big[\partial_{\boldsymbol{{\btheta}}} L_{\beta}(\bs_{\bda,t}^{\beta}, \dot{\bs}_{\bda,t}^{\beta}, \btheta) 
    - \partial_{\boldsymbol{{\btheta}}} L_{0}(\bs_{\bda,t}^{0}, \dot{\bs}_{\bda,t}^0, \btheta)\Big] \dt \,.
\end{align*}

The CBPVP formulation eliminates all problematic boundary residual terms, 
yielding a clean gradient estimator that only requires integrating differences 
between Lagrangian derivatives over the two trajectories.
\end{proof}

\section{Proof of Theorem~\ref{thm:rhel}: Gradient estimator from RHEL with parametrized initial state}
\label{app:proof-rhel-parametrized}

This section shows how Theorem~\ref{thm:rhel} can be recovered from the results in the RHEL paper~\citep{rhel}. We proceed in two steps: first clarifying the notation differences between the two papers, then deriving how the gradient with respect to a parametrized initial state combines with the gradient with respect to parameters.

\subsection{Notation Correspondence}

The RHEL paper~\citep{rhel} uses a time convention where the forward pass runs from $t \in [-T, 0]$ and the echo phase runs from $t \in [0, T]$. In contrast, this paper uses $t \in [0, T]$ for the forward pass. The correspondences are:

\paragraph{Time indexing.} For the forward trajectory:
\begin{itemize}
    \item RHEL paper: forward trajectory $\bPhi(t)$ for $t \in [-T, 0]$
    \item This paper: forward trajectory $\bPhi_t$ for $t \in [0, T]$ with inputs $\bx_t$
    \item Relationship: $\bPhi_t \leftrightarrow \bPhi(-t+T)$ in RHEL notation
\end{itemize}

For the echo trajectory:
\begin{itemize}
    \item RHEL paper: echo trajectory $\bPhi^e(t, \epsilon)$ for $t \in [0, T]$ with inputs $\bu(-t)$, where $\epsilon$ is the nudging strength
    \item This paper: echo trajectory $\bPhi^e_t$ for $t \in [0, T]$ with inputs $\bx_{T-t}$. The dependence on the nudging strength $\beta$ is left implicit in the subscript notation $\bPhi^e_t \equiv \bPhi^e_t(\beta)$, since $\beta$ is fixed throughout the forward and echo phases and only appears explicitly in the limit $\beta \to 0$
    \item Relationship: inputs are time-reversed relative to forward pass
\end{itemize}

\paragraph{State variables.} The phase space variables are:
\begin{itemize}
    \item RHEL paper: $\bPhi = \begin{pmatrix} \bphi \\ \bpi \end{pmatrix}$ where $\bphi$ represents positions and $\bpi$ represents conjugate momenta
    \item This paper: $\bPhi = \begin{pmatrix} \bs \\ \bp \end{pmatrix}$ where $\bs$ represents positions and $\bp$ represents conjugate momenta, with $\bs$ also denoted as $\balpha$ and $\bp$ as $\bgamma^p$ for initial conditions
    \item Correspondence for forward phase: $\bPhi_t = \begin{pmatrix} \bs_t \\ \bp_t \end{pmatrix}$ (this paper) $\leftrightarrow$ $\bPhi(t) = \begin{pmatrix} \bphi_t \\ \bpi_t \end{pmatrix}$ (RHEL paper)
    \item Correspondence for echo phase: $\bPhi^e_t = \begin{pmatrix} \bs^e_t \\ \bp^e_t \end{pmatrix}$ (this paper) $\leftrightarrow$ $\bPhi^e(t) = \begin{pmatrix} \bphi^e_t \\ \bpi^e_t \end{pmatrix}$ (RHEL paper)
\end{itemize}

\paragraph{Nudging parameter.} The RHEL paper uses $\epsilon$ for bidirectional perturbations $\pm\epsilon$ while this paper uses $\beta$ for unidirectional perturbation; both converge to the same gradient as $\epsilon, \beta \to 0$.

\subsection{Gradient Decomposition for Parametrized Initial States}

We now show how to recover Theorem~\ref{thm:rhel} from the original RHEL result (Theorem 3.1 in~\citep{rhel}). We proceed by first recalling the RHEL result for fixed initial conditions, then the gradient with respect to initial state, and finally combining them.

\paragraph{Step 1: Gradient with respect to parameters (fixed initial state).}

When the initial state $\bPhi_0 = \begin{pmatrix} \balpha_0 \\ \bmu_0 \end{pmatrix}$ is held fixed (independent of $\btheta$), Theorem 3.1 in~\citep{rhel} gives:\footnote{Note that the RHEL paper uses bidirectional nudging with perturbations $\pm\epsilon$ for improved numerical accuracy, while we present the unidirectional formulation here for simplicity. Both converge to the same gradient in the continuous-time limit; see the note at the end of this section for details.}
\begin{align*}
\partial_{\btheta} C[\bPhi(\btheta, \bPhi_0)] = \lim_{\beta \to 0} \frac{1}{\beta} \left[-\int_0^T \left(\partial_{\btheta} H(\bPhi^e_t, \btheta, \bx_{T-t}) - \partial_{\btheta} H(\bPhi_t, \btheta, \bx_t)\right) dt\right] \,,
\end{align*}
where $\bPhi_t$ is the forward trajectory and $\bPhi^e_t$ is the echo trajectory with nudging strength $\beta$.

\paragraph{Step 2: Gradient with respect to initial state (fixed parameters).}

The RHEL paper also provides the gradient of the cost with respect to the initial state (holding parameters $\btheta$ fixed). This sensitivity can be expressed through the echo trajectory difference at the final time:
\begin{align*}
\partial_{\bPhi_0} C = \lim_{\beta \to 0} \frac{1}{\beta} \bSigma_x \left(\begin{pmatrix} \bs^e_T \\ \bp^e_T \end{pmatrix} - \begin{pmatrix} \balpha_0 \\ -\bmu_0 \end{pmatrix}\right) \,,
\end{align*}
where $\begin{pmatrix} \bs^e_T \\ \bp^e_T \end{pmatrix} = \bPhi^e_T$ is the echo trajectory at final time $T$, and the sign flip in the momentum component arises from the momentum-flipping boundary condition of the echo phase.

\paragraph{Step 3: Combining both contributions via chain rule.}

When the initial state depends on parameters, $\bPhi_0(\btheta) = \begin{pmatrix} \balpha_0(\btheta) \\ \bmu_0(\btheta) \end{pmatrix}$, the total gradient must account for both the direct dependence on $\btheta$ and the indirect dependence through $\bPhi_0(\btheta)$. By the chain rule:
\begin{align*}
\dtheta C[\bPhi(\btheta, \bPhi_0(\btheta))] = \underbrace{\partial_{\btheta} C[\bPhi(\btheta, \bPhi_0)]}_{\text{direct, holding $\bPhi_0$ fixed}} + \underbrace{\left(\partial_{\btheta} \bPhi_0\right)^\top \partial_{\bPhi_0} C}_{\text{indirect, through $\bPhi_0(\btheta)$}} \,.
\end{align*}

Substituting the expressions from Steps 1 and 2:
\begin{align*}
\dtheta C = \lim_{\beta \to 0} \frac{1}{\beta} \Bigg[&-\int_0^T \left(\partial_{\btheta} H(\bPhi^e_t, \btheta, \bx_{T-t}) - \partial_{\btheta} H(\bPhi_t, \btheta, \bx_t)\right) dt \\
&+ \left(\partial_{\btheta} \begin{pmatrix} \balpha_0 \\ \bmu_0 \end{pmatrix}\right)^\top \bSigma_x \left(\begin{pmatrix} \bs^e_T \\ \bp^e_T \end{pmatrix} - \begin{pmatrix} \balpha_0 \\ -\bmu_0 \end{pmatrix}\right)\Bigg] \,,
\end{align*}
which is precisely the gradient estimator $\Delta^{\text{RHEL}}(\beta, \balpha_0(\btheta), \bmu_0(\btheta))$ in Eq.~\eqref{eq:Delta^RHEL} of Theorem~\ref{thm:rhel}.

\subsection{Note on bidirectional vs.~unidirectional nudging}
\label{app:bidirectional-nudging}

The RHEL paper uses bidirectional nudging with perturbations $\pm\beta$, computing:
\begin{equation*}
\Delta^{\text{RHEL}}(\beta) = \frac{1}{2\beta}\left[\left(\partial_{\btheta} H(\bPhi^e_t(+\beta), \btheta, \bx_{T-t}) - \partial_{\btheta} H(\bPhi_t, \btheta, \bx_t)\right) - \left(\partial_{\btheta} H(\bPhi^e_t(-\beta), \btheta, \bx_{T-t}) - \partial_{\btheta} H(\bPhi_t, \btheta, \bx_t)\right)\right] \,,
\end{equation*}
which is a centered finite-difference approximation. In this paper, we present the unidirectional formulation:
\begin{equation*}
\Delta^{\text{RHEL}}(\beta) = \frac{1}{\beta}\left(\partial_{\btheta} H(\bPhi^e_t(\beta), \btheta, \bx_{T-t}) - \partial_{\btheta} H(\bPhi_t, \btheta, \bx_t)\right) \,,
\end{equation*}
which is a forward finite-difference approximation. Both converge to the same gradient in the limit $\beta \to 0$. The bidirectional version has better numerical accuracy (second-order error $O(\beta^2)$ vs.~first-order $O(\beta)$), a well-known trick in equilibrium propagation~\citep{laborieuxScalingEquilibriumPropagation2021}.

\section{Proof of Proposition~\ref{prop:solution-pfvp-reversibility}: The bouncing-back trick}
\begin{proof}[Proof of Proposition~\ref{prop:solution-pfvp-reversibility}]
Define the time-reversed trajectory:
\begin{align}
    \bs_{rev,t}^\beta := \bs_{\bwda,T-t}^\beta(\btheta,(\balpha_T, \bgamma_T)) \,.
    \label{eq:construction_srev}
\end{align}
By the chain rule ($\frac{d(T-t)}{dt} = -1$), its velocity and acceleration satisfy:
\begin{align}
    \dot{\bs}_{rev,t}^\beta = -\dot{\bs}_{\bwda,T-t}^\beta, \qquad \ddot{\bs}_{rev,t}^\beta = \ddot{\bs}_{\bwda,T-t}^\beta \,.
    \label{eq:rev_chain_rule}
\end{align}

\textbf{Step 1: $\bs_{rev}$ satisfies the Euler-Lagrange equation.} We show that $\EL_r(t, \btheta, \beta) = 0$ along $\bs_{rev}$ by relating each term back to the Euler-Lagrange equation satisfied by $\bs_{\bwda}^\beta$. Let $t' := T-t$.

\emph{Momentum term.} By \eqref{eq:rev_chain_rule} and the odd-derivative property (Lemma~\ref{lemma:odd_derivative}):
\begin{align*}
    \partial_{\dot{\bs}} L_{\beta}(\bs_{rev,t}^\beta, \dot{\bs}_{rev,t}^\beta, \btheta)
    &= \partial_{\dot{\bs}} L_{\beta}(\bs_{\bwda,t'}^\beta, -\dot{\bs}_{\bwda,t'}^\beta, \btheta) 
    = -\partial_{\dot{\bs}} L_{\beta}(\bs_{\bwda,t'}^\beta, \dot{\bs}_{\bwda,t'}^\beta, \btheta) \,.
\end{align*}
Taking the total time derivative and using $d_t = -d_{t'}$:
\begin{align*}
    d_t\partial_{\dot{\bs}} L_{\beta}(\bs_{rev,t}^\beta, \dot{\bs}_{rev,t}^\beta, \btheta) 
    &= d_t\!\left[-\partial_{\dot{\bs}} L_{\beta}(\bs_{\bwda,t'}^\beta, \dot{\bs}_{\bwda,t'}^\beta, \btheta)\right]
    = (-d_{t'})\!\left[-\partial_{\dot{\bs}} L_{\beta}(\bs_{\bwda,t'}^\beta, \dot{\bs}_{\bwda,t'}^\beta, \btheta)\right]
    = d_{t'}\partial_{\dot{\bs}} L_{\beta}(\bs_{\bwda,t'}^\beta, \dot{\bs}_{\bwda,t'}^\beta, \btheta) \,.
\end{align*}

\emph{Position term.} Since $L_\beta$ is even in $\dot{\bs}$, so is $\partial_{\bs} L_\beta$ (differentiating $L_\beta(\bs, \dot{\bs}, \btheta) = L_\beta(\bs, -\dot{\bs}, \btheta)$ with respect to $\bs$):
\begin{align*}
    \partial_{\bs} L_{\beta}(\bs_{rev,t}^\beta, \dot{\bs}_{rev,t}^\beta, \btheta)
    &= \partial_{\bs} L_{\beta}(\bs_{\bwda,t'}^\beta, -\dot{\bs}_{\bwda,t'}^\beta, \btheta)
    = \partial_{\bs} L_{\beta}(\bs_{\bwda,t'}^\beta, \dot{\bs}_{\bwda,t'}^\beta, \btheta) \,.
\end{align*}

\emph{Combining:}
\begin{align*}
    \partial_{\bs} L_{\beta}(\bs_{rev,t}^\beta, \dot{\bs}_{rev,t}^\beta,\btheta) 
    - d_t\partial_{\dot{\bs}} L_{\beta}(\bs_{rev,t}^\beta, \dot{\bs}_{rev,t}^\beta, \btheta)
    &= \partial_{\bs} L_{\beta}(\bs_{\bwda,t'}^\beta, \dot{\bs}_{\bwda,t'}^\beta,\btheta)
    - d_{t'}\partial_{\dot{\bs}} L_{\beta}(\bs_{\bwda,t'}^\beta, \dot{\bs}_{\bwda,t'}^\beta, \btheta) \\
    &= 0 \,. \quad \text{($\bs_{\bwda}^\beta$ satisfies Euler-Lagrange at $t' = T-t$)}
\end{align*}
Note that $\EL_r(t', \btheta, \beta)$ is evaluated with input $\bx_{t'}$ and target $\by_{t'}$ at time $t' = T-t$, so the nudged dynamics in the IVP automatically use the time-reversed input and target sequences.

\textbf{Step 2: Initial conditions of $\bs_{rev}$.} Using \eqref{eq:construction_srev} and \eqref{eq:rev_chain_rule} with the PFVP boundary conditions $\bs_{\bwda,T}^\beta = \balpha_T$ and $\dot{\bs}_{\bwda,T}^\beta = \bgamma_T$:
\begin{align*}
    \bs_{rev,0}^\beta &= \bs_{\bwda,T}^\beta = \balpha_T, \qquad
    \dot{\bs}_{rev,0}^\beta = -\dot{\bs}_{\bwda,T}^\beta = -\bgamma_T \,.
\end{align*}

\textbf{Step 3: Uniqueness.} Since $\bs_{rev}^\beta$ and $\bs_{\fwda,t}^\beta\left(\btheta, \left(\balpha_T, -\bgamma_T\right)\right)$ both satisfy the same Euler-Lagrange equation with the same initial conditions $(\balpha_T, -\bgamma_T)$ at $t=0$, they are identical by uniqueness of the initial value problem (Remark~\ref{rmk:regularity}):
\begin{align*}
     \bs_{\fwda,t}^\beta\left(\btheta, \left(\balpha_T, -\bgamma_T\right)\right) &= \bs_{rev,t}^\beta = \bs_{\bwda,T-t}^\beta(\btheta,(\balpha_T, \bgamma_T))  \quad \text{(by \eqref{eq:construction_srev})} \,.
\end{align*}
A time translation $t' \leftarrow T-t$ gives the desired result. 
\end{proof}

\section{Proof Theorem~\ref{thm:pfvp-cancellation}: PFVP cancels the boundary residuals}
\begin{proof} [Proof of Theorem~\ref{thm:pfvp-cancellation}]
Let's analyze the boundary residual term from Theorem~\ref{thm:general-ep} 
for the PFVP trajectories $t \mapsto \bs_{\bwda,t}^\beta(\btheta, (\balpha_T(\btheta), \bgamma_T(\btheta)))$:

\begin{align*}
    \left[\left(\partial_{\btheta} \bs_{\bwda,t}^{0} \right)^\top  
    \cdot d_\beta\partial_{\dot{\bs}} L_{\beta}\left(\bs_{\bwda,t}^{\beta}, 
    \dot{\bs}_{\bwda,t}^{\beta}, \btheta\right) -
    \left(d_{\btheta}\partial_{\dot{\bs}} L_{0}\left(\bs_{\bwda,t}^{0}, 
    \dot{\bs}_{\bwda,t}^{0}, \btheta\right)\right)^\top \cdot \partial_\beta \bs_{\bwda,t}^{\beta}\right]_0^T \,.
\end{align*}
We examine the boundary conditions at both temporal endpoints.

\paragraph{Analysis at $t=T$:} 
The boundary residual vanishes due to the parametric final value constraint. 

The right term disappears because $\partial_\beta \bs_{\bwda,T}^{\beta}=0$. 
By the PFVP construction, the nudged trajectory satisfies the boundary condition 
$\bs_{\bwda,T}^\beta(\btheta, (\balpha_T(\btheta), \bgamma_T(\btheta))) = \balpha_T(\btheta)$, 
which is independent of $\beta$.
The left term cancels because:  
\begin{align*}
    d_\beta\partial_{\dot{\bs}} L_{\beta}\left(\bs_{\bwda,T}^{\beta}, 
    \dot{\bs}_{\bwda,T}^{\beta}, \btheta\right) 
    &= d_\beta \partial_{\dot{\bs}} L_{0}\left(\bs_{\bwda,T}^{\beta}, 
    \dot{\bs}_{\bwda,T}^{\beta}, \btheta\right) + \partial^2_{\beta, \dot{\bs}} 
    L_{\beta}\left(\bs_{\bwda,T}^0, \dot{\bs}_{\bwda,T}^0, \btheta\right) \\
    &= \partial^2_{\beta, \dot{\bs}} L_{\beta}\left(\bs_{\bwda,T}^0, 
    \dot{\bs}_{\bwda,T}^0, \btheta\right) \quad 
    \text{(both $\bs_{\bwda,T}^\beta = \balpha_T(\btheta)$ and $\dot{\bs}_{\bwda,T}^\beta = \bgamma_T(\btheta)$ are $\beta$-independent)} \\
    &= \partial_{\dot{\bs}} c\left(\bs_{\bwda,T}^0, \dot{\bs}_{\bwda,T}^0\right) \\
    &= 0 \,. \qquad \text{(cost function $c$ depends only on position, not velocity)}
\end{align*}

\paragraph{Analysis at $t=0$:} 
The boundary residual reduces to easy-to-compute terms at $t=0$.
\begin{align*}
    \bs_{\bwda,0}^{0} &= \bs_{\bwda,0}^{0}(\btheta, (\balpha_T(\btheta), \bgamma_T(\btheta))) 
    \quad \text{(Definition~\ref{def:PFVP-general})}\\
    &= \bs_0^{0}(\btheta, (\balpha_0(\btheta), \bgamma_0(\btheta)))
    \quad \text{(Proposition~\ref{prop:equivalence_IVP_EVP} evaluated at $t=0$)}\\
    &= \balpha_0(\btheta) \,.
\end{align*}

Similarly we have $\dot{\bs}_{\bwda,0}^{0} = \bgamma_0(\btheta)$.

Since $\bs_{\bwda,0}^{0} = \balpha_0(\btheta)$ and $\dot{\bs}_{\bwda,0}^{0} = \bgamma_0(\btheta)$, 
the right term simplifies to:
\begin{align*}
    d_{\btheta}\partial_{\dot{\bs}} L_{\beta}\left(\bs_{\bwda,0}^{0}, 
    \dot{\bs}_{\bwda,0}^{0}, \btheta\right)
    &= d_{\btheta}\partial_{\dot{\bs}} L_{\beta}\left(\balpha_0(\btheta), 
    \bgamma_0(\btheta), \btheta\right) \,.
\end{align*}

\paragraph{Final result:} 
All terms cancel at $t=T$. At $t=0$, the boundary residual evaluates to:

\begin{align*}
    &\left[\left(\partial_{\btheta} \bs_{\bwda,t}^{0} \right)^\top 
    d_\beta\partial_{\dot{\bs}} L_{\beta}\left(\bs_{\bwda,t}^{\beta}, 
    \dot{\bs}_{\bwda,t}^{\beta}, \btheta\right) -
    \left(d_{\btheta}\partial_{\dot{\bs}} L_{0}\left(\bs_{\bwda,t}^{0}, 
    \dot{\bs}_{\bwda,t}^{0}, \btheta\right)\right)^\top \cdot \partial_\beta \bs_{\bwda,t}^{\beta}\right]_0^T \\
    &= \left(\partial_{\btheta} \balpha_0(\btheta)\right)^\top 
    d_\beta\partial_{\dot{\bs}} L_{\beta}\left(\bs_{\bwda,0}^{\beta}, 
    \dot{\bs}_{\bwda,0}^{\beta}, \btheta\right) -
    \left(d_{\btheta}\partial_{\dot{\bs}} L_{0}\left(\balpha_0(\btheta), 
    \bgamma_0(\btheta), \btheta\right)\right)^\top \cdot \partial_\beta \bs_{\bwda,0}^{\beta} \,.
\end{align*}

This yields the desired result.
\end{proof}

\section{Proof Theorem~\ref{thm:lep-rhel-equivalence}: Equivalence between Lagrangian EP
and Recurrent Hamiltonian Echo Learning }

\paragraph{Proof roadmap.}
The proof proceeds in three stages.
\begin{enumerate}[leftmargin=*]
    \item \textbf{Lagrangian--Hamiltonian correspondence (Section~E.1).} 
    We recall the classical result that the Legendre transform maps Euler--Lagrange trajectories bijectively to Hamilton trajectories (Theorem~\ref{thm:legendre_transform_on_dyna}).
    \item \textbf{PFVP $\leftrightarrow$ HES trajectory correspondence (Section~E.2).}
    Using Theorem~\ref{thm:legendre_transform_on_dyna} together with the PFVP reversibility (Proposition~\ref{prop:solution-pfvp-reversibility}), we construct bijective maps between the PFVP free/nudged phases and the HES forward/echo phases.
    \item \textbf{Gradient equivalence (Section~E.3).}
    We transform the PFVP gradient estimator term by term---first the integral term, then the boundary term ---to obtain the RHEL estimator.
\end{enumerate}

\begin{center}
\begin{tikzpicture}[
    node distance=0.6cm and 0.8cm,
    box/.style={draw, rounded corners, text width=5.4cm, align=center, 
                minimum height=1cm, inner sep=6pt, font=\small},
    wide/.style={box, text width=9cm},
    arr/.style={-{Stealth[length=5pt]}, thick}
]
\node[box] (leg) {Section~E.1\\[2pt] Lagrangian $\leftrightarrow$ Hamiltonian correspondence};
\node[box, below left=0.9cm and -0.3cm of leg] (free) {Section~E.2 (free phase)\\[2pt] PFVP $\to$ IVP $\to$ HES forward};
\node[box, below right=0.9cm and -0.3cm of leg] (nudge) {Section~E.2 (nudged phase)\\[2pt] PFVP $\to$ IVP $\to$ HES echo};
\node[wide, below=2.4cm of leg, fill=gray!10] (equiv) {End of Section~E.2\\[2pt] PFVP $\leftrightarrow$ HES equivalence (forward $+$ echo)};
\node[wide, below=0.6cm of equiv] (grad) {Section~E.3: Gradient equivalence\\[2pt] Integral term  $+$ Boundary term };
\node[box, below=0.6cm of grad, text width=6cm, fill=gray!10] (result) {$\Delta^{\text{PFVP}} = \Delta^{\text{RHEL}}$};

\draw[arr] (leg) -- (free);
\draw[arr] (leg) -- (nudge);
\draw[arr] (free) -- (equiv);
\draw[arr] (nudge) -- (equiv);
\draw[arr] (equiv) -- (grad);
\draw[arr] (grad) -- (result);
\end{tikzpicture}
\end{center}

\subsection{Relating the solutions of Euler-Lagrange and Hamilton equations}
Here we first recall a classic theorem about the Legendre transform and how it's used in physics to relate solutions of the Euler-Lagrange and Hamilton's equation. 
\begin{theorem}[Equivalence of Lagrangian and Hamiltonian dynamics]
\label{thm:legendre_transform_on_dyna}
Assume the Legendre transform of Proposition~\ref{prop:Legendre_transform} 
is well defined along the trajectories considered, i.e., the Hessian condition 
$\det(\partial^2_{\dot{\bs},\dot{\bs}} L(\bs_t,\dot{\bs}_t)) \neq 0$ holds at 
each point along the trajectory.

Then the Legendre transform maps solutions of the Euler--Lagrange equations 
bijectively to solutions of Hamilton's equations, together with their 
 initial conditions.

\begin{enumerate}
    \item \textbf{Correspondence of initial conditions.}
    For every Lagrangian initial condition $(\bs_0,\dot{\bs}_0)$ there exists a unique 
    Hamiltonian initial condition
    \begin{equation*}
        \bp_0 = \partial_{\dot{\bs}} L(\bs_0,\dot{\bs}_0)\,,
    \end{equation*}
    and for every Hamiltonian initial condition $(\bs_0,\bp_0)$ there exists a unique 
    Lagrangian initial condition
    \begin{equation*}
        \dot{\bs}_0 = \partial_{\bp} H(\bs_0,\bp_0)\,.
    \end{equation*}
    Thus the Legendre map induces a bijection between  initial conditions.

    \item \textbf{Correspondence of solutions.}
    Let $\bp_t = \partial_{\dot{\bs}} L(\bs_t,\dot{\bs}_t)$. Then:
    \begin{itemize}
        \item If the trajectory $t \mapsto \bs_t$ satisfies the Euler--Lagrange equations
        \begin{equation*}
            \mathrm{d}_t \bigl(\partial_{\dot{\bs}} L(\bs_t,\dot{\bs}_t)\bigr)
            - \partial_{\bs} L(\bs_t,\dot{\bs}_t) = 0\,,
        \end{equation*}
        then the pair $(\bs_t,\bp_t)$ satisfies Hamilton's equations
        \begin{equation*}
            \dot{\bs}_t = \partial_{\bp} H(\bs_t,\bp_t),
            \qquad
            \dot{\bp}_t = -\partial_{\bs} H(\bs_t,\bp_t)\,.
        \end{equation*}

        \item Conversely, if $(\bs_t,\bp_t)$ satisfies Hamilton's equations, then  
        $\bs_t$ satisfies the Euler--Lagrange equations, with
        \begin{equation*}
            \dot{\bs}_t = \partial_{\bp} H(\bs_t,\bp_t),
            \qquad
            \bp_t = \partial_{\dot{\bs}} L(\bs_t,\dot{\bs}_t)\,.
        \end{equation*}
    \end{itemize}
\end{enumerate}

Consequently, under a well-defined Legendre transform, there is a one-to-one 
correspondence between Lagrangian trajectories $\bs_t$ and Hamiltonian 
trajectories $(\bs_t,\bp_t)$, together with their initial conditions.
\end{theorem}

\begin{proof}
By assumption, the Hessian condition $\det(\partial^2_{\dot{\bs},\dot{\bs}} L) \neq 0$ holds along the trajectories considered, so the Legendre transform of Proposition~\ref{prop:Legendre_transform}
gives a smooth locally invertible map
\begin{equation*}
    (\bs_t,\dot{\bs}_t) \longleftrightarrow (\bs_t,\bp_t)
\end{equation*}
at each time $t$, with
\begin{equation*}
    \bp_t = \partial_{\dot{\bs}} L(\bs_t,\dot{\bs}_t),
    \qquad
    H(\bs_t,\bp_t) = \bp_t^\top \dot{\bs}_t - L(\bs_t,\dot{\bs}_t)\,,
\end{equation*}
and inverse relations
\begin{equation*}
    \dot{\bs}_t = \partial_{\bp} H(\bs_t,\bp_t),
    \qquad
    L(\bs_t,\dot{\bs}_t) = \bp_t^\top \dot{\bs}_t - H(\bs_t,\bp_t)\,.
\end{equation*}

We first show that this induces a bijection between  initial conditions,
then prove the equivalence of the equations of motion.

\medskip
\noindent\textbf{1. Correspondence of initial conditions.}
Since for each fixed $\bs$ the map
\begin{equation*}
    \dot{\bs} \mapsto \bp = \partial_{\dot{\bs}} L(\bs,\dot{\bs})
\end{equation*}
is locally invertible (by the non-degenerate Hessian condition of Proposition~\ref{prop:Legendre_transform}),
it follows in particular that at time $t=0$ the map
\begin{equation*}
    (\bs_0,\dot{\bs}_0) \longleftrightarrow (\bs_0,\bp_0),
    \qquad
    \bp_0 = \partial_{\dot{\bs}} L(\bs_0,\dot{\bs}_0)\,,
\end{equation*}
is one-to-one and onto, with inverse
\begin{equation*}
    \dot{\bs}_0 = \partial_{\bp} H(\bs_0,\bp_0)\,.
\end{equation*}
This proves the stated bijection between Lagrangian and Hamiltonian initial
conditions.

\medskip
\noindent\textbf{2. Two basic identities for the Legendre transform.}
We now derive two standard identities that hold whenever $H$ is the 
Legendre transform of $L$:

\begin{equation*}
    \partial_{\bp} H(\bs,\bp) = \dot{\bs},
    \qquad
    \partial_{\bs} H(\bs,\bp) = -\partial_{\bs} L(\bs,\dot{\bs})\,,
\end{equation*}
where $\dot{\bs}$ is implicitly defined by $\bp = \partial_{\dot{\bs}} L(\bs,\dot{\bs})$.

\medskip
\noindent\emph{Identity $\partial_{\bp} H = \dot{\bs}$.}
By definition of $H$,
\begin{equation*}
    H(\bs,\bp) = \bp^\top \dot{\bs}(\bs,\bp) - L\bigl(\bs,\dot{\bs}(\bs,\bp)\bigr)\,,
\end{equation*}
where we view $\dot{\bs}$ as a function of $(\bs,\bp)$ defined implicitly by
\begin{equation*}
    \bp = \partial_{\dot{\bs}} L\bigl(\bs,\dot{\bs}(\bs,\bp)\bigr)\,.
\end{equation*}
Differentiating $H$ with respect to $\bp$ at fixed $\bs$ gives
\begin{equation*}
    \partial_{\bp} H
    = \dot{\bs} + (\partial_{\bp} \dot{\bs})^\top \bp
      - (\partial_{\bp} \dot{\bs})^\top \partial_{\dot{\bs}} L(\bs,\dot{\bs})\,.
\end{equation*}
Since $\bp = \partial_{\dot{\bs}} L(\bs,\dot{\bs})$, the last two terms cancel, and we obtain
\begin{equation*}
    \partial_{\bp} H(\bs,\bp) = \dot{\bs}\,.
\end{equation*}

\medskip
\noindent\emph{Identity $\partial_{\bs} H = -\partial_{\bs} L$.}
Again, from
\begin{equation*}
    H(\bs,\bp) = \bp^\top \dot{\bs}(\bs,\bp) - L\bigl(\bs,\dot{\bs}(\bs,\bp)\bigr)\,,
\end{equation*}
differentiate with respect to $\bs$ at fixed $\bp$:
\begin{equation*}
    \partial_{\bs} H
    = \bp^\top \partial_{\bs} \dot{\bs}
      - \partial_{\bs} L(\bs,\dot{\bs})
      - (\partial_{\bs} \dot{\bs})^\top \partial_{\dot{\bs}} L(\bs,\dot{\bs})\,.
\end{equation*}
Using $\bp = \partial_{\dot{\bs}} L(\bs,\dot{\bs})$, the first and third terms cancel, so
\begin{equation*}
    \partial_{\bs} H(\bs,\bp) = -\partial_{\bs} L(\bs,\dot{\bs})\,.
\end{equation*}

\medskip
\noindent\textbf{3. Euler--Lagrange $\Rightarrow$ Hamilton.}
Assume the trajectory $t \mapsto \bs_t$ satisfies the Euler--Lagrange equations
\begin{equation*}
    \mathrm{d}_t\bigl(\partial_{\dot{\bs}} L(\bs_t,\dot{\bs}_t)\bigr)
    - \partial_{\bs} L(\bs_t,\dot{\bs}_t) = 0\,.
\end{equation*}
Define the momentum
\begin{equation*}
    \bp_t = \partial_{\dot{\bs}} L(\bs_t,\dot{\bs}_t)\,.
\end{equation*}
We must show that $(\bs_t,\bp_t)$ satisfies Hamilton's equations
\begin{equation*}
    \dot{\bs}_t = \partial_{\bp} H(\bs_t,\bp_t),
    \qquad
    \dot{\bp}_t = -\partial_{\bs} H(\bs_t,\bp_t)\,.
\end{equation*}

The first Hamilton equation follows immediately from the identity 
$\partial_{\bp} H = \dot{\bs}$:
\begin{equation*}
    \dot{\bs}_t = \partial_{\bp} H(\bs_t,\bp_t)\,.
\end{equation*}

For the second equation, note that the Euler--Lagrange equation implies
\begin{equation*}
    \dot{\bp}_t = \mathrm{d}_t\bigl(\partial_{\dot{\bs}} L(\bs_t,\dot{\bs}_t)\bigr)
    = \partial_{\bs} L(\bs_t,\dot{\bs}_t)\,.
\end{equation*}
Using the identity $\partial_{\bs} H = -\partial_{\bs} L$ evaluated along the trajectory,
we obtain
\begin{equation*}
    \dot{\bp}_t = \partial_{\bs} L(\bs_t,\dot{\bs}_t)
    = -\partial_{\bs} H(\bs_t,\bp_t)\,,
\end{equation*}
which is exactly the second Hamilton equation.

\medskip
\noindent\textbf{4. Hamilton $\Rightarrow$ Euler--Lagrange.}
Conversely, assume $(\bs_t,\bp_t)$ satisfies Hamilton's equations
\begin{equation*}
    \dot{\bs}_t = \partial_{\bp} H(\bs_t,\bp_t),
    \qquad
    \dot{\bp}_t = -\partial_{\bs} H(\bs_t,\bp_t)\,,
\end{equation*}
and that $L$ and $H$ are related by the Legendre transform as above.

Define the velocity via the inverse Legendre relation
\begin{equation*}
    \dot{\bs}_t = \partial_{\bp} H(\bs_t,\bp_t)\,,
\end{equation*}
and define
\begin{equation*}
    \bp_t = \partial_{\dot{\bs}} L(\bs_t,\dot{\bs}_t)\,,
\end{equation*}
which is consistent by assumption (the Legendre map is a bijection).

We want to show that $\bs_t$ satisfies the Euler--Lagrange equation
\begin{equation*}
    \mathrm{d}_t\bigl(\partial_{\dot{\bs}} L(\bs_t,\dot{\bs}_t)\bigr)
    - \partial_{\bs} L(\bs_t,\dot{\bs}_t) = 0\,.
\end{equation*}

By definition of $\bp_t$,
\begin{equation*}
    \mathrm{d}_t\bigl(\partial_{\dot{\bs}} L(\bs_t,\dot{\bs}_t)\bigr)
    = \dot{\bp}_t\,.
\end{equation*}
Using Hamilton's second equation and the identity 
$\partial_{\bs} H = -\partial_{\bs} L$, we obtain
\begin{equation*}
    \dot{\bp}_t = -\partial_{\bs} H(\bs_t,\bp_t)
    = \partial_{\bs} L(\bs_t,\dot{\bs}_t)\,.
\end{equation*}
Therefore
\begin{equation*}
    \mathrm{d}_t\bigl(\partial_{\dot{\bs}} L(\bs_t,\dot{\bs}_t)\bigr)
    - \partial_{\bs} L(\bs_t,\dot{\bs}_t) = 0\,,
\end{equation*}
which is precisely the Euler--Lagrange equation.

\medskip
\noindent\textbf{5. Bijection of trajectories.}
Steps 3 and 4 show that:

\begin{itemize}
    \item Any trajectory $t \mapsto \bs_t$ solving the Euler--Lagrange equation,
    together with $\bp_t = \partial_{\dot{\bs}} L(\bs_t,\dot{\bs}_t)$, yields a trajectory
    $(\bs_t,\bp_t)$ solving Hamilton's equations.

    \item Any trajectory $(\bs_t,\bp_t)$ solving Hamilton's equations, together with
    $\dot{\bs}_t = \partial_{\bp} H(\bs_t,\bp_t)$, yields a trajectory $\bs_t$ solving the
    Euler--Lagrange equation.
\end{itemize}

Combined with the bijection at the level of initial condition shown in step~1,
this establishes the one-to-one correspondence between Lagrangian trajectories
$\bs_t$ and Hamiltonian trajectories $(\bs_t,\bp_t)$, together with their initial
conditions.
\end{proof}

\subsection{Constructing the invertible mapping between PFVP and HES}
\label{sec:pfvp-hes-trajectory-correspondence}

For readability, in this section we will omit the $\btheta$ dependence on the variable $\balpha_0,\bgamma_0$ and $\balpha^H_0$.

\subsubsection{Free-phase and Forward phase}

\paragraph{From PFVP to HES.} We now demonstrate how the forward phase of an
HES can be constructed from the free phase of the PFVP. 
From Proposition~\ref{prop:equivalence_IVP_EVP}, we can express the free phase of the PFVP as a solution of a IVP:
\begin{align*}
    \bs_{\bwda,t}^0\left(\btheta, \left(\balpha_T(\btheta), 
    \bgamma_T(\btheta)\right)\right) = \bs_{\fwda,t}^0(\btheta, (\balpha_0, \bgamma_0)) \quad\text{for all } t\in[0,T]\,.
\end{align*}
where $\balpha_T(\btheta) = \bs_{\fwda,T}^0(\btheta, (\balpha_0, \bgamma_0))$ and $\bgamma_T(\btheta) = \dot{\bs}_{\fwda,T}^0(\btheta, (\balpha_0, \bgamma_0))$ (Eq.~\ref{def:PFVP-boundary-from-CIVP}).
Applying the forward Legendre transformation of Theorem~\ref{thm:legendre_transform_on_dyna} on this IVP we get the HES forward trajectory $\bPhi_t({\btheta}, (\balpha_0, \bmu_0)^\top)$ that is a solution of the associated Hamilton equation of the IVP:
\begin{align}
    \label{eq:forward-construction}
    \forall t \in [0,T], \quad \bPhi_t({\btheta}, (\balpha_0, \bmu_0)^\top):=\begin{pmatrix}
        \bs_{\fwda,t}^0(\btheta, (\balpha_0, \bgamma_0)) \\
        \partial_{\dot{\bs}} L_{0}\left(
    \bs_{\fwda,t}^0, 
    \dot{\bs}_{\fwda,t}^0,\btheta\right)
    \end{pmatrix}, \quad 
    \begin{pmatrix} \balpha_0 \\ \bmu_0 \end{pmatrix} :=
    \begin{pmatrix} 
        \balpha_0 \\ 
        \partial_{\dot{\bs}} L_0(\balpha_0, \bgamma_0, \btheta)
    \end{pmatrix}\,.
\end{align}

\paragraph{From HES to PFVP.}  To construct the forward phase we applied the two following transformations:
\begin{equation*}
\underbrace{t\mapsto \bs_{\bwda,t}^0\!\left(\btheta,(\balpha_T(\btheta),\bgamma_T(\btheta))\right)}_{\text{PFVP free}}
\!\xrightarrow[\text{Prop.~\ref{prop:equivalence_IVP_EVP}}]{\text{PFVP}\to\text{IVP}}\!
\underbrace{t \mapsto\bs_{\fwda,t}^0(\btheta, (\balpha_0, \bgamma_0))}_{\text{IVP free}}
\!\xrightarrow[\text{Thm.~\ref{thm:legendre_transform_on_dyna}}]{\text{Legendre}}\!
\underbrace{t\mapsto \bPhi_t(\btheta,(\balpha_0, \bmu_0)^\top)}_{\text{HES forward}}\,.
\end{equation*}

Since each of these two transformations is bijective, their composition is also a bijection. Hence the free phase of the PFVP can be constructed from the forward phase of the HES, and vice versa. Applying the inverse maps we get:
\begin{align*}
    \forall t \in [0, T], \quad
    \begin{pmatrix} 
    \bs_{\bwda,t}^0(\btheta, (\balpha_T(\btheta), \bgamma_T(\btheta))) \\
    \dot{\bs}_{\bwda,t}^0(\btheta, (\balpha_T(\btheta), \bgamma_T(\btheta)))
    \end{pmatrix}
    := 
    \begin{pmatrix} 
    {\bs}^0_t(\btheta, (\balpha_0, \bmu_0)^\top) \\
    \partial_{\bp} H(\bPhi_t(\btheta, (\balpha_0, \bmu_0)^\top),\btheta)
    \end{pmatrix}\,,
\end{align*}
where ${\bs}^0_t(\btheta, (\balpha_0, \bmu_0)^\top)$ refers to the first vector component of $\bPhi_t(\btheta, (\balpha_0, \bmu_0)^\top)$, and $\partial_{\bp} H$ means the derivative with respect to second vector component of $\bPhi_t(\btheta, (\balpha_0, \bmu_0)^\top)$. Also, the initial condition of this PFVP are:
\begin{align*}
    \begin{pmatrix} 
    \balpha_0 \\
    \bgamma_0
    \end{pmatrix}
    := 
    \begin{pmatrix} 
    \balpha_0 \\
    \partial_{\bp} H(\balpha_0, \bmu_0,\btheta)
    \end{pmatrix}\,.
\end{align*}

where $\begin{pmatrix} \balpha_0 \\ \bmu_0 \end{pmatrix}$ with $\balpha_0$ being the position and $\bmu_0$ being the momentum.

\subsubsection{Nudged-phase and Echo-phase}

\paragraph{From PFVP to HES}

We now show how the echo phase of the HES arises from the \emph{nudged} PFVP.  
The nudged PFVP trajectory is defined by
\begin{equation*}
t\mapsto \bs_{\bwda,t}^\beta\!\left(\btheta,(\balpha_T(\btheta),\bgamma_T(\btheta))\right),
\qquad t\in[0,T]\,,
\end{equation*}
By Proposition~\ref{prop:solution-pfvp-reversibility}, this can be rewritten as a time translation of the IVP $t\mapsto \bs_{\fwda,t}^\beta \left(\btheta,(\balpha_T(\btheta),-\bgamma_T(\btheta))\right)$:
\begin{equation}
\forall t\in[0,T],\qquad
\bs_{\bwda,t}^\beta\!\left(\btheta,(\balpha_T(\btheta),\bgamma_T(\btheta))\right)
=
\bs_{\fwda,T-t}^\beta\!\left(\btheta,(\balpha_T(\btheta),-\bgamma_T(\btheta))\right)\,.
\label{eq:pfvp-ivp-echo}
\end{equation}
Applying the forward Legendre transform of
Theorem~\ref{thm:legendre_transform_on_dyna},
to the nudged IVP yields the echo phase:
\begin{align}
    \label{eq:phi_e_construction_ivp}
    \forall t \in [0,T], \quad {\bPhi}^e_t({\btheta}, \balpha^{H,e}_0):=\begin{pmatrix}
        \bs_{\fwda,t}^\beta(\btheta, (\balpha_T(\btheta), -\bgamma_T(\btheta))) \\
        \partial_{\dot{\bs}} L_{\beta}\left(
    {\mathbf{s}}_{\fwda,t}^\beta, 
    \dot{{\mathbf{s}}}_{\fwda,t}^\beta,\btheta\right)
    \end{pmatrix}, \quad 
    \balpha^{H,e}_0:=
    \begin{pmatrix} 
        \balpha_T(\btheta) \\ 
        \partial_{\dot{\bs}} L_{\beta}(\balpha_T(\btheta), -\bgamma_T(\btheta), \btheta)
    \end{pmatrix}\,.
\end{align}

To get the full mapping to desired echo phase, we now show that $\balpha^{H,e}_0=\bSigma_z\bPhi_T$. We analyze the second component of $\balpha^{H,e}_0$.  
By definition, it involves the term
\begin{equation*}
    \partial_{\dot{\bs}} L_{\beta}(\balpha_T(\btheta),-\bgamma_T(\btheta),\btheta)\,.
\end{equation*}
By Lemma~\ref{lemma:odd_derivative}, we obtain
\begin{equation*}
    \partial_{\dot{\bs}} L_{\beta}(\balpha_T(\btheta),-\bgamma_T(\btheta),\btheta)
    = -\,\partial_{\dot{\bs}} L_{\beta}(\balpha_T(\btheta),\bgamma_T(\btheta),\btheta)\,.
\end{equation*}
which gives:
\begin{align}
    \balpha^{H,e}_0&=
    \begin{pmatrix} 
        \balpha_T(\btheta) \\ 
        -\partial_{\dot{\bs}} L_{\beta}(\balpha_T(\btheta),\bgamma_T(\btheta), \btheta)
    \end{pmatrix} \notag\\
    &=
    \label{eq:lambda^e_dvped}
    \begin{pmatrix} 
        \balpha_T(\btheta) \\ 
        -\partial_{\dot{\bs}} L_{0}(\balpha_T(\btheta),\bgamma_T(\btheta), \btheta)
    \end{pmatrix}\,. \quad \text{(Lemma~\ref{lemma:beta_independent_momentum})}
\end{align}
We now evaluate Eq.~\eqref{eq:forward-construction} at time $t=T$. 
\begin{equation*}
\bPhi_T(\btheta,(\balpha_0, \bmu_0)^\top)
=
\begin{pmatrix}
\bs_{\fwda,T}^0(\btheta,(\balpha_0, \bgamma_0)) \\[3pt]
\partial_{\dot{\bs}} L_0\!\left( \bs_{\fwda,T}^0, \dot{\bs}_{\fwda,T}^0, \btheta \right)
\end{pmatrix}\,.
\end{equation*}
By the PFVP construction (Eq.~\eqref{def:PFVP-boundary-from-CIVP}),
\begin{equation*}
\bs_{\fwda,T}^0 = \balpha_T(\btheta),
\qquad
\dot{\bs}_{\fwda,T}^0 = \bgamma_T(\btheta)\,,
\end{equation*}
so that
\begin{equation}
\label{eq:phi_0}
    \bPhi_T(\btheta,(\balpha_0, \bmu_0)^\top)
=
\begin{pmatrix}
\balpha_T(\btheta) \\[3pt]
\partial_{\dot{\bs}} L_0(\balpha_T(\btheta),\bgamma_T(\btheta),\btheta)
\end{pmatrix}\,.
\end{equation}
Taking this last Equation~\eqref{eq:phi_0} with Equation~\eqref{eq:lambda^e_dvped}, we have the final condition that makes the constructed echo-phase well-defined:
\begin{align*}
\;\balpha^{H,e}_0 =  \bSigma_z\, \bPhi_T(\btheta,(\balpha_0, \bmu_0)^\top)\,.
\end{align*}
Rewriting our construction (Equation~\eqref{eq:phi_e_construction_ivp})  in terms of PFVP variables, we have constructed $t\mapsto \bPhi_t^e(\btheta,\balpha^{H,e}_0)$ with:
\begin{equation}
    \label{eq:phe_e_construction}
    \forall t \in [0,T], \quad {\bPhi}^e_t({\btheta}, \balpha^{H,e}_0):=\begin{pmatrix}
        \bs_{\bwda,T-t}^\beta(\btheta, (\balpha_T(\btheta), \bgamma_T(\btheta))) \\
        \partial_{\dot{\bs}} L_{\beta}\left(
    \bs_{\bwda,T-t}^\beta,
    -\dot{\bs}_{\bwda,T-t}^\beta,\btheta\right)
    \end{pmatrix} \quad, \balpha^{H,e}_0 := \bSigma_z\, \begin{pmatrix}
\balpha_T(\btheta) \\[3pt]
\partial_{\dot{\bs}} L_0(\balpha_T(\btheta),\bgamma_T(\btheta),\btheta)
\end{pmatrix}\,.
\end{equation}

\paragraph{From HES to PFVP.}  To construct the echo phase, we applied the two following transformations:

\begin{equation*}
\underbrace{t\mapsto \bs_{\bwda,t}^\beta\!\left(\btheta,(\balpha_T(\btheta),\bgamma_T(\btheta))\right)}_{\text{PFVP nudge}}
\!\xrightarrow[\text{Prop.~\ref{prop:solution-pfvp-reversibility}}]{\text{PFVP}\to\text{IVP}  \text{, time translation}}\!
\underbrace{t\mapsto \bs_{\fwda,t}^\beta\!\left(\btheta,(\balpha_T(\btheta),-\bgamma_T(\btheta))\right)}_{\text{IVP nudge}}
\!\xrightarrow[\text{Thm.~\ref{thm:legendre_transform_on_dyna}}]{\text{Legendre}}\!
\underbrace{t\mapsto \bPhi_t^e(\btheta,\bSigma_z\bPhi_T)}_{\text{HES echo}}\,.
\end{equation*}
Since each of these two transformations is bijective, their composition is also a bijection. Hence the nudged phase of the PFVP can be constructed from the echo phase of the HES, and vice versa.


\subsection{Gradient Equivalence.}
We prove that the PFVP gradient estimator equals the RHEL gradient estimator by applying the forward Legendre transform. This direction of the proof leverages the trajectory correspondences already established in Section~\ref{sec:pfvp-hes-trajectory-correspondence}.

\subsection*{Starting Point: PFVP Gradient Estimator}
The PFVP gradient estimator in Lagrangian variables is (from Theorem~\ref{thm:pfvp-cancellation}):
\begin{align*}
\Delta^{\text{PFVP}}(\beta, \balpha_0, \bgamma_0) := \frac{1}{\beta} \Bigg[ &\underbrace{\int_0^T \left(\partial_{\btheta} L_{\beta}(\bs_{\bwda,t}^{\beta}, \dot{\bs}_{\bwda,t}^{\beta}, \btheta) - \partial_{\btheta} L_{0}(\bs_{\bwda,t}^{0}, \dot{\bs}_{\bwda,t}^{0}, \btheta)\right) \dt}_{\text{Integral term}: \bm{I}} \\
&+ \underbrace{d_{\btheta}\begin{pmatrix}\partial_{\dot{\bs}} L_0(\balpha_0, \bgamma_0, \btheta) \\
\balpha_0
\end{pmatrix}^\top \bSigma_z
\begin{pmatrix}
(\bs_{\bwda,0}^\beta - \balpha_0) \\
\partial_{\dot{\bs}} L_\beta(\bs_{\bwda,0}^\beta, \dot{\bs}_{\bwda,0}^\beta, \btheta) - \partial_{\dot{\bs}} L_0(\balpha_0, \bgamma_0, \btheta)
\end{pmatrix}}_{\text{Boundary term}: \bm{B}} \Bigg]\,.
\end{align*}
Our goal is to show that this gradient estimator is equivalent to the following one (Theorem~\ref{thm:rhel}):
\begin{align*}
\Delta^{\text{RHEL}}(\beta, \balpha^H_0(\btheta)) = -\frac{1}{\beta} \left(\int_0^T 
\left[ \partial_{\btheta} H_\beta(\bPhi^e_t(\beta), {\btheta}) 
- \partial_{\btheta} H_0(\bPhi_t, {\btheta}) \right]\dt - \left(\partial_{\btheta} \balpha^H_0 \right)^\top \bSigma_x (\bPhi^e_T(\beta) - \bSigma_z\bPhi_0)\right)\,.
\end{align*}

\subsection*{Main Proof: Transforming PFVP to RHEL}

The proof relies on the trajectory correspondences established in Section E.2. Rather than restating these correspondences, we will reference the relevant equations from E.2 as needed throughout the proof.

\paragraph{Step 1: Transform the Integral Term}

We start with the integral term of PFVP:
\begin{align*}
    \bm{I} = \int_0^T \left(\partial_{\btheta} L_{\beta}(\bs_{\bwda,t}^{\beta}, \dot{\bs}_{\bwda,t}^{\beta}, \btheta) - \partial_{\btheta} L_{0}(\bs_{\bwda,t}^{0}, \dot{\bs}_{\bwda,t}^{0}, \btheta)\right) \dt\,.
\end{align*}

\textbf{Step 1.1: Applying the Parameter-Gradient Relation.}

To transform this integral, we use the parameter-gradient relation established in Lemma~\ref{lemma:param_grad}: $\partial_{\btheta} H(\bPhi_t,\btheta) = - \partial_{\btheta} L(\bs_t,\dot{\bs}_t,\btheta)$.

Recall from the beginning of Step 1:
\begin{align*}
    \bm{I} = \int_0^T \left(\partial_{\btheta} L_{\beta}(\bs_{\bwda,t}^{\beta}, \dot{\bs}_{\bwda,t}^{\beta}, \btheta) - \partial_{\btheta} L_{0}(\bs_{\bwda,t}^{0}, \dot{\bs}_{\bwda,t}^{0}, \btheta)\right) \dt, \quad t \in [0, T]\,.
\end{align*}

By Lemma~\ref{lemma:beta_independent_momentum}, we have $\partial_{\btheta} L_\beta = \partial_{\btheta} L_0$. Thus:
\begin{align*}
    \bm{I} = \int_0^T \left(\partial_{\btheta} L_{0}(\bs_{\bwda,t}^{\beta}, \dot{\bs}_{\bwda,t}^{\beta}, \btheta) - \partial_{\btheta} L_{0}(\bs_{\bwda,t}^{0}, \dot{\bs}_{\bwda,t}^{0}, \btheta)\right) \dt, \quad t \in [0, T]\,.
\end{align*}

To transform this to Hamiltonians, we recall the two equality from Section E.2:
\begin{align*}
    {\bPhi}^e_t=\begin{pmatrix}
        \bs_{\bwda,T-t}^\beta \\
        \partial_{\dot{\bs}} L_{\beta}(\bs_{\bwda,T-t}^\beta, -\dot{\bs}_{\bwda,T-t}^\beta,\btheta)
    \end{pmatrix}, \quad t \in [0,T]\,, \quad \text{(Eq.~\ref{eq:phe_e_construction})}
\end{align*}
\begin{align*}
    \bPhi_t = \begin{pmatrix}
        \bs_{\fwda,t}^0 \\
        \partial_{\dot{\bs}} L_{0}(\bs_{\fwda,t}^0, \dot{\bs}_{\fwda,t}^0, \btheta)
    \end{pmatrix}, \quad t \in [0,T]\,. \quad \text{(Eq.~\ref{eq:forward-construction})}
\end{align*}

We apply Lemma~\ref{lemma:param_grad} to transform each term. 

For the first term, we apply Lemma~\ref{lemma:param_grad} to the augmented system with Hamiltonian $H_\beta$ and Lagrangian $L_\beta$ at $\bPhi^e_{t'}$ with $t' \in [0, T]$.
The Lagrangian trajectory corresponding to $\bPhi^e_{t'}$ is the IVP trajectory at time $t'$, whose velocity is $-\dot{\bs}_{\bwda,T-t'}^{\beta}$ (cf.\ Eq.~\ref{eq:pfvp-ivp-echo}). Thus:
\begin{align*}
    \partial_{\btheta} H_\beta(\bPhi^e_{t'}, \btheta) = -\partial_{\btheta} L_{\beta}(\bs_{\bwda,T-t'}^{\beta}, -\dot{\bs}_{\bwda,T-t'}^{\beta}, \btheta), \quad t' \in [0, T]\,.
\end{align*}
By Lemma~\ref{lemma:beta_independent_momentum}, we have $\partial_{\btheta} L_\beta = \partial_{\btheta} L_0$ and $\partial_{\btheta} H_\beta = \partial_{\btheta} H_0$, giving:
\begin{align*}
    \partial_{\btheta} H_0(\bPhi^e_{t'}, \btheta) = -\partial_{\btheta} L_{0}(\bs_{\bwda,T-t'}^{\beta}, -\dot{\bs}_{\bwda,T-t'}^{\beta}, \btheta), \quad t' \in [0, T]\,.
\end{align*}
Since $L_0$ is a reversible Lagrangian, i.e.\ $L_0(\bs, -\dot{\bs}, \btheta) = L_0(\bs, \dot{\bs}, \btheta)$ (cf.\ Eq.~\ref{def:PFVP-general}), differentiating with respect to $\btheta$ gives $\partial_{\btheta} L_{0}(\bs, -\dot{\bs}, \btheta) = \partial_{\btheta} L_{0}(\bs, \dot{\bs}, \btheta)$.
Change of variables $t' = T - t$ then gives:
\begin{align*}
    \partial_{\btheta} L_{0}(\bs_{\bwda,t}^{\beta}, \dot{\bs}_{\bwda,t}^{\beta}, \btheta) = -\partial_{\btheta} H_0(\bPhi^e_{T-t}, \btheta), \quad t \in [0, T]\,.
\end{align*}

For the second term, we apply Lemma~\ref{lemma:param_grad} to the non-augmented system with Hamiltonian $H_0$ and Lagrangian $L_0$ at $\bPhi_{t'}$ with $t' \in [0, T]$:
\begin{align*}
    \partial_{\btheta} H_0(\bPhi_{t'}, \btheta) = -\partial_{\btheta} L_{0}(\bs_{\bwda,t'}^{0}, \dot{\bs}_{\bwda,t'}^{0}, \btheta), \quad t' \in [0, T]\,.
\end{align*}
Therefore:
\begin{align*}
    \partial_{\btheta} L_{0}(\bs_{\bwda,t}^{0}, \dot{\bs}_{\bwda,t}^{0}, \btheta) = -\partial_{\btheta} H_0(\bPhi_t, \btheta), \quad t \in [0, T]\,.
\end{align*}

Substituting both results into $\bm{I}$ for $t \in [0, T]$:
\begin{align*}
    \bm{I} &= \int_0^T \left(-\partial_{\btheta} H_0(\bPhi^e_{T-t}, \btheta) - (-\partial_{\btheta} H_0(\bPhi_t, \btheta))\right) \dt \\
    &= \int_0^T \left(-\partial_{\btheta} H_0(\bPhi^e_{T-t}, \btheta) + \partial_{\btheta} H_0(\bPhi_t, \btheta)\right) \dt\,.
\end{align*}

Final change of variables: Let $t' = T - t$ so that $dt' = -dt$. When $t \in [0,T]$, we have $t' \in [T, 0]$:
\begin{align*}
    \bm{I} &= \int_T^0 \left(\partial_{\btheta} H_0(\bPhi^e_{t'}, \btheta) - \partial_{\btheta} H_0(\bPhi_{T-t'}, \btheta)\right) (-dt') \\
    &= -\int_0^T \left(\partial_{\btheta} H_0(\bPhi^e_t, \btheta) - \partial_{\btheta} H_0(\bPhi_{T-t}, \btheta)\right) dt \\
    &= -\int_0^T \left(\partial_{\btheta} H_0(\bPhi^e_t, \btheta) - \partial_{\btheta} H_0(\bPhi_t, \btheta)\right) dt\,.
\end{align*}
where the last equality uses the change of dummy integration variable $u = T - t$ in the second term only: $\int_0^T f(\bPhi_{T-t})\,dt = \int_0^T f(\bPhi_u)\,du$.

This matches (up to sign) the integral term in RHEL.

\paragraph{Step 2: Transform the Boundary Term}

The boundary term in PFVP (from Theorem~\ref{thm:pfvp-cancellation}) is:
\begin{align*}
    \bm{B} = d_{\btheta}\begin{pmatrix}\balpha_0 \\
\partial_{\dot{\bs}} L_0(\balpha_0, \bgamma_0, \btheta)
\end{pmatrix}^\top \bSigma_x
\begin{pmatrix}
(\bs_{\bwda,0}^\beta - \balpha_0) \\
-\partial_{\dot{\bs}}L_\beta(\bs_{\bwda,0}^\beta, \dot{\bs}_{\bwda,0}^\beta, \btheta)+\partial_{\dot{\bs}} L_0(\balpha_0, \bgamma_0, \btheta)
\end{pmatrix}\,.
\end{align*}

Recall from Section~\ref{sec:pfvp-hes-trajectory-correspondence} the mapping:
\begin{align*}
    \bPhi^e_T &= \begin{pmatrix}
        \bs_{\bwda,0}^\beta \\
        \partial_{\dot{\bs}} L_\beta(\bs_{\bwda,0}^\beta, -\dot{\bs}_{\bwda,0}^\beta, \btheta)
    \end{pmatrix}\,, \quad \text{(Eq.~\ref{eq:phe_e_construction} at $t=T$)}
\end{align*}
\begin{align}
    \label{eq:initial_mapping}
    \bPhi_0 = \begin{pmatrix} \balpha_0 \\ \bmu_0 \end{pmatrix} = \begin{pmatrix}
        \balpha_0 \\
        \partial_{\dot{\bs}} L_0(\balpha_0, \bgamma_0, \btheta)
    \end{pmatrix}\,. \quad \text{(Eq.~\ref{eq:forward-construction} at $t=0$)}
\end{align}

Therefore:
\begin{align}
    \bPhi^e_T - \bSigma_z\bPhi_0 &= \begin{pmatrix}
        \bs_{\bwda,0}^\beta - \balpha_0 \\
        \partial_{\dot{\bs}} L_\beta(\bs_{\bwda,0}^\beta, -\dot{\bs}_{\bwda,0}^\beta, \btheta) + \partial_{\dot{\bs}} L_0(\balpha_0, \bgamma_0, \btheta)
    \end{pmatrix}\notag\\
    &=\begin{pmatrix}
        \bs_{\bwda,0}^\beta - \balpha_0 \\
        -\partial_{\dot{\bs}} L_\beta(\bs_{\bwda,0}^\beta, \dot{\bs}_{\bwda,0}^\beta, \btheta) + \partial_{\dot{\bs}} L_0(\balpha_0, \bgamma_0, \btheta)
    \end{pmatrix}\,. \quad \text{(Lemma~\ref{lemma:odd_derivative})} \label{eq:phi_sum}
\end{align}

Also, from the initial condition $\begin{pmatrix} \balpha_0 \\ \bmu_0 \end{pmatrix} = \bPhi_0$ (Eq.~\ref{eq:initial_mapping}), we can deduce:
\begin{align}\label{eq:partial_theta_lambda}
    \partial_{\btheta} \begin{pmatrix} \balpha_0 \\ \bmu_0 \end{pmatrix} = \partial_{\btheta} \begin{pmatrix} 
        \balpha_0 \\ 
        \bmu_0
    \end{pmatrix} = d_{\btheta} \begin{pmatrix} 
        \balpha_0 \\ 
        \partial_{\dot{\bs}} L_0(\balpha_0, \bgamma_0, \btheta) 
    \end{pmatrix}\,. 
\end{align}

We now show that the RHEL boundary term equals $\bm{B}$. Starting from the RHEL boundary term:
\begin{align*}
    \left(\partial_{\btheta} \balpha^H_0 \right)^\top \bSigma_x (\bPhi^e_T - \bSigma_z\bPhi_0)
    &= \left(d_{\btheta}\begin{pmatrix}\balpha_0 \\
\partial_{\dot{\bs}} L_0(\balpha_0, \bgamma_0, \btheta)
\end{pmatrix}\right)^\top \bSigma_x (\bPhi^e_T - \bSigma_z\bPhi_0) \quad \text{(substitute Eq.~\ref{eq:partial_theta_lambda})} \\
    &= \left(d_{\btheta}\begin{pmatrix}\balpha_0 \\
\partial_{\dot{\bs}} L_0(\balpha_0, \bgamma_0, \btheta)
\end{pmatrix}\right)^\top \bSigma_x \begin{pmatrix}
        \bs_{\bwda,0}^\beta - \balpha_0 \\
        -\partial_{\dot{\bs}} L_\beta(\bs_{\bwda,0}^\beta, \dot{\bs}_{\bwda,0}^\beta, \btheta) + \partial_{\dot{\bs}} L_0(\balpha_0, \bgamma_0, \btheta)
    \end{pmatrix} \quad \text{(substitute Eq.~\ref{eq:phi_sum})} \\
    &= \bm{B}\,. \quad \text{(matches the PFVP boundary term)}
\end{align*}

This shows the boundary terms match exactly.

\paragraph{Step 3: Combine and Conclude}

Combining both terms from Step 1 and Step 2, we have:
\begin{align*}
    \Delta^{\text{PFVP}}(\beta) &= \frac{1}{\beta} \left(\bm{I} + \bm{B}\right) \\
    &= \frac{1}{\beta} \left(-\int_0^T \left(\partial_{\btheta} H_0(\bPhi^e_t, \btheta) - \partial_{\btheta} H_0(\bPhi_t, \btheta)\right) dt + \left(\partial_{\btheta} \begin{pmatrix} \balpha_0 \\ \bmu_0 \end{pmatrix} \right)^\top \bSigma_x (\bPhi^e_T(\beta) - \bSigma_z\bPhi_0)\right) \\
    &= -\frac{1}{\beta} \left(\int_0^T \left(\partial_{\btheta} H_0(\bPhi^e_t, \btheta) - \partial_{\btheta} H_0(\bPhi_t, \btheta)\right) dt - \left(\partial_{\btheta} \begin{pmatrix} \balpha_0 \\ \bmu_0 \end{pmatrix} \right)^\top \bSigma_x (\bPhi^e_T(\beta) - \bSigma_z\bPhi_0)\right) \\
    &= \Delta^{\text{RHEL}}(\beta, \balpha_0(\btheta), \bmu_0(\btheta))\,.
\end{align*}



\section{Dissipative Lagrangian Equilibrium Propagation}
\label{appx:dissipative_LEP}

This appendix presents the general theory of dissipative Lagrangian Equilibrium Propagation (LEP), including the proof of the main theorem and the energy dissipation property. The harmonic oscillator instantiation is presented in the following section as a concrete example.

\subsection{Proof of Theorem~\ref{thm:dissipative_LEP}: Dissipative LEP}

\begin{proof}
We first derive the Euler-Lagrange equation~\eqref{eq:dissipative_EL}, then apply Theorem~\ref{thm:pfvp-cancellation}.

\textbf{Step 1: Derivation of the dissipative Euler-Lagrange equation.}
The standard Euler-Lagrange equation for $L^{\mathrm{diss}}_\beta$ is:
\begin{equation*}
    \partial_{\bs} L^{\mathrm{diss}}_\beta - d_t\partial_{\dot{\bs}} L^{\mathrm{diss}}_\beta = 0\,.
\end{equation*}

Since $c(\bs_t, \by_t)$ does not depend on $\dot{\bs}_t$, the velocity derivative is:
\begin{equation*}
    \partial_{\dot{\bs}} L^{\mathrm{diss}}_\beta = \exp(\bzeta t) \cdot \partial_{\dot{\bs}} L_0\,.
\end{equation*}

Taking the time derivative using the product rule:
\begin{align*}
    d_t\partial_{\dot{\bs}} L^{\mathrm{diss}}_\beta &= \bzeta \exp(\bzeta t) \cdot \partial_{\dot{\bs}} L_0 + \exp(\bzeta t) \cdot d_t\left( \partial_{\dot{\bs}} L_0 \right)  \\
    &= \exp(\bzeta t) \cdot \left( \bzeta \, \partial_{\dot{\bs}} L_0 + d_t\partial_{\dot{\bs}} L_0 \right)\,.
\end{align*}

The position derivative is:
\begin{equation*}
    \partial_{\bs} L^{\mathrm{diss}}_\beta = \exp(\bzeta t) \cdot \partial_{\bs} L_0 + \beta \, \partial_{\bs} c\,.
\end{equation*}

Substituting into the Euler-Lagrange equation $\partial_{\bs} L^{\mathrm{diss}}_\beta - d_t\partial_{\dot{\bs}} L^{\mathrm{diss}}_\beta = 0$ and multiplying through by $\exp(-\bzeta t)$ yields~\eqref{eq:dissipative_EL}.

\textbf{Step 2: Physical interpretation (free phase).}
For $\beta = 0$, dividing by $\exp(\bzeta t) \neq 0$:
\begin{equation*}
    \partial_{\bs} L_0 - d_t\partial_{\dot{\bs}} L_0 = \bzeta \, \partial_{\dot{\bs}} L_0\,.
\end{equation*}
This shows that the effect of the exponential time-scaling is to add a friction-like term proportional to $\partial_{\dot{\bs}} L_0$ to the standard Euler-Lagrange equation. When the Lagrangian has quadratic kinetic energy ($\partial_{\dot{\bs}} L_0 = \dot{\bs}$), this reduces to Newton's second law with viscous friction $\mathbf{F}_{\mathrm{friction}} = -\bzeta \dot{\bs}$.

\textbf{Step 3: Application of Theorem~\ref{thm:pfvp-cancellation}.}
Since $\partial_{\btheta} L^{\mathrm{diss}}_\beta = \partial_{\btheta} L_0 \cdot \exp(\bzeta t)$ (the cost $c$ does not depend on $\btheta$), the integral term in the PFVP gradient estimator becomes:
\begin{equation*}
    \int_0^T \left( \partial_{\btheta} L_0(\bs_{\bwda,t}^\beta, \dot{\bs}_{\bwda,t}^\beta, \btheta) - \partial_{\btheta} L_0(\bs_{\bwda,t}^0, \dot{\bs}_{\bwda,t}^0, \btheta) \right) \exp(\bzeta t) \, \mathrm{d}t\,.
\end{equation*}

For the boundary terms at $t = 0$, we have $\partial_{\dot{\bs}} L^{\mathrm{diss}}_\beta = \partial_{\dot{\bs}} L_0 \cdot \exp(\bzeta \cdot 0) = \partial_{\dot{\bs}} L_0$, so they remain unchanged from Theorem~\ref{thm:pfvp-cancellation}. The PFVP-to-IVP reduction (Proposition~\ref{prop:solution-pfvp-reversibility}) generalizes to the dissipative setting: since the undamped Lagrangian $L_0$ is time-reversible, the bouncing-backward kick applies with the replacement $\bzeta \to -\bzeta$ in the echo phase, corresponding to energy pumping during the nudged backward trajectory (see Proposition~\ref{prop:dissipative_time_reversal}).

\textbf{Remark (Exponential weighting):} The factor $\exp(\bzeta t)$ weights later time steps exponentially more than earlier ones. 
\end{proof}

\subsection{Proof of Proposition~\ref{prop:energy_dissipation}: Energy Dissipation}

\begin{proposition}[Energy Dissipation]
\label{prop:energy_dissipation}
Consider the isolated dissipative system ($\bx_t = 0$). For a trajectory $t \mapsto \bs_t$ satisfying the dissipative Euler-Lagrange equation~\eqref{eq:dissipative_EL} with $\beta = 0$ and $\bx_t = 0$, the physical energy $E$ (defined as in~\eqref{eq:energy_definition}) evolves according to:
\begin{equation}
    d_t E = - \bzeta \, \dot{\bs}_t^\top \partial_{\dot{\bs}} L_0^{\mathrm{iso}}\,.
    \label{eq:energy_dissipation_general}
\end{equation}

\textbf{Quadratic kinetic energy case:} When the Lagrangian admits a decomposition $L_0^{\mathrm{iso}} = E_{\mathrm{kin}}(\dot{\bs}_t) - U_{\mathrm{int}}(\bs_t, \btheta)$ with quadratic kinetic energy $E_{\mathrm{kin}}(\dot{\bs}_t) = \frac{1}{2} \|\dot{\bs}_t\|^2$, we have $\partial_{\dot{\bs}} L_0^{\mathrm{iso}} = \dot{\bs}_t$, yielding:
\begin{equation}
    d_t E = - \bzeta \|\dot{\bs}_t\|^2 \leq 0\,.
    \label{eq:energy_dissipation}
\end{equation}
Since $\bzeta > 0$, energy is strictly dissipated whenever $\dot{\bs}_t \neq 0$.
\end{proposition}

\begin{proof}
Starting from the energy definition~\eqref{eq:energy_definition}:
\begin{equation*}
    E = \dot{\bs}_t^\top \partial_{\dot{\bs}} L_0^{\mathrm{iso}} - L_0^{\mathrm{iso}}\,.
\end{equation*}

Taking the time derivative:
\begin{align*}
    d_t E &= \ddot{\bs}_t^\top \partial_{\dot{\bs}} L_0^{\mathrm{iso}} + \dot{\bs}_t^\top d_t\left(\partial_{\dot{\bs}} L_0^{\mathrm{iso}}\right) - d_t L_0^{\mathrm{iso}}  \\
    &= \ddot{\bs}_t^\top \partial_{\dot{\bs}} L_0^{\mathrm{iso}} + \dot{\bs}_t^\top d_t\left(\partial_{\dot{\bs}} L_0^{\mathrm{iso}}\right) - \partial_{\bs} L_0^{\mathrm{iso}} \cdot \dot{\bs}_t - \partial_{\dot{\bs}} L_0^{\mathrm{iso}} \cdot \ddot{\bs}_t  \\
    &= \dot{\bs}_t^\top \left( d_t\partial_{\dot{\bs}} L_0^{\mathrm{iso}} - \partial_{\bs} L_0^{\mathrm{iso}} \right)\,,
\end{align*}
where the first two terms (with $\ddot{\bs}_t$) cancel, and the chain rule gives $d_t L_0^{\mathrm{iso}} = \partial_{\bs} L_0^{\mathrm{iso}} \cdot \dot{\bs}_t + \partial_{\dot{\bs}} L_0^{\mathrm{iso}} \cdot \ddot{\bs}_t$.

For the isolated system ($\bx_t = 0$) with $\beta = 0$, the dissipative Euler-Lagrange equation~\eqref{eq:dissipative_EL} reduces to:
\begin{equation*}
    \partial_{\bs} L_0^{\mathrm{iso}} - d_t\partial_{\dot{\bs}} L_0^{\mathrm{iso}} - \bzeta \, \partial_{\dot{\bs}} L_0^{\mathrm{iso}} = 0\,.
\end{equation*}

Rearranging:
\begin{equation*}
    d_t\partial_{\dot{\bs}} L_0^{\mathrm{iso}} - \partial_{\bs} L_0^{\mathrm{iso}} = -\bzeta \, \partial_{\dot{\bs}} L_0^{\mathrm{iso}}\,.
\end{equation*}

Substituting into the energy evolution expression:
\begin{equation*}
    d_t E = \dot{\bs}_t^\top \left( -\bzeta \, \partial_{\dot{\bs}} L_0^{\mathrm{iso}} \right) = - \bzeta \, \dot{\bs}_t^\top \partial_{\dot{\bs}} L_0^{\mathrm{iso}}\,.
\end{equation*}

This proves equation~\eqref{eq:energy_dissipation_general}.

For the quadratic kinetic energy case where $L_0^{\mathrm{iso}} = \frac{1}{2} \|\dot{\bs}_t\|^2 - U_{\mathrm{int}}(\bs_t, \btheta)$, we have:
\begin{equation*}
    \partial_{\dot{\bs}} L_0^{\mathrm{iso}} = \dot{\bs}_t\,.
\end{equation*}

Therefore:
\begin{equation*}
    d_t E = - \bzeta \, \dot{\bs}_t^\top \dot{\bs}_t = - \bzeta \|\dot{\bs}_t\|^2 \leq 0\,.
\end{equation*}

Since $\bzeta > 0$, this shows that energy is strictly dissipated (decreases) whenever $\dot{\bs}_t \neq 0$, confirming the physically expected behavior of a dissipative system.
\end{proof}

\section{Dissipative Harmonic Oscillators: Complete Derivation}
\label{appx:dissipative_oscillators}

This appendix provides the complete derivation of the dissipative harmonic oscillator system summarized in Table~\ref{tab:dissipative_summary} of Section~\ref{subsec:empirical}.

\subsection{Derivation of Free and Nudged Dynamics}

\paragraph{Lagrangian and dissipative formulation.}
The physical Lagrangian is given by equation~\eqref{eq:harmonic_lagrangian}:
\begin{equation*}
    L_0(\bs_t, \dot{\bs}_t, \btheta, x_t) = \frac{1}{2} (\bm{m} \odot \dot{\bs}_t) \cdot \dot{\bs}_t - \frac{1}{2} \bs_t^\top \bm{K} \bs_t - \bm{e}_1^\top \bs_t \, x_t\,,
\end{equation*}
where the kinetic energy uses the mass vector $\bm{m}$ with element-wise operations, and the potential energy uses the dense symmetric stiffness matrix $\bm{K}$ that couples all oscillators. The input coupling term $-\bm{e}_1^\top \bs_t \, x_t = -s_{1,t} x_t$ describes the external force acting on the first oscillator.

Following the dissipative Lagrangian formulation~\eqref{eq:dissipative_lagrangian}, we use a scalar damping coefficient $\bzeta > 0$. This gives the dissipative Lagrangian:
\begin{equation*}
    L^{\mathrm{diss}}_\beta(\bs_t, \dot{\bs}_t, \btheta, x_t, y_t) = \exp(\bzeta t) \cdot L_0(\bs_t, \dot{\bs}_t, \btheta, x_t) + \beta \, c(\bs_t, y_t)\,,
\end{equation*}
with cost function $c(\bs_t, y_t) = \frac{1}{2}(s_{d,t} - y_t)^2$, where $s_{d,t}$ denotes the $d$-th component of $\bs_t$ (the last oscillator).

\paragraph{Free dynamics ($\beta = 0$).}
Applying the dissipative Euler-Lagrange equation~\eqref{eq:dissipative_EL}, the free dynamics are:
\begin{align*}
    \partial_{\bs} L_0 - d_t \partial_{\dot{\bs}} L_0 - \bzeta \, \partial_{\dot{\bs}} L_0 &= \mathbf{0}\,.
\end{align*}
Computing the gradients:
\begin{align*}
    \partial_{\bs} L_0 &= -\bm{K} \bs_t - x_t \bm{e}_1, \qquad
    \partial_{\dot{\bs}} L_0 = \bm{m} \odot \dot{\bs}_t\,.
\end{align*}
Defining the element-wise damping vector $\bgamma := \zeta \bm{m} = (\zeta m_1, \ldots, \zeta m_d)^\top$, this yields the driven damped coupled harmonic oscillator equations:
\begin{equation*}
    \bm{m} \odot \ddot{\bs}_t + \bgamma \odot \dot{\bs}_t + \bm{K} \bs_t = -x_t \bm{e}_1\,.
\end{equation*}
This recovers the well-known damped harmonic oscillator equation with proportional damping (damping force proportional to mass with uniform coefficient $\zeta$).

\paragraph{Nudged dynamics ($\beta > 0$).}
With the cost function term acting on the last oscillator, applying the dissipative Euler-Lagrange equation gives:
\begin{equation*}
    \bm{m} \odot \ddot{\bs}^\beta_t + \bgamma \odot \dot{\bs}^\beta_t + \bm{K} \bs^\beta_t = -x_t \bm{e}_1 - \beta \, \exp(-\bzeta t) \, \bm{e}_d (s_{d,t}^\beta - y_t)\,,
\end{equation*}
where $\bm{e}_d = (0, \ldots, 0, 1)^\top$ selects the last oscillator where the cost is applied. Note the exponential factor $\exp(-\bzeta t)$ in the nudging term, which arises from the dissipative Lagrangian formulation and ensures that the nudging strength is properly weighted along the time-scaled trajectory.

\subsection{Time-Reversibility and PFVP Implementation}

As in Lagrangian EP, both the free and nudged phases are formulated as \emph{Parametric Final Value Problems} (PFVP), where the final conditions at time $T$ are parametrically determined by $\btheta$, while the initial conditions are fixed.

\textbf{Free phase:} The free dynamics are solved as a standard initial value problem, integrating forward in time from $t = 0$ to $t = T$ with fixed initial conditions $(\bs_0, \dot{\bs}_0) = (\balpha_0, \mathbf{0})$:
\begin{equation*}
    \bm{m} \odot \ddot{\bs}^0_t + \bgamma \odot \dot{\bs}^0_t + \bm{K} \bs^0_t = -x_t \bm{e}_1, \quad t \in [0, T]\,.
\end{equation*}
This yields the free trajectory and determines the final conditions $(\bs^0_T, \dot{\bs}^0_T)$.

\textbf{Nudged phase:} The nudged dynamics are formulated as a final value problem. To implement the PFVP condition that both free and nudged trajectories share the same final state, we solve the nudged dynamics \emph{backward in time} from $t = T$ to $t = 0$, starting from the final conditions $(\bs^\beta_T, \dot{\bs}^\beta_T) = (\bs^0_T, \dot{\bs}^0_T)$.

The critical implementation detail is given by the following proposition:

\begin{proposition}[Time-reversibility of dissipative PFVP]
\label{prop:dissipative_time_reversal}
Consider the dissipative dynamics with damping vector $\bgamma = \zeta \bm{m}$ (where $\zeta > 0$ is scalar), mass vector $\bm{m}$, and stiffness matrix $\bm{K}$:
\begin{equation*}
    \bm{m} \odot \ddot{\bs}_t + \bgamma \odot \dot{\bs}_t + \bm{K} \bs_t = \bm{f}(t)\,,
\end{equation*}
where $\bm{f}(t) \in \mathbb{R}^d$ is an external forcing term. The solution of the PFVP with final conditions $(\bs_T, \dot{\bs}_T)$ can be computed by integrating \emph{forward in time} $t' \in [0,T]$ the Initial Value Problem with velocity-reversed initial conditions $(\bs_T, -\dot{\bs}_T)$ where the dissipative term changes sign:
\begin{equation*}
    \bm{m} \odot \ddot{\bs}_{t'} - \bgamma \odot \dot{\bs}_{t'} + \bm{K} \bs_{t'} = \bm{f}(T - t'), \quad t' \in [0, T]\,,
\end{equation*}
with initial conditions $(\bs_0, \dot{\bs}_0) = (\bs_T, -\dot{\bs}_T)$. The PFVP solution at physical time $t$ is given by $\bs_t = \bs_{t'}$ where $t' = T - t$.
\end{proposition}

\begin{proof}
See Appendix~\ref{appx:dissipative_time_reversal}.
\end{proof}

Applying Proposition~\ref{prop:dissipative_time_reversal} to the nudged dynamics, we integrate \emph{forward in time} $t' \in [0, T]$ starting from the velocity-reversed final conditions $(\bs^\beta_T, -\dot{\bs}^\beta_T) = (\bs^0_T, -\dot{\bs}^0_T)$:
\begin{equation*}
    \bm{m} \odot \ddot{\bs}^\beta_{t'} - \bgamma \odot \dot{\bs}^\beta_{t'} + \bm{K} \bs^\beta_{t'} = -x_{T-t'} \bm{e}_1 - \beta \, \exp(-\bzeta (T-t')) \, \bm{e}_d (s_{d,t'}^\beta - y_{T-t'}), \quad t' \in [0, T]\,.
\end{equation*}
Crucially, this is an Initial Value Problem that is integrated \emph{forward} in integration time $t'$ from $0$ to $T$ (corresponding to physical time $t$ going backward from $T$ to $0$). The inputs $x_{T-t'}$ and targets $y_{T-t'}$ are fed in reverse temporal order: at integration time $t'$, we use the input and target from physical time $T-t'$.

\textbf{Physical interpretation:} The sign flip has a natural physical interpretation. When we run a dissipative system forward in time, energy is dissipated and the system loses energy through friction (see Eq.~\eqref{eq:energy_dissipation}, where the term $-\bzeta \|\dot{\bs}_t\|^2 < 0$ represents energy dissipation). When we run the nudge phase backward, the friction term must reverse its action—effectively \emph{adding} energy back into the system (as $-\bzeta$ becomes $+\bzeta$, making the term positive).

\subsection{Gradient Estimator with Fixed Initial Conditions}

For fixed initial conditions $\bs_0 = \balpha_0$ (independent of $\btheta$) and zero initial velocity $\dot{\bs}_0 = \mathbf{0}$, the gradient estimator from Theorem~\ref{thm:dissipative_LEP} simplifies. The boundary terms in~\eqref{eq:dissipative_PFVP_gradient} cancel because:
\begin{itemize}[leftmargin=*]
    \item At $t = 0$: The initial conditions are fixed ($\partial_{\btheta} \balpha_0 = \mathbf{0}$, $\partial_{\btheta} \bgamma_0 = \mathbf{0}$), so the boundary term involving $(\partial_{\btheta} \balpha_0)^\top$ vanishes. The term $\left( \dtheta \partial_{\dot{\bs}} L_0 \right)^\top (\bs^\beta_0 - \balpha_0)$ also vanishes since both trajectories start from the same initial position ($\bs^\beta_0 = \bs^0_0 = \balpha_0$).
    \item At $t = T$: With the PFVP formulation, the final conditions of both free and nudged trajectories are matched, so $(\bs^\beta_T - \bs^0_T) = \mathbf{0}$ and $(\dot{\bs}^\beta_T - \dot{\bs}^0_T) = \mathbf{0}$, eliminating any final boundary contributions.
\end{itemize}
Therefore, only the integral term remains:
\begin{equation*}
    \dtheta \mathcal{C}[\bs^0(\btheta)] = \lim_{\beta \to 0} \frac{1}{\beta} \int_0^T \left[ \partial_{\btheta} L_0(\bs^\beta_t, \dot{\bs}^\beta_t, \btheta, x_t) - \partial_{\btheta} L_0(\bs^0_t, \dot{\bs}^0_t, \btheta, x_t) \right] \exp(\bzeta t) \, \mathrm{d}t\,.
\end{equation*}

The parameter gradients of $L_0$ are:
\begin{align*}
    \partial_{m_i} L_0 &= \frac{1}{2} \dot{s}_{i,t}^2 \quad \text{(for each mass $i = 1, \ldots, d$)} \\
    \partial_{\bm{K}} L_0 &= -\frac{1}{2} \bs_t \bs_t^\top \quad \text{(yields a $d \times d$ matrix)} \\
    \partial_{\bzeta} L_0 &= t \cdot L_0(\bs_t, \dot{\bs}_t, \btheta, x_t)\,.
\end{align*}
The damping coefficient gradient involves the full Lagrangian weighted by time $t$, reflecting how damping affects the exponential time-weighting factor $\exp(\bzeta t)$ in the dissipative formulation.

This gradient estimator provides an unbiased estimate of $\dtheta \mathcal{C}$ by comparing the time-weighted Lagrangian along free and nudged trajectories, without requiring any boundary term corrections.

\subsection{Energy Evolution for Harmonic Oscillators}
\label{appx:harmonic_energy_evolution}

We derive the explicit energy evolution for the dissipative harmonic oscillator system. Following Section~\ref{sec:dissipative}, we define the physical energy with respect to the isolated Lagrangian $L_0^{\mathrm{iso}}$ (obtained by setting $x_t = 0$).

\paragraph{Physical energy definition.}
For the harmonic oscillator, the isolated Lagrangian is:
\begin{equation*}
    L_0^{\mathrm{iso}}(\bs_t, \dot{\bs}_t, \btheta) = \frac{1}{2} (\bm{m} \odot \dot{\bs}_t) \cdot \dot{\bs}_t - \frac{1}{2} \bs_t^\top \bm{K} \bs_t\,.
\end{equation*}
The physical energy $E$ (as defined in~\eqref{eq:energy_definition}) is:
\begin{equation*}
    E(t) = \dot{\bs}_t^\top \partial_{\dot{\bs}} L_0^{\mathrm{iso}} - L_0^{\mathrm{iso}} = \underbrace{\frac{1}{2} (\bm{m} \odot \dot{\bs}_t) \cdot \dot{\bs}_t}_{E_{\mathrm{kin}}(t)} + \underbrace{\frac{1}{2} \bs_t^\top \bm{K} \bs_t}_{U_{\mathrm{int}}(t)}\,.
\end{equation*}
This is the standard mechanical energy: kinetic energy plus internal potential energy.

\begin{proposition}[Energy evolution for dissipative harmonic oscillators]
\label{prop:harmonic_energy_evolution}
For the harmonic oscillator system with proportional damping $\bgamma = \bzeta \bm{m}$, the physical energy $E(t) = E_{\mathrm{kin}}(t) + U_{\mathrm{int}}(t)$ evolves as:
\begin{equation}
    E(t) = E(0) + \underbrace{\left(-\int_0^t \dot{s}_{1,\tau} \, x_\tau \, \mathrm{d}\tau\right)}_{W_{\mathrm{input}}(t)} - \underbrace{\int_0^t \bgamma \cdot \dot{\bs}_\tau^2 \, \mathrm{d}\tau}_{D_{\mathrm{diss}}(t)}\,,
    \label{eq:energy_free_harmonic}
\end{equation}
where $\dot{\bs}_\tau^2 = \dot{\bs}_\tau \odot \dot{\bs}_\tau$ denotes element-wise squaring.

Equivalently, using $E_{\mathrm{kin}}(\tau) = \frac{1}{2} (\bm{m} \odot \dot{\bs}_\tau) \cdot \dot{\bs}_\tau$ and $\bgamma = \bzeta \bm{m}$:
\begin{equation*}
    E(t) = E(0) + W_{\mathrm{input}}(t) - 2\bzeta \int_0^t E_{\mathrm{kin}}(\tau) \, \mathrm{d}\tau\,.
\end{equation*}

The energy contributions are:
\begin{itemize}[leftmargin=*]
    \item \textbf{Input work} $W_{\mathrm{input}}(t) = -\int_0^t \dot{s}_{1,\tau} \, x_\tau \, \mathrm{d}\tau$: Work done by the external force $F_{\mathrm{ext}} = -x_t$ on the first oscillator. This follows the standard mechanics formula: power = force $\times$ velocity $= (-x_t) \cdot \dot{s}_{1,t}$. Can be positive (energy injection) or negative (energy extraction) depending on the correlation between velocity $\dot{s}_{1,\tau}$ and force $-x_\tau$.
    \item \textbf{Dissipation} $D_{\mathrm{diss}}(t) = \int_0^t \bgamma \cdot \dot{\bs}_\tau^2 \, \mathrm{d}\tau = 2\bzeta \int_0^t E_{\mathrm{kin}}(\tau) \, \mathrm{d}\tau \geq 0$: Energy dissipated by friction, proportional to the time-integrated kinetic energy. Always removes energy.
\end{itemize}
\end{proposition}

\begin{proof}
We derive the energy evolution for the free phase ($\beta = 0$) from first principles.

\paragraph{Step 1: Energy definition.}
Following Section~\ref{sec:dissipative}, the physical energy is defined with respect to the isolated Lagrangian:
\begin{equation*}
    E = E_{\mathrm{kin}} + U_{\mathrm{int}} = \frac{1}{2} (\bm{m} \odot \dot{\bs}_t) \cdot \dot{\bs}_t + \frac{1}{2} \bs_t^\top \bm{K} \bs_t\,.
\end{equation*}

\paragraph{Step 2: Time derivative of $E$.}
Taking the total time derivative:
\begin{align}
    d_t E &= d_t E_{\mathrm{kin}} + d_t U_{\mathrm{int}} \nonumber \\
    &= (\bm{m} \odot \ddot{\bs}_t) \cdot \dot{\bs}_t + \bs_t^\top \bm{K} \dot{\bs}_t \nonumber \\
    &= \dot{\bs}_t^\top \left( \bm{m} \odot \ddot{\bs}_t + \bm{K} \bs_t \right)\,.
    \label{eq:dE_step2}
\end{align}

\paragraph{Step 3: Using the equations of motion.}
For the dissipative harmonic oscillator (free phase with $\beta = 0$), the equation of motion is:
\begin{equation*}
    \bm{m} \odot \ddot{\bs}_t + \bgamma \odot \dot{\bs}_t + \bm{K} \bs_t = -x_t \bm{e}_1\,.
\end{equation*}
Rearranging:
\begin{equation}
    \bm{m} \odot \ddot{\bs}_t + \bm{K} \bs_t = -x_t \bm{e}_1 - \bgamma \odot \dot{\bs}_t\,.
    \label{eq:eom_rearranged}
\end{equation}

\paragraph{Step 4: Final expression for $d_t E$.}
Substituting~\eqref{eq:eom_rearranged} into~\eqref{eq:dE_step2}:
\begin{align*}
    d_t E &= \dot{\bs}_t^\top \left( -x_t \bm{e}_1 - \bgamma \odot \dot{\bs}_t \right)  \\
    &= -x_t \, \dot{s}_{1,t} - \bgamma \cdot \dot{\bs}_t^2\,,
\end{align*}
where $\dot{\bs}_t^2 = \dot{\bs}_t \odot \dot{\bs}_t$ denotes element-wise squaring.

Equivalently, using $E_{\mathrm{kin}}(t) = \frac{1}{2} (\bm{m} \odot \dot{\bs}_t) \cdot \dot{\bs}_t$ and $\bgamma = \bzeta \bm{m}$:
\begin{equation}
    d_t E = -x_t \, \dot{s}_{1,t} - 2\bzeta \, E_{\mathrm{kin}}(t)\,.
    \label{eq:harmonic_dE}
\end{equation}

\paragraph{Physical interpretation.}
The energy $E = E_{\mathrm{kin}} + U_{\mathrm{int}}$ evolves with two power contributions:
\begin{itemize}[leftmargin=*]
    \item $P_{\mathrm{input}} = -x_t \, \dot{s}_{1,t} = F_{\mathrm{ext}} \cdot \dot{s}_{1,t}$: Power delivered by the external force $F_{\mathrm{ext}} = -x_t$ acting on the first oscillator. This follows the standard mechanics formula: power = force $\times$ velocity. When the force and velocity are aligned (same sign), energy is injected; when opposed, energy is extracted.
    \item $P_{\mathrm{diss}} = \bgamma \cdot \dot{\bs}_t^2 = 2\bzeta \, E_{\mathrm{kin}}(t) \geq 0$: Power dissipated by friction (always positive, always removes energy from the system).
\end{itemize}

Integrating~\eqref{eq:harmonic_dE} from $0$ to $t$ yields the energy evolution~\eqref{eq:energy_free_harmonic}:
\begin{equation*}
    E(t) - E(0) = -\int_0^t x_\tau \, \dot{s}_{1,\tau} \, \mathrm{d}\tau - \int_0^t \bgamma \cdot \dot{\bs}_\tau^2 \, \mathrm{d}\tau\,.
\end{equation*}

\paragraph{Special case: isolated system.}
When $x_t = 0$ (no external input), the energy evolution simplifies to:
\begin{equation*}
    d_t E = - \bgamma \cdot \dot{\bs}_t^2 = -2\bzeta \, E_{\mathrm{kin}}(t) \leq 0\,.
\end{equation*}
This confirms that dissipation always removes energy from the system, as stated in Proposition~\ref{prop:energy_dissipation}.
\end{proof}

\subsection{Proof of Proposition~\ref{prop:dissipative_time_reversal}: Time-Reversal for Dissipative Systems}
\label{appx:dissipative_time_reversal}

\begin{proof}[Proof of Proposition~\ref{prop:dissipative_time_reversal}]
Consider the dissipative dynamics in the forward time direction. For simplicity, we present the proof for a single component (the multi-dimensional case follows by applying the same argument component-wise):
\begin{equation}
    m \ddot{s}_t + \zeta m \dot{s}_t + k s_t = f(t), \quad t \in [0, T]\,,
    \label{eq:forward_time_proof}
\end{equation}
where $m$, $\zeta$, and $k$ are scalar parameters, and the damping is proportional to the mass.

To solve this equation backward in time from $t = T$ to $t = 0$ as a final value problem, we introduce the backward time parameter $t' = T - t$. As $t$ runs from $T$ to $0$, $t'$ runs from $0$ to $T$.

\paragraph{Change of variables.}
Under the substitution $t = T - t'$, we have:
\begin{align*}
    s_t = s_{T-t'} &\equiv \overset{\scriptstyle\leftarrow}{s}_{t'} \\
    d_t &= -d_{t'} \\
    d_t^2 &= d_{t'}^2\,.
\end{align*}

\paragraph{First derivative transformation.}
The first time derivative transforms as:
\begin{equation*}
    \dot{s}_t = d_t s_t = d_t \overset{\scriptstyle\leftarrow}{s}_{t'} = d_{t'} \overset{\scriptstyle\leftarrow}{s}_{t'} \cdot d_t t' = -d_{t'} \overset{\scriptstyle\leftarrow}{s}_{t'} = -\dot{\overset{\scriptstyle\leftarrow}{s}}_{t'}\,,
\end{equation*}
where $\dot{\overset{\scriptstyle\leftarrow}{s}}_{t'} := d_{t'} \overset{\scriptstyle\leftarrow}{s}_{t'}$.

\paragraph{Second derivative transformation.}
The second time derivative transforms as:
\begin{align*}
    \ddot{s}_t = d_t^2 s_t &= d_t(d_t s_t) = d_t(-\dot{\overset{\scriptstyle\leftarrow}{s}}_{t'}) \\
    &= -d_t \dot{\overset{\scriptstyle\leftarrow}{s}}_{t'} = -d_{t'} \dot{\overset{\scriptstyle\leftarrow}{s}}_{t'} \cdot d_t t' \\
    &= -\ddot{\overset{\scriptstyle\leftarrow}{s}}_{t'} \cdot (-1) = \ddot{\overset{\scriptstyle\leftarrow}{s}}_{t'}\,.
\end{align*}

\paragraph{Equation transformation.}
Substituting these transformations into~\eqref{eq:forward_time_proof}:
\begin{align*}
    m \ddot{\overset{\scriptstyle\leftarrow}{s}}_{t'} + \zeta m \left(-\dot{\overset{\scriptstyle\leftarrow}{s}}_{t'}\right) + k \overset{\scriptstyle\leftarrow}{s}_{t'} &= f(T - t') \\
    m \ddot{\overset{\scriptstyle\leftarrow}{s}}_{t'} - \zeta m \dot{\overset{\scriptstyle\leftarrow}{s}}_{t'} + k \overset{\scriptstyle\leftarrow}{s}_{t'} &= f(T - t'), \quad t' \in [0, T]\,.
\end{align*}

This establishes that the dissipative term $\zeta m \dot{s}_t$ changes sign to $-\zeta m \dot{\overset{\scriptstyle\leftarrow}{s}}_{t'}$ when we transform to backward time, while the second-order term $m \ddot{s}_t$ remains unchanged (since it involves an even number of time derivatives).

\paragraph{Extension to vector case and IVP formulation.}
For the multi-dimensional case with mass vector $\bm{m}$, damping vector $\bgamma = \zeta \bm{m}$ (where $\zeta > 0$ is scalar), and stiffness matrix $\bm{K}$, the same time-reversal transformation applies component-wise. Following the same derivation as above with $\bs_{\bwda,t'} := \bs_{T-t'}$, we obtain the time-reversed equation. For the actual IVP formulation, we denote the solution trajectory simply as $\bs_{t'}$ (dropping the tilde notation), which satisfies:
\begin{equation*}
    \bm{m} \odot \ddot{\bs}_{t'} - \bgamma \odot \dot{\bs}_{t'} + \bm{K} \bs_{t'} = \bm{f}(T - t'), \quad t' \in [0, T]\,.
\end{equation*}
Note that only the dissipative term (the damping force $\bgamma \odot \dot{\bs}_t$) changes sign, while the stiffness term $\bm{K}\bs_{t'}$ (a matrix-vector product) remains unchanged.

\paragraph{Physical interpretation.}
The sign change of the dissipative term under time reversal reflects the fact that dissipation is time-irreversible: in forward time, friction removes energy from the system ($-\gamma \dot{s}_t$ opposes the velocity), while in backward time, the effective friction must add energy back into the system to reconstruct trajectories consistent with the forward dynamics.

\paragraph{Velocity-reversed initial conditions.}
To solve the PFVP with final conditions $(\bs_T, \dot{\bs}_T)$, we use the IVP in the $t'$ time coordinate with initial conditions:
\begin{equation*}
    (\bs_0, \dot{\bs}_0) = (\bs_T, -\dot{\bs}_T)\,.
\end{equation*}
Note the crucial sign flip on the initial velocity vector. This ensures that when we integrate forward in $t'$ with the sign-flipped dissipative term, we reconstruct the trajectory that would have led to the desired final conditions in the original time coordinate $t$.
\end{proof}

\section{Computational Complexity Analysis of LEP Instantiations}
\label{app:complexity}

\subsection{Motivation}

Although the ultimate goal of Lagrangian Equilibrium Propagation is to enable learning in continuous-time physical systems, understanding the computational complexity requires analyzing discrete-time implementations. This analysis serves two purposes. First, it provides concrete complexity characterizations for numerical simulations, which remain the primary means of validating these algorithms. Second, it reveals the fundamental scaling properties that carry over to continuous-time implementations, where the number of time steps $N$ corresponds to the temporal resolution or duration of the physical process.

Throughout this appendix, we discretize the continuous-time dynamics using the simplest Euler integration scheme. While higher-order integrators may be preferred in practice for numerical stability, they do not change the asymptotic complexity with respect to the key parameters: sequence length $N$, state dimension $d_s$, and parameter count $d_\theta$. The choice of Euler integration thus provides a lower bound on computational cost while maintaining clarity of exposition.

\subsection{Setup and Notation}

We analyze the computational complexity of three instantiations of Lagrangian Equilibrium Propagation: CIVP, CBPVP, and PFVP/RHEL. For concreteness, we consider the Hopfield Lagrangian from Table~\ref{tab:ml_hamiltonians}:
\begin{equation*}
L_0(\bs, \dot{\bs}, \btheta, \bu) = \frac{1}{2}\dot{\bs}^\top \mathrm{diag}(\tau) \dot{\bs} - \frac{\alpha}{2}\|\bs\|^2 - b^\top \rho(\bs) - \frac{1}{2}\rho(\bs)^\top W \rho(\bs) - \rho(\bs)^\top B \rho(\bu)\,,
\end{equation*}
which yields the second-order dynamics:
\begin{equation*}
\mathrm{diag}(\tau)\ddot{\bs} = -\rho'(\bs) \odot \left(\alpha \bs + W\rho(\bs) + b + B\rho(\bu)\right)\,,
\end{equation*}
where $\tau \in \mathbb{R}^{d_s}_{>0}$ is a vector of learnable time constants, $\rho$ denotes a pointwise nonlinearity (e.g., $\tanh$), $W \in \mathbb{R}^{d_s \times d_s}$ is a symmetric weight matrix, $B \in \mathbb{R}^{d_s \times d_u}$ is the input coupling matrix, and $\odot$ denotes elementwise multiplication.

We adopt the following notation throughout. Let $N$ denote the number of discrete time steps, corresponding to the sequence length. If the continuous-time dynamics span duration $T$ and the integration step size is $\Delta t$, then $N = T / \Delta t$. Let $d_s$ denote the state dimension, where both position $\bs$ and velocity $\dot{\bs}$ have dimension $d_s$. Let $d_\theta$ denote the number of learnable parameters; for the Hopfield model, $d_\theta = \mathcal{O}(d_s^2)$ due to the dense matrices $W$ and $B$. For CBPVP, we additionally define $K$ as the number of iterations required for the boundary value problem solver to converge. If $T_{\text{gd}}$ denotes the total optimization time needed for convergence and $\Delta\tau$ is the step size in the artificial relaxation time $\tau$, then $K = T_{\text{gd}} / \Delta\tau$. Empirically, for systems related to Equilibrium Propagation, $K$ typically scales with the number of neurons and the number of layers in hierarchical architectures~\citep{scellierEquilibriumPropagationBridging2017}. In CBPVP, time is spatialized (Section~\ref{subsec:cfvp}), so each discrete time step can be understood as a single layer. Under this analogy, $d_s$ controls within-layer relaxation while $N$ controls between-layer signal propagation, suggesting that $K$ will generally grow with both $N$ and $d_s$.

We denote by $C_f$ the cost of one dynamics evaluation. For the Hopfield Lagrangian, each evaluation of the right-hand side $f(\bs, \dot{\bs}, \btheta, \bu) = -\mathrm{diag}(\tau)^{-1}\rho'(\bs) \odot (\alpha \bs + W\rho(\bs) + b + B\rho(\bu))$ requires computing the pointwise nonlinearity $\rho(\bs)$ in $\mathcal{O}(d_s)$ operations, the dense matrix-vector product $W\rho(\bs)$ in $\mathcal{O}(d_s^2)$ operations, and the input coupling $B\rho(\bu)$ in $\mathcal{O}(d_s \cdot d_u)$ operations. The diagonal scaling by $\mathrm{diag}(\tau)^{-1}$ adds only $\mathcal{O}(d_s)$ operations. The total cost is therefore $C_f = \mathcal{O}(d_s^2)$, dominated by the dense matrix-vector multiplication. For architectures with diagonal or sparse weight matrices, this reduces to $C_f = \mathcal{O}(d_s)$.

\subsection{CIVP (Constant Initial Value Problem)}
\label{app:complexity:civp}

The CIVP formulation is defined in Section~\ref{subsec:civp}. In CIVP, all trajectories share fixed initial conditions $(\bs_0, \dot{\bs}_0) = (\balpha, \bgamma)$ that are independent of both the parameters $\btheta$ and the nudging strength $\beta$. The free and nudged trajectories are computed by forward integration from this common initial state.

\paragraph{Dynamics computation.}

Both the free phase ($\beta = 0$) and nudged phase ($\beta > 0$) constitute initial value problems that can be solved by forward integration. Using Euler discretization, the update rule takes the form:
\begin{equation*}
\bs_{t+1} = 2\bs_t - \bs_{t-1} + \Delta t^2 \cdot f(\bs_t, \dot{\bs}_t, \btheta, \bx_t)\,.
\end{equation*}
Each time step requires one evaluation of the dynamics at cost $C_f = \mathcal{O}(d_s^2)$. With $N$ time steps and two phases (free and nudged), the total dynamics computation requires $\mathcal{O}(N \cdot d_s^2)$ operations. 

Regarding memory, the Euler integrator only requires access to the current and previous states to compute the next state. The dynamics computation therefore requires only $\mathcal{O}(d_s)$ memory.

\paragraph{Gradient computation.}

The CIVP gradient estimator, given by Corollary~\ref{corollary:civp_gradient}, takes the form:
\begin{equation}
\Delta^{\text{CIVP}}(\beta) = \frac{1}{\beta}\left[\int_0^T [\partial_{\btheta} L_\beta - \partial_{\btheta} L_0]\, dt + (\partial_{\dot{\bs}} L_\beta - \partial_{\dot{\bs}} L_0)^\top \partial_{\btheta} \bs_T^0 - (d_{\btheta} \partial_{\dot{\bs}} L_0)^\top (\bs_T^\beta - \bs_T^0)\right]\,.
\label{eq:civp_estimator_appendix}
\end{equation}
The problematic term is $\partial_{\btheta} \bs_T^0 \in \mathbb{R}^{d_s \times d_\theta}$, which represents the sensitivity of the final state with respect to all parameters. This full Jacobian can be computed via backpropagation through time (BPTT), but since BPTT computes the gradient of a scalar output, one must run $d_s$ separate backward passes---one for each component of $\bs_T^0$---to obtain the complete matrix.

BPTT proceeds by first storing the entire forward trajectory $\{\bs_t : t = 0, \ldots, N\}$, then executing backward passes to accumulate gradients. Each backward pass has the same computational structure as the forward pass, requiring $\mathcal{O}(N \cdot d_s^2)$ operations, so computing the full Jacobian costs $\mathcal{O}(d_s \cdot N \cdot d_s^2) = \mathcal{O}(N \cdot d_s^3)$ operations. Moreover, BPTT necessitates storing all intermediate states to enable the backward passes, resulting in $\mathcal{O}(N \cdot d_s)$ memory consumption.

The remaining terms in Eq.~\eqref{eq:civp_estimator_appendix} are as follows. The integral term $\int_0^T [\partial_{\btheta} L_\beta - \partial_{\btheta} L_0]\, dt$ requires $\mathcal{O}(N \cdot d_\theta)$ operations and can be accumulated during the two forward passes by maintaining two running sums:
\begin{align*}
\text{acc}_{\text{free}} &\leftarrow \text{acc}_{\text{free}} + \partial_{\btheta} L_0(\bs_t^0, \dot{\bs}_t^0, \btheta) \cdot \Delta t \\
\text{acc}_{\text{nudged}} &\leftarrow \text{acc}_{\text{nudged}} + \partial_{\btheta} L_\beta(\bs_t^\beta, \dot{\bs}_t^\beta, \btheta) \cdot \Delta t\,.
\end{align*}
Each $\partial_{\btheta} L$ evaluation is performed once and immediately accumulated, requiring no trajectory storage for this term. The difference $(\partial_{\dot{\bs}} L_\beta - \partial_{\dot{\bs}} L_0)$ and the state difference $(\bs_T^\beta - \bs_T^0)$ are both $\mathcal{O}(d_s)$ to compute. However, the term $d_{\btheta} \partial_{\dot{\bs}} L_0$ is equally problematic: by the chain rule, $d_{\btheta} \partial_{\dot{\bs}} L_0 = \partial^2_{\dot{\bs}, \dot{\bs}} L_0 \cdot d_{\btheta} \dot{\bs}_T^0 + \partial^2_{\btheta, \dot{\bs}} L_0$, which involves the Jacobian $d_{\btheta} \dot{\bs}_T^0 \in \mathbb{R}^{d_s \times d_\theta}$---the sensitivity of the final velocity to all parameters. Computing this Jacobian incurs the same $\mathcal{O}(N \cdot d_s^3)$ cost as $\partial_{\btheta} \bs_T^0$.

This memory cost, which scales linearly with the sequence length $N$, constitutes the fundamental limitation of CIVP. It negates the primary advantage of Equilibrium Propagation, which aims to avoid storing trajectories for gradient computation.

\paragraph{Forward-only property.}

CIVP is not forward-only. The gradient computation requires an explicit backward pass through the stored computational graph. The system cannot compute gradients by running forward dynamics alone; it must differentiate through the ODE solver, necessitating either trajectory storage with backpropagation or forward propagation of a $d_s \times d_\theta$ Jacobian at each step (the RTRL algorithm, which incurs even greater time complexity).

\subsection{CBPVP (Constant Boundary Position Value Problem)}

The CBPVP formulation is defined in Section~\ref{subsec:cfvp}. In CBPVP, all trajectories satisfy fixed position boundary conditions at both temporal endpoints: $\bs_0 = \balpha$ and $\bs_T = \bgamma$, independent of $\btheta$ and $\beta$. The velocities $\dot{\bs}_0$ and $\dot{\bs}_T$ remain free to vary.

\paragraph{Dynamics computation.}

Unlike CIVP, the CBPVP formulation defines a two-point boundary value problem (BVP) that cannot be solved by simple forward integration. Instead, one solves it via gradient descent on the action functional, as described in Eq.~26:
\begin{equation*}
\partial_\tau \bs_t = -\delta_{\bs} \mathcal{A}_\beta = -\text{EL}(\bs_{t-1}, \bs_t, \bs_{t+1}, \btheta, \beta), \quad t = 1, \ldots, N-1\,,
\end{equation*}
with fixed boundaries $\bs_0 = \balpha$ and $\bs_T = \bs_N = \bgamma$. Here $\tau$ represents an artificial relaxation time, while the physical time $t$ becomes a spatial index. The procedure initializes a trajectory guess satisfying the boundary conditions, then iteratively updates the interior points according to the Euler-Lagrange residual until convergence.

Each relaxation iteration requires evaluating the Euler-Lagrange expression at all $N$ time points, with each evaluation costing $\mathcal{O}(C_f) = \mathcal{O}(d_s^2)$. A single iteration therefore costs $\mathcal{O}(N \cdot d_s^2)$. Convergence typically requires $K$ iterations, where $K$ depends on the problem conditioning and initialization quality. The total dynamics computation thus requires $\mathcal{O}(K \cdot N \cdot d_s^2)$ operations.

The iterative nature of the BVP solver requires storing the entire trajectory $\{\bs_t : t = 0, \ldots, N\}$ simultaneously, as all points are updated together in each iteration. The dynamics memory is therefore $\mathcal{O}(N \cdot d_s)$.

\paragraph{Gradient computation.}

The CBPVP gradient estimator, given by Corollary 2 (Eq.~24), simplifies considerably:
\begin{equation*}
\Delta^{\text{CBPVP}}(\beta) = \frac{1}{\beta}\int_0^T [\partial_{\btheta} L_\beta - \partial_{\btheta} L_0]\, dt\,.
\end{equation*}
The boundary residuals vanish entirely because both endpoint positions are fixed. This is the principal advantage of the CBPVP formulation.

Computing this estimator requires evaluating the Lagrangian parameter derivatives $\partial_{\btheta} L$ at each of the $N$ time points along both converged trajectories, after the $K$ relaxation iterations have completed. For the Hopfield model, the dominant cost arises from $\partial_W L$, which involves outer products of dimension $d_s \times d_s$. Since $d_\theta = \mathcal{O}(d_s^2)$, the gradient computation requires $\mathcal{O}(N \cdot d_\theta)$ operations.

The gradient estimator only requires accumulating a running sum of dimension $d_\theta$, resulting in $\mathcal{O}(d_\theta)$ memory for the gradient computation.

\paragraph{Forward-only property.}

CBPVP is forward-only in the sense that no backward pass through a computational graph is required. The gradient estimator does not require computing complex boundary residuals and the iterative solver only requires forward dynamics evaluations. However the iterative solver is much more expensive than the forward dynamics, requiring $K$ iterations and $\mathcal{O}(N \cdot d_s)$ memory to store all time points simultaneously. These constraints preclude online or streaming processing of temporal sequences.

\subsection{PFVP/RHEL (Parametric Final Value Problem)}

The PFVP formulation is introduced in Section~\ref{subsubsec:pfvp-definition} and its equivalence to RHEL (Section~\ref{sec:rhel}) is established in Section~\ref{sec:rhel_is_ep}. In PFVP, the nudged trajectory shares its final conditions with the free trajectory's final state, but with reversed velocity: $(\bs_{\bwda,T}^\beta, \dot{\bs}_{\bwda,T}^\beta) = (\bs_T^0, -\dot{\bs}_T^0)$. These boundary conditions depend on $\btheta$ through the free trajectory, which distinguishes PFVP from the constant boundary conditions of CIVP and CBPVP.

\paragraph{Key insight: exploiting reversibility.}

For time-reversible Lagrangians satisfying $L(\bs, \dot{\bs}) = L(\bs, -\dot{\bs})$, Proposition 2 establishes that the final value problem can be converted to an initial value problem:
\begin{equation*}
\bs_{\bwda,t}^\beta(\btheta, (\bs_T^0, \dot{\bs}_T^0)) = \bs_{T-t}^\beta(\btheta, (\bs_T^0, -\dot{\bs}_T^0))\,.
\end{equation*}
Rather than solving a difficult final value problem, one simply performs forward integration from the velocity-reversed final state while playing the input sequence backward. This transformation is the key enabler of PFVP's computational efficiency.

\paragraph{Dynamics computation.}

The free phase proceeds by standard forward integration from initial conditions $(\balpha, \bgamma)$ over the interval $t \in [0, T]$, storing only the final state $(\bs_T^0, \dot{\bs}_T^0)$ upon completion. This requires $\mathcal{O}(N \cdot d_s^2)$ operations.

The echo phase initializes from $(\bs_T^0, -\dot{\bs}_T^0)$ and integrates forward over $t \in [0, T]$, using the time-reversed input sequence $\bx_{T-t}$ and targets $\by_{T-t}$. This also requires $\mathcal{O}(N \cdot d_s^2)$ operations.

The total dynamics computation is therefore $\mathcal{O}(N \cdot d_s^2)$, identical to CIVP. Note, however, that the two phases must be executed sequentially: the echo phase requires the final state $(\bs_T^0, \dot{\bs}_T^0)$ from the free phase to initialize its dynamics. This contrasts with CIVP, where the free and nudged trajectories are independent initial value problems that can be computed in parallel. 

Regarding memory, each phase requires only the current state for the Euler integrator, consuming $\mathcal{O}(d_s)$ memory. The only additional storage is the final state of the free phase, which is needed to initialize the echo phase, also $\mathcal{O}(d_s)$. The dynamics memory is therefore $\mathcal{O}(d_s)$, independent of the sequence length $N$.

\paragraph{Gradient computation.}

The PFVP gradient estimator, given by Theorem 3 (Eq.~43), takes the form:
\begin{equation*}
\Delta^{\text{PFVP}}(\beta) = \frac{1}{\beta}\left[\int_0^T [\partial_{\btheta} L_\beta - \partial_{\btheta} L_0]\, dt + (d_{\btheta} \partial_{\dot{\bs}} L_0)^\top (\bs_{\bwda,0}^\beta - \balpha) - (\partial_{\btheta} \balpha)^\top (\partial_{\dot{\bs}} L_\beta - \partial_{\dot{\bs}} L_0)\right]\,.
\end{equation*}

Unlike CIVP, no trajectory sensitivities such as $\partial_{\btheta} \bs_T^0$ appear in this estimator, which eliminates the need for backpropagation.

As in CIVP, the integral term can be computed by maintaining two accumulators that are updated at each time step during the forward integration, requiring $\mathcal{O}(N \cdot d_\theta)$ operations. The boundary terms are computed as follows: the state difference $(\bs_{\bwda,0}^\beta - \balpha)$ and momentum difference $(\partial_{\dot{\bs}} L_\beta - \partial_{\dot{\bs}} L_0)$ are both $\mathcal{O}(d_s)$ to compute. When $\partial_{\btheta} \balpha = 0$, the second boundary term vanishes. The first boundary term $(d_{\btheta} \partial_{\dot{\bs}} L_0)^\top (\bs_{\bwda,0}^\beta - \balpha)$ involves $d_{\btheta} \partial_{\dot{\bs}} L_0 = \partial^2_{\btheta, \dot{\bs}} L_0$, which for the Hopfield model is $\mathcal{O}(d_s \times d_\theta)$; however, the contraction with the $\mathcal{O}(d_s)$ vector $(\bs_{\bwda,0}^\beta - \balpha)$ yields an $\mathcal{O}(d_\theta)$ result. For the Hopfield Lagrangian where $L = \frac{1}{2}\dot{\bs}^\top \mathrm{diag}(\tau) \dot{\bs} - V(\bs, \btheta)$, we have $\partial_{\dot{\bs}} L = \mathrm{diag}(\tau)\dot{\bs}$. The only $\btheta$-dependence is through $\tau$, so $\partial^2_{\tau, \dot{\bs}} L = \mathrm{diag}(\dot{\bs})$, which is diagonal and $\mathcal{O}(d_s)$. The contraction $(\partial^2_{\tau, \dot{\bs}} L)^\top (\bs_{\bwda,0}^\beta - \balpha) = \dot{\bs} \odot (\bs_{\bwda,0}^\beta - \balpha)$ is therefore $\mathcal{O}(d_s)$ to compute.

The total gradient computation requires $\mathcal{O}(N \cdot d_\theta)$ operations, dominated by the integral term. The memory requirement is $\mathcal{O}(d_\theta)$ for the two accumulators plus $\mathcal{O}(d_s)$ for the boundary quantities.

\paragraph{Forward-only property.}

PFVP/RHEL satisfies the forward-only property in the strongest sense. Both the free and echo phases consist of pure forward integration. No iterative solving is required, in contrast to CBPVP. No backward pass through a computational graph is needed, in contrast to CIVP. The gradients are computed from Lagrangian derivatives accumulated along the trajectories during the forward passes.

It is important to note that the echo phase is not a backward pass in the algorithmic sense. It is a forward pass with reversed initial velocity and reversed input sequence. The physical system runs forward in time during both phases. This property makes PFVP/RHEL uniquely suitable for implementation in physical hardware, where information propagates forward through the system's natural dynamics.

\subsection{Summary}

The analysis reveals a clear hierarchy among the three instantiations, as summarized in Table~\ref{tab:complexity} in the main text. CIVP achieves efficient dynamics computation but requires BPTT for gradient estimation, incurring $\mathcal{O}(N \cdot d_s)$ memory to store the trajectory and $\mathcal{O}(N \cdot d_s ^3)$ in time complexity.

CBPVP eliminates the boundary residuals entirely, yielding a clean gradient estimator, and maintains the forward-only property. However, solving the boundary value problem requires $K$ iterations and $\mathcal{O}(N \cdot d_s)$ memory to store all time points simultaneously. These constraints preclude online or streaming processing and may result in slow convergence for challenging problems.

PFVP/RHEL achieves optimal scaling across all metrics. The dynamics computation matches CIVP's efficiency through pure forward integration. The gradient computation requires only local Lagrangian derivatives accumulated during the forward passes. Both time and memory complexities are independent of the sequence length $N$ in terms of trajectory storage, with memory scaling only as $\mathcal{O}(d_s + d_\theta)$. These properties make PFVP/RHEL the only instantiation suitable for online learning in physical systems where memory and the ability to process streaming data are fundamental constraints.

\section{Experimental Details}
\label{appx:experimental_details}

\subsection{Hopfield-Inspired System (Figure~\ref{fig:empirical_validation})}

This experiment trains a $d=6$ Hopfield-inspired system over 100 epochs in a teacher-student setup. Both RHEL and LEP training runs use the Adam optimizer with learning rate $\eta = 0.005$, nudging strength $\beta = 0.01$, Euler integration with $\dt = 0.001$, total duration $T = 10$, and random seed $50$. Gradients are saved every 10 epochs.

\paragraph{Weight initialization.} The symmetric weight matrix $\bm{W} \in \mathbb{R}^{d \times d}$ is initialized via QR decomposition with controlled eigenvalues. A random matrix is drawn from $\mathcal{N}(0,1)$ and its QR factorization yields an orthogonal matrix $\bm{U}$. A diagonal matrix $\bm{S} = \mathrm{diag}(\bm{\lambda})$ is formed with eigenvalues $\lambda_i \sim \mathrm{Uniform}(0.1, 1.0)$. The weight matrix is then constructed as $\bm{W} = \bm{U} \bm{S} \bm{U}^\top$, ensuring symmetry and bounded eigenvalues for stable dynamics.

\paragraph{Time constants.} Each $\tau_i$ is sampled independently from $\mathrm{Uniform}(0.5, 1.0)$.

\paragraph{Initial conditions.} The Hamiltonian initial conditions are fixed:
\begin{align*}
    \bs_0 &= (0.10,\; 0.15,\; 0.08,\; 0.12,\; 0.10,\; 0.11)^\top\,, \\
    \bp_0 &= (0.0,\; 0.4,\; -0.6,\; 0.45,\; 0.5,\; 0.0)^\top\,.
\end{align*}
In the LEP training run, the Lagrangian initial velocity is $\dot{\bs}_0 = \mathrm{diag}(\bm{\tau})^{-1} \bp_0$, which changes across training epochs as $\bm{\tau}$ evolves.

\paragraph{Input signal.} The input $x_t$ is a superposition of $n_{\text{waves}} = 10$ random sine waves injected into neuron 0 (with input scaling $1.0$). The activation function is $\rho = \tanh$.

\subsection{Dissipative Harmonic Oscillators (Figure~\ref{fig:dissipative_oscillators})}

This experiment validates the dissipative LEP gradient estimator on a $d=6$ system of coupled damped harmonic oscillators. No training is performed: the gradient comparison is computed at a single epoch from the randomly initialized parameters, comparing the dissipative LEP gradient estimate against the autodiff/BPTT ground truth. Integration uses $\dt = 0.001$, total duration $T = 10$, nudging strength $\beta = 0.01$, and random seed $50$.

\paragraph{Mass and stiffness initialization.} Masses $m_i$ are sampled from $\mathrm{Uniform}(0.8, 1.2)$. The stiffness matrix $\bm{K}$ is constructed to be symmetric positive semi-definite: diagonal self-coupling terms are scaled by $0.5$, off-diagonal coupling terms by $1.0$, and the matrix is symmetrized via $\bm{K} = \frac{1}{2}(\bm{K} + \bm{K}^\top)$, with diagonal entries adjusted to include the row sums of the coupling matrix.

\paragraph{Dissipation.} The damping coefficient is $\bzeta = 0.2$, giving per-dimension damping forces $\gamma_i = \bzeta \cdot m_i$.

\paragraph{Initial conditions.} All positions are initialized to $s_{i,0} = 1.0$ and all velocities to $\dot{s}_{i,0} = 0$, ensuring that boundary terms in the gradient estimator vanish (Remark~\ref{remark:zero_initial_momentum}).

\paragraph{Input signal.} The external drive is injected into oscillator 1 with input scaling $5.0$. The output is measured from oscillator $d$ (the last one). The cost function is $c(\bs_t, y_t) = \frac{1}{2}(s_{d,t} - y_t)^2$.

\section{Generalization to Anisotropic Damping}
\label{app:anisotropic}

In Section~\ref{sec:dissipative}, we introduced the dissipative Lagrangian $L^{\mathrm{diss}}_\beta = L_0 \cdot \exp(\zeta t)$ with a \emph{scalar} damping coefficient $\zeta > 0$, which produces uniform proportional damping $\bm{m} \odot \ddot{\bs}_t + \zeta \, \bm{m} \odot \dot{\bs}_t + \bm{K} \bs_t = \bm{f}(t)$ where all oscillators share the same damping ratio. Here we present a generalization that allows \emph{anisotropic} (per-dimension) damping rates while maintaining a variational structure.

\subsection{Anisotropic Exponential Integrating-Factor Lagrangian}

Let $\bs(t) \in \mathbb{R}^d$, $\bm{m} = (m_1, \ldots, m_d)^\top \in \mathbb{R}^d$ be the mass vector, and define the per-dimension damping coefficients $\bm{\gamma} = (\gamma_1, \ldots, \gamma_d)^\top \in \mathbb{R}^d$ with $\gamma_i > 0$ (not necessarily equal). Define the elementwise exponential:
\begin{equation*}
    \bm{e}(t) := \exp(\bm{\gamma} \odot t) = (e^{\gamma_1 t}, \ldots, e^{\gamma_d t})^\top \in \mathbb{R}^d\,,
\end{equation*}
where $\odot$ denotes elementwise multiplication.

Pick \emph{any symmetric matrix function} $B(t) = B(t)^\top \in \mathbb{R}^{d \times d}$. The time-dependent Lagrangian is:
\begin{equation*}
    L(t, \bs, \dot{\bs}) = \frac{1}{2} \sum_{i=1}^d e^{\gamma_i t} m_i \dot{s}_i^2 - \frac{1}{2} \bs^\top B(t) \bs\,.
\end{equation*}

\paragraph{Euler-Lagrange equations.}
For each dimension $i$, we have $\partial_{\dot{s}_i} L = e^{\gamma_i t} m_i \dot{s}_i$ and $\frac{d}{dt}(e^{\gamma_i t}) = \gamma_i e^{\gamma_i t}$. The Euler-Lagrange equation gives:
\begin{equation*}
    \frac{d}{dt}\big( e^{\gamma_i t} m_i \dot{s}_i \big) + [B(t) \bs]_i = 0
    \quad\Longleftrightarrow\quad
    e^{\gamma_i t} m_i \ddot{s}_i + \gamma_i e^{\gamma_i t} m_i \dot{s}_i + [B(t) \bs]_i = 0\,.
\end{equation*}
Dividing by $e^{\gamma_i t}$ and writing in vector form:
\begin{equation*}
    \bm{m} \odot \ddot{\bs}_t + \bm{\gamma} \odot \bm{m} \odot \dot{\bs}_t + \bm{k}_{\mathrm{eff}}(t) = \bm{0}, \qquad \bm{k}_{\mathrm{eff}}(t) := \exp(-\bm{\gamma} \odot t) \odot (B(t) \bs_t)\,.
\end{equation*}

Thus, \emph{anisotropic damping} with per-dimension damping rates $\gamma_i$ is generated exactly. The price is an induced, generally \emph{time-varying}, effective force $\bm{k}_{\mathrm{eff}}(t)$ determined by the choice of $B(t)$.

\subsection{Physical Interpretation: Time-Varying Coupling}

The effective force $\bm{k}_{\mathrm{eff}}(t) = \exp(-\bm{\gamma} \odot t) \odot (B(t) \bs_t)$ has a natural interpretation: \emph{the coupling between oscillators switches off with time}. Each oscillator $i$ has its own exponential decay coefficient $e^{-\gamma_i t}$, so the coupling from oscillator $i$ to the rest of the network decays according to \emph{its own damping rate} $\gamma_i$. Different oscillators can "disconnect" from the network at different rates, leading to time-dependent coupling encoded in $\bm{k}_{\mathrm{eff}}(t)$.

If one wishes the physical coupling to remain \emph{time-independent} in the sense that $\bm{k}_{\mathrm{eff}}(t) = \bm{K} \bs_t$ for some constant matrix $\bm{K}$, one must choose $B(t)$ such that $\exp(-\bm{\gamma} \odot t) \odot (B(t) \bs_t) = \bm{K} \bs_t$ for all $\bs_t$. This requires $B(t)_{ij} = e^{(\gamma_i + \gamma_j)t/2} K_{ij}$ (assuming a symmetric construction). However, for $B(t)$ to be symmetric (as required for a proper mechanical potential), we need additional structure.

The simplest cases where time-independent coupling is achievable are:
\begin{itemize}[leftmargin=*]
    \item All $\gamma_i$ equal (scalar damping) — this is the case in Section~\ref{sec:dissipative};
    \item $\bm{K}$ is diagonal (uncoupled oscillators);
    \item Special damping structures where the per-dimension rates align with the coupling structure.
\end{itemize}

When these conditions fail (generic coupling with different $\gamma_i$), maintaining time-independent physical coupling within the variational framework is not possible—one must either restrict the damping structure or accept time-varying $\bm{k}_{\mathrm{eff}}(t)$ in the learning dynamics.

\subsection{Comparison with Alternative Approaches}

One might consider using Rayleigh dissipation functions $\mathcal{R}(\dot{\bs}) = \frac{1}{2} \sum_{ij} C_{ij} \dot{s}_i \dot{s}_j$ (separate from the Lagrangian $L$), which handle arbitrary damping matrices elegantly in classical mechanics via the modified Euler-Lagrange equation $d_t \partial_{\dot{\bs}} L - \partial_{\bs} L + \partial_{\dot{\bs}} \mathcal{R} = 0$. However, this approach is \emph{incompatible} with the variational gradient estimator framework presented in this work (Theorem~\ref{thm:dissipative_LEP}), which requires all system dynamics to be encoded within the Lagrangian $L^{\mathrm{diss}}_\beta$ itself. The gradient estimator depends on $\partial_{\btheta} L_0$, not on a separate dissipation function.

More broadly, one could perform gradient descent directly on the action functional. However, as discussed in the CBPVP instantiation (Section~\ref{sec:rhel_is_ep}), the \emph{converging phase} of such optimization is dissipative (evolving in the artificial relaxation time $\tau$), while the \emph{fixed Hamiltonian system} implemented after convergence corresponds to a non-dissipative system on the physical time axis. The value of maintaining a \emph{variational principle within the Lagrangian itself} is that it guides the construction of learning algorithms systematically, enabling principled extensions like the dissipative LEP framework, rather than relying on ad-hoc inspired guesses as was done in prior work (e.g., RHEL before its variational foundation was established in Theorem~\ref{thm:lep-rhel-equivalence}).

\section{Unconstrained Action Minimization}
\label{app:unconstrained-action}

In the main text (Section~\ref{subsec:instantiation}), we noted that if one is willing to accept iterative optimization rather than forward Euler-Lagrange integration, boundary conditions need not be imposed at all. We elaborate on this observation here.

Consider minimizing the action functional without any boundary constraints:
\begin{align}
    \bs^\beta = \arg\min_{\bs} A_\beta[\bs]\,.
    \label{eq:unconstrained-action}
\end{align}
Since the initial and final values $\bs_0^\beta$ and $\bs_T^\beta$ are determined implicitly as part of the optimization, the variational principle is no longer partial: the boundary terms in the first variation of the action vanish by the natural boundary conditions (which require $\partial_{\dot{\bs}} L_\beta = 0$ at both endpoints). Consequently, boundary residuals vanish entirely in Theorem~\ref{thm:general-ep}, and the gradient estimator reduces to the simple form of CBPVP (Eq.~\eqref{eq:CBPVP_gradient}).

However, this formulation inherits the same non-causal drawbacks as CBPVP---and is in fact more expensive. In CBPVP, the $2d_s$ boundary values $(\boldsymbol{\alpha}_0, \boldsymbol{\alpha}_T)$ are fixed, so the optimization runs over the interior of the trajectory: a space of dimension $(N-2) \times d_s$. In the unconstrained formulation, the full trajectory including its endpoints becomes part of the optimization, increasing the search space to $N \times d_s$. The iterative solver cost remains $\mathcal{O}(K N d_s^2)$ with a potentially larger $K$ due to the additional degrees of freedom.

In summary, unconstrained action minimization yields a ``perfect'' variational formulation---analogous to standard EP---where the gradient estimator is free of boundary residuals. Yet this comes at the price of a non-causal, iterative computation that is at least as expensive as CBPVP, making it equally impractical for forward-only hardware implementations.

\end{document}